%% file: main.tex
\documentclass{article}

\usepackage[preprint]{neurips_2026}

\usepackage[utf8]{inputenc} 
\usepackage[T1]{fontenc}    
\usepackage{hyperref}       
\usepackage{url}            
\usepackage{booktabs}       
\usepackage{amsfonts}       
\usepackage{nicefrac}       
\usepackage{microtype}      
\usepackage[table]{xcolor}
\usepackage{subcaption}
\usepackage{wrapfig}
\usepackage{multirow}
\usepackage{booktabs}
\usepackage{array}
\usepackage{enumitem}

\usepackage{amsthm}
\usepackage{amsmath,amssymb,mathtools}
\usepackage{algorithm}
\usepackage{algpseudocode}

\usepackage{booktabs}
\usepackage{multirow}
\usepackage{makecell}
\usepackage{graphicx}

\title{
Queryable LoRA: Instruction-Regularized Routing Over Shared Low-Rank Update Atoms
}

\newtheorem{theorem}{Theorem}[section]
\newtheorem{corollary}{Corollary}[theorem]
\newtheorem{lemma}[theorem]{Lemma}
\newtheorem{assumption}[theorem]{Assumption}

\newtheorem{proposition}[theorem]{Proposition}

\newcommand{\std}[1]{{\scriptsize $\pm$#1}}

\author{%
  Omatharv Vaidya
  \\
  Department of Computer Science\\
  University of Texas at Austin\\
  \texttt{vomatharv@texas.edu} \\
  \And
  Connor T. Jerzak \\
  Department of Government \\
  University of Texas at Austin\\
  \texttt{connor.jerzak@austin.utexas.edu} \\
  \And
  Nhat Ho \\
  Department of Statistics and Data Sciences \\
  University of Texas at Austin \\
  \texttt{minhnhat@utexas.edu} \\
  \And
  Chandrajit Bajaj \\
  Department of Computer Science \\
  University of Texas at Austin \\
  \texttt{bajaj@cs.utexas.edu} \\
}

\begin{document}

\maketitle

\input{Chapters/0_Abstract}
\input{Chapters/1_Introduction}
\input{Chapters/2_Problem_setup}
\input{Chapters/3_Approach}

\input{Chapters/4_Empirical_study}
\input{Chapters/5_Theoretical_results}

\input{Chapters/6_Conclusion}

\clearpage 

\bibliographystyle{plainnat} 
\bibliography{references}    

\input{Chapters/7_Appendix}




\end{document}

%% file: Chapters/0_Abstract.tex
\begin{abstract}
We present a data-adaptive method for parameter-efficient fine-tuning of large neural networks. Standard low-rank adaptation methods improve efficiency by restricting each layer update to a fixed low-rank form, but this static parameterization can be too rigid when the appropriate correction depends on the input and on the evolving depth-wise computation of the network. Our approach replaces a purely layer-local adapter with a shared queryable memory of low-rank update atoms. For each block of layers, the model forms a query from the current low-rank state and a running summary of previous blocks, uses this query to retrieve a content-dependent combination of shared update components via attention, and applies the resulting routed operator within the low-rank bottleneck. In this way, the method retains the efficiency and scalability of low-rank adaptation while allowing the effective update to vary across inputs and to share reusable structure across layers. The resulting architecture provides a principled middle ground between static LoRA-style updates and fully generated parameter updates: it remains compact and parameter-efficient while supporting dynamic, context-sensitive adaptation. Further, we incorporate instruction-regularization by augmenting routing logits with a language-induced prior over update atoms, thereby biasing the selection of low-rank transformations toward semantically relevant directions without generating unconstrained parameter updates. Experiments on noisy non-linear regression tasks and LLM fine-tuning suggest that this queryable update-memory formulation can improve final test performance and training stability compared to standard low-rank adaptation, while using a comparable number of trainable parameters.
\end{abstract}

%% file: Chapters/1_Introduction.tex
\section{Introduction}

Large language models (LLMs) are now widely adapted to downstream tasks through fine-tuning. Standard low-rank adaptation methods for LLMs---most prominently, LoRA \citep{hu2022lora}---achieve parameter efficiency by restricting each layer update to a fixed, layer-local low-dimensional subspace. This approach can be effective; however, it introduces some key structural limitations. For instance, the same adapter is reused for every input, even though the optimal low-rank correction might vary substantially across examples and stages of computation. Additionally, LoRA fragments trainable capacity across layers by assigning each layer an independent update; consequently, recurring adaptation patterns must be relearned independently at multiple depths.
In many cases, the appropriate correction may depend not only on the current hidden representation, but also on information accumulated from earlier layers \citep{ShareLoRA2024, VBLoRA2024, AttentionResiduals2026}.

Our goal, then, is to retain the efficient, low-rank bottleneck of LoRA and the reusable structure that makes LoRA attractive, while allowing the effective correction to vary with the current example and the depth-wise state of the computation. We therefore introduce a shared memory of small-rank space-update atoms and a blockwise router that selects a sparse mixture of these atoms from the model's current low-rank state and its running depth summary.

We introduce a queryable global memory of rank-space update atoms and a lightweight blockwise router that assembles an example-dependent operator inside the LoRA bottleneck. The method learns a set of reusable transformations in the low-rank space and combines them using routing weights that depend on the current low-rank representation and a running summary of earlier blocks. The approach keeps the main advantages of low-rank adaptation from LoRA; efficiency and scalability \citep{hu2022lora}, and additionally, makes the adapter significantly more flexible because the effective update can vary with the input. Layers that need similar kinds of corrections can share the same update components, and information from earlier parts of the network can help determine which update directions to use later. 

Our contributions are:

\begin{itemize}[leftmargin=*]
    \item \textbf{A queryable update memory:} We introduce a globally shared memory bank of low-rank update atoms. This setup replaces static layer modifications with an adaptive mechanism that configures the parameter space based on the current context. To retrieve these updates, the model uses a routing mechanism that evaluates the current low-rank representation and a running depth summary of earlier blocks.
    
    \item \textbf{Instruction regularization:} We regularize the selection of update atoms using language instructions as a semantic prior \citep{TextToLoRA2025, DocToLoRA2026}. This method guides low-rank transformations toward semantically relevant updates without allowing unconstrained parameter changes.
    
    \item \textbf{Empirical gains in stability and accuracy:} We provide empirical evidence that the queryable update-memory formulation can improve held-out performance and optimization stability on noisy nonlinear regression tasks and several LLM fine-tuning benchmarks. Our experiments show these improvements hold across both noisy non-linear regression tasks and LLM benchmarks. Furthermore, the network achieves these gains while using a number of trainable parameters comparable to standard low-rank adaptation.
    
    \item \textbf{Theoretical guarantees for bounded updates:} We prove that our dynamic updates remain bounded and norm-controlled. Because the model forms the effective update from a convex mixture of shared atoms, the adapter gains flexibility without sacrificing the norm control that makes standard LoRA-style fine-tuning reliable. We also establish that the routing weights solve a principled optimization problem, ensuring language priors guide the model without arbitrarily overriding its internal state signal.
\end{itemize}

%% file: Chapters/2_Problem_setup.tex
\section{Problem Setup, Notation \& Related Work}

Formally, in fine-tuning, we seek to adapt a frozen pretrained network $f_{\theta_0}$ for a new downstream task. Let the baseline network contain $L$ layers with pre-trained weight matrices $\boldsymbol{W}^0_{\ell}$. Standard low-rank adaptation, such as LoRA, injects trainable rank-decomposition matrices $\boldsymbol{A}_{\ell}$ and $\boldsymbol{B}_{\ell}$ into each adapted layer $\ell$. Rather than updating the full model, LoRA applies a low-rank update to the weight matrix, $\Delta \boldsymbol{W}_{\ell} = \frac{\alpha}{r} \boldsymbol{B}_{\ell} \boldsymbol{A}_{\ell}$, where the bottleneck rank $r$ restricts the adapter's capacity. During the forward pass, the network updates its hidden representation $\boldsymbol{h}_{\ell}$ via $\boldsymbol{h}_{\ell+1} = \phi(\boldsymbol{W}^0_{\ell}\boldsymbol{h}_{\ell} + \Delta \boldsymbol{W}_{\ell}\boldsymbol{h}_{\ell})$.

As noted earlier, LoRA restricts fine-tuning updates to a fixed, layer-local low-dimensional subspace. While this makes LoRA highly scalable, this static formulation poses key structural limitations. For instance, the network must reuse the exact same adapter for every input, even when specific examples or depth-wise stages require different corrections. Furthermore, LoRA fragments its trainable capacity across independent layers. As a result, the model must independently relearn useful adaptation patterns that recur throughout the network. 

Several methods have attempted to modify LoRA to address some of these limitations. For example, magnitude-direction variants restructure the low-rank matrices to improve optimization dynamics \citep{liu2024dora}. Sharing-based variants reduce redundant local structure across layers \citep{ShareLoRA2024,Lily2024,VBLoRA2024}. While these methods improve training efficiency, their final adapters remain static. They do not allow the low-rank transformation to adapt jointly to the current hidden state and the preceding depth-wise computation inside the forward pass. Alternatively, mixture-of-experts approaches like MixLoRA dynamically construct input-tailored low-rank matrices to mitigate task interference, though they typically focus on multimodal settings rather than depth-wise state tracking \citep{MixLoRA2024, luo2024moelora}. 

Furthermore, other methods synthesize adapter parameters dynamically using external information. Text-to-LoRA-based architectures generate dense parameter updates directly from language instructions \citep{TextToLoRA2025, DocToLoRA2026}. Although generating hypernetwork weights directly from text embeddings provides additional flexibility, this unconstrained approach incurs quadratic scaling in parameter count due to the need to produce dense $d\times d$ weight matrices \citep{HyRA2025}.  Our method avoids this instability by using language as a semantic prior for retrieving fixed, reusable rank-space primitives. This preserves the efficiency of the low-rank bottleneck and yields norm-bounded updates while still supporting dynamic, context-sensitive adaptation.

\begin{figure}[h]
\centering
\fbox{\parbox{0.95\linewidth}{
\textbf{Intuition:} \textit{Instruction-regularized queryable LoRA updates.}
\vspace{2pt}
\begin{itemize}[leftmargin=12pt, itemsep=1pt, parsep=0pt, topsep=0pt, partopsep=0pt]
  \item \textbf{Frozen backbone} $f_{\theta_0}$: pretrained Transformer kept fixed during adaptation.
  \item \textbf{LoRA adapter} $\Delta W_\ell = B_\ell A_\ell$: low-rank update inserted into layer/projection $\ell$ instead of full fine-tuning.
  \item \textbf{Instruction} $s$: task or user directive that describes what behavior the adapter should induce.
  \item \textbf{Block query} $q_b$: a state and depth-dependent query formed from the layer prior, the current rank-space state, the accumulated depth summary, and optionally the instruction embedding.
  \item \textbf{Queryable rank-space memory}: a shared bank of atoms $\{C_m\}_{m=1}^M$ with keys $\{k_m\}_{m=1}^M$; the router selects a sparse convex mixture $S_b(c)=\sum_m \alpha_{b,m}(c)C_m$ inside the LoRA bottleneck.
  \item \textbf{Instruction regularization}: a language-induced prior over atoms is added to the routing logits, biasing retrieval toward semantically relevant rank-space atoms.
  \item \textbf{Global sharing}: the same $G_\phi$ amortizes adaptation across tasks, layers, and projections, giving LoRA more transfer than independent per-task adapters.
\end{itemize}
}}
\caption{Overview of the proposed method. The backbone model remains frozen while LoRA provides efficient low-rank adaptation. A query mechanism then retrieves shared update atoms based on the current computation, and language instructions regularize this retrieval process toward semantically meaningful adapter updates.}
\label{fig:iqgtu-intuition}
\end{figure}

In contrast to static LoRA, parameter-heavy hypernetworks, and methods that synthesize weights directly from text, our approach dynamically assembles updates from a globally shared memory of rank-space atoms. As Figure \ref{fig:iqgtu-intuition} outlines, we replace the rigid low-rank bottleneck with a queryable operator. During the forward pass, the model evaluates its current internal state alongside an attention-based summary of preceding depth-wise activations. This handling of the computational trajectory allows the model to selectively retrieve and reuse structural updates across both layers and tasks. By using depth-aware hidden states for routing and language for regularization, the model assembles context-sensitive, highly reusable corrections without the explosive parameter costs or instability of unconstrained generation. \S\ref{s:ApproachMain}  formalizes this architecture.

%% file: Chapters/3_Approach.tex

\section{Approach: Instruction Queryable Memory for Data-Adaptive PEFT}\label{s:ApproachMain}

This section formalizes the proposed instruction-regularized queryable update memory. Recall that standard LoRA restricts the fine-tuning update to a fixed, layer-specific low-rank bottleneck, $\Delta \boldsymbol{W}_\ell = \frac{\alpha_L}{r} \boldsymbol{B}_\ell \boldsymbol{A}_\ell$, where $\alpha_L$ denotes the LoRA scaling factor. To support dynamic, context-sensitive adaptation without the explosive parameter cost of hypernetworks, we replace this rigid transformation with a queryable operator $S_b(c)$ routed inside the bottleneck. The resulting input- and instruction-dependent update becomes:
\begin{equation}
\Delta \boldsymbol{W}_{\ell}(\boldsymbol{h}_{\ell}; c) = \frac{\alpha}{r}\boldsymbol{B}_{\ell} \bigl(\boldsymbol{I}_r + g_{\ell}S_b(c)\bigr) \boldsymbol{A}_{\ell},
\end{equation}
where $g_\ell = \sigma(\eta_\ell) \in (0,1)$ is a learned scalar gate. Because $S_b(c) \in \mathbb{R}^{r \times r}$ mixes the coordinates purely within the rank space, it is expressive enough to rotate and scale the adapter direction, yet compact enough to be drawn from a globally shared memory bank of atoms $\mathcal{C} = \{c_m\}_{m=1}^M$ paired with keys $\{\boldsymbol{k}_m\}_{m=1}^M$. This global sharing strategy is grounded in prior research showing that adaptation patterns frequently recur across network depths \citep{ShareLoRA2024}, but unlike ShareLoRA, which ties weights to form a static, albeit shared, layer-agnostic adapter, our architecture treats these shared components as a dynamically queryable vocabulary. By maintaining a global bank of rank-space atoms rather than fixed shared matrices, the network can flexibly retrieve and recombine these fundamental structural building blocks on the fly. This allows the model to construct an input- and depth-dependent transformation, breaking the rigidity of static parameter sharing while preserving its parameter efficiency.

\paragraph{Blockwise Routing and State Summarization.}
To amortize routing and encourage consistent structural adaptations, layers are partitioned into continuous blocks $\{\mathcal{B}_b\}_{b=1}^B$ \citep{AttentionResiduals2026}. The operator $S_b(c)$ is computed once per block using a router conditioned on the layer prior, the current block-entry state, the depth summary of earlier blocks, and, optionally, the external language instruction $c$.

Let $\boldsymbol{s}_{\ell_b}^{\mathrm{entry}} = \boldsymbol{A}_{\ell_b}\mathcal{D}_{\ell_b}(\boldsymbol{h}_{\ell_b})$ be the rank-space state at the entry of block $b$, and $e(c) \in \mathbb{R}^{d_c}$ be a frozen text embedding of the instruction. We first construct an instruction-conditioned pre-query:
\begin{equation}
\boldsymbol{q}_b^{(0)} = \boldsymbol{w}_{\ell_b} + \boldsymbol{Q}_{\mathrm{cur}} \boldsymbol{s}^{\mathrm{entry}}_{\ell_b} + \lambda_{\mathrm{ctx}} \boldsymbol{Q}_{\mathrm{ctx}} e(c),
\end{equation}
where $\boldsymbol{w}_{\ell_b}$ is the stable layer prior. Rather than discarding past computation, we allow blocks to conditionally retrieve from earlier layers via an attention-based depth summary. Defining the block average $\bar{\boldsymbol{s}}_i = \frac{1}{|\mathcal{B}_i|}\sum_{\ell \in \mathcal{B}_i} \boldsymbol{s}_{\ell}$, the running summary is $\boldsymbol{u}_{b-1}^{\mathrm{att}} = \sum_{i=1}^{b-1} \beta_{i\to b}\,\bar{\boldsymbol{s}}_i$. Here, the attention weights $\beta_{i\to b}$ are proportional to $\exp\bigl(\langle \widehat{\boldsymbol{Q}}_{\mathrm{dep},q}\boldsymbol{q}_b^{(0)}, \widehat{\boldsymbol{Q}}_{\mathrm{dep},k}\bar{\boldsymbol{s}}_i \rangle / (\sqrt{d_k}T_{\mathrm{dep}})\bigr)$, where $\widehat{\boldsymbol{x}} = \operatorname{RMSNorm}(\boldsymbol{x})$ \citep{zhang2019root, vaswani2017attention}. The final state query integrates this historical context: $\boldsymbol{q}_b = \boldsymbol{q}_b^{(0)} + \boldsymbol{Q}_{\mathrm{dep}}\boldsymbol{u}_{b-1}^{\mathrm{att}}$.

\paragraph{Instruction Regularization.}
Unlike Text-to-LoRA methods that generate dense parameter updates solely from language \citep{TextToLoRA2025}, this approach uses language as an optional semantic prior to retrieve fixed, reusable rank-space primitives. We compute a language prior distribution over the atoms $p_m(c) \propto \exp\bigl(\langle \widehat{\boldsymbol{R}}_{\mathrm{ctx}}e(c), \widehat{\boldsymbol{k}}_m \rangle / (\sqrt{d_k}T_{\mathrm{lang}})\bigr)$. We then regularize the state-dependent routing logits $\zeta_{b,m} = \langle \widehat{\boldsymbol{q}}_b, \widehat{\boldsymbol{k}}_m \rangle / (\sqrt{d_k}T_{\mathrm{attn}})$ with this instruction prior:
\begin{equation}
\tilde{\zeta}_{b,m}(c) = \zeta_{b,m} + \tau_{\mathrm{lang}}\log p_m(c),
\end{equation}
where $\tau_{\mathrm{lang}}$ controls the strength of the language conditioning. The routed operator $S_b(c)$ is finally formed via a sparse top-$k$ convex combination: $S_b(c) = \sum_{m \in I_b} \alpha_{b,m}^{(\mathrm{top}\,k)} c_m$, where $\boldsymbol{\alpha}_b$ is the softmax over the $k$ largest values of $\tilde{\zeta}_{b,m}(c)$. Setting $\tau_{\mathrm{lang}}=0$ and $\lambda_{\mathrm{ctx}}=0$ recovers a purely state-dependent dynamic adapter, while setting $g_\ell=0$ recovers standard LoRA, as we show below. 

Importantly, external instructions are not required for routing. The memory is already queryable from the model state through the layer prior, the current rank-space activation, and the accumulated depth summary. Language does not generate adapter weights. Instead, it influences routing through two controlled channels: the term $\lambda_{\mathrm{ctx}}Q_{\mathrm{ctx}}e(c)$ lets the instruction shape the block query and therefore the depth/state-dependent route, while $\tau_{\mathrm{lang}}\log p_m(c)$ provides an explicit atom-level prior. Both mechanisms bias retrieval from the same fixed memory bank of reusable rank-space atoms. Figure \ref{fig:VizAbstractQuery} provides an architectural overview; further details can be found in \S\ref{s:MethodDetails}.

\begin{figure}[h]
 \begin{center}  \includegraphics[width=0.8\linewidth]{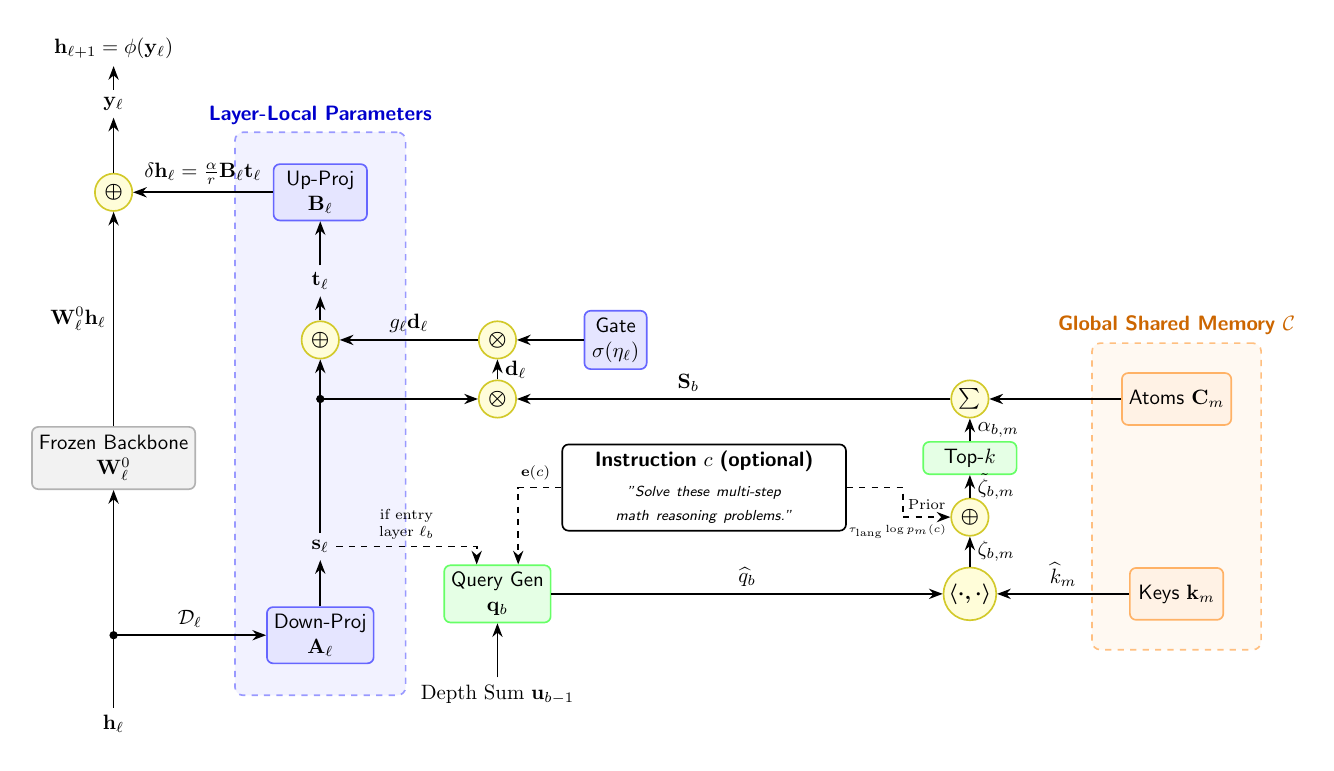}
\caption{Instruction-regularized global atomic updates of LoRA. Standard LoRA would only have the down- and up-projection without incorporating information from global memory or from instructions.}\label{fig:VizAbstractQuery}
\end{center}
\end{figure}

%% file: Chapters/4_Empirical_study.tex
\section{Empirical Evidence}\label{s:EmpiricsMain}

The approach above shows how shared queryable memory can inject dynamic capacity into a low-rank bottleneck. This architectural flexibility, however, introduces routing complexity that could disrupt optimization if the network fails to learn meaningful update atoms. We therefore designed empirical evaluation to test whether state-dependent routing translates to measurable generalization gains. Because isolating this routing mechanism directly within a large language model is difficult, we structure our analysis in two stages. We first evaluate the core queryable memory on synthetic, highly non-convex regression tasks, which isolates its capacity to adapt to shifting local structures. 

\subsection{Experiments on Synthetic 2-D Non-Convex Functions}

In this synthetic experiment, we evaluate the queryable adapter on nine two-dimensional stochastic non-convex benchmark functions drawn from \cite{simulationlib} and compare it against representative PEFT baselines.

Table \ref{table:NonConvexTrain} presents the post-training loss results for state-of-the-art LoRA-based PEFT methods, as compared against full fine-tuning with SOAP \citep{vyas2024soap}. Similarly, Table \ref{table:NonConvexTest} presents the test loss results. The number of epochs for pre-training was 300 and for post-training 5000, on dataset sizes of 3,000 and 1,200, respectively. The learning rate was 3\!$\times$\!10$^{-3}$ / 5\!$\times$\!10$^{-4}$ for pre-/post-training, respectively; the neural backbone was a standard 8-depth, 256-hidden-dimension Transformer \citep{vaswani2017attention}; we use AdamW \citep{loshchilov2017decoupled} as the optimizer. In the pre-training regime, data are generated from the distribution of noisy two-dimensional regression samples used to learn the frozen backbone; in the post-training regime, data are freshly sampled from the target stochastic non-convex function with shifted parameters (in particular, we vary the coefficients on the non-linear terms in the functions and rotate the outputs). The queryable approach keeps the number of trainable parameters within 10\% of LoRA across all runs.

\begin{table}[htbp]
    \renewcommand{\arraystretch}{1.2} 
    \centering
    \scriptsize
    \caption{Summary -- \textbf{MSE Training Loss}. Comparative post-training performance for modeling stochastic non-convex functions over $5$ independent runs. The values following ``$\pm$'' denote the standard deviation across independent runs. \textbf{Bold} = best overall result.
    }
    \label{table:NonConvexTrain}
    \resizebox{0.98\textwidth}{!}{%
    \begin{tabular}{lcccccc}
        \toprule
        \textbf{Function} & \textbf{LoRA} \citep{hu2022lora} & \textbf{DoRA} \citep{liu2024dora} & \textbf{HyRA} \citep{HyRA2025} & \textbf{RepLoRA} \citep{truong2025replora} & \textbf{DoRAN} \citep{diep2025doran} & \textbf{Ours} \\
        \midrule
       Ackley           & 0.0754\std{0.0225} & 0.1544\std{0.0295} & 40.974\std{75.620} & 0.0632\std{0.0278} & 0.1159\std{0.0236} & \textbf{0.0559\std{0.0159}} \\
        Dropwave          & 24.0509 \std{42.8973} & 266.2681 \std{577.2793} & 1437.6514 \std{193.1731} & 1.3918 \std{1.3943} & 259.8996 \std{580.7419} & \textbf{0.2527 \std{0.1777}} \\
        Langermann      & 2.2579 \std{2.6938} & 2.4252 \std{2.2126} & 3.7290 \std{2.2653} & 2.2295 \std{2.7166} & \textbf{1.5239 \std{2.1633}} & 1.8061 \std{2.3035} \\
        Levy              &  0.0124 \std{0.0053} & 0.01388 \std{0.0065} & 248.3561 \std{154.4848} & 0.0027 \std{0.0010} & 0.0073 \std{0.0034} & \textbf{0.0009 \std{0.0003}} \\
        Matyas            & 0.0034\std{0.0004} & 0.0061\std{0.0006} & 1500.0727\std{83.0140} & 0.0041 \std{0.0003} & 0.0125 \std{0.0025} & \textbf{0.0032 \std{0.0003}} \\
        Michalewicz       & 0.0590 \std{0.0231} & 0.0745 \std{0.0250} & 0.1417 \std{0.0171} & 0.0465 \std{0.0368} & 0.0454 \std{0.0408} & \textbf{0.0380 \std{0.0202}} \\
        Rastrigin & 22.4812 \std{4.5454} & 25.4298 \std{7.0738} & 998.8924 \std{1676.4226} & 9.0337 \std{5.2511} & 5.5401 \std{1.4612} & \textbf{3.3663 \std{2.6077}} \\
        Sin-Cos           & 0.0002\std{0.0001} & 0.0003\std{0.0001} & 0.7841\std{0.2493} & 0.0001\std{0.0001} & 0.0003\std{0.0001} & \textbf{0.0001\std{0.0000}} \\
        Styblinski-Tang   & 2.0728 \std{0.5664} & 2.2513 \std{0.4360} & 29677.6550 \std{43197.9384} & 0.6690 \std{0.1545} & 2.5213 \std{0.7855} & \textbf{0.3862 \std{0.0965}} \\
        \bottomrule
    \end{tabular}
}
\end{table}

\begin{table}[htbp]
    \renewcommand{\arraystretch}{1.2} 
    \centering
    \scriptsize
    \caption{Summary -- \textbf{MSE Test Loss}. Comparative post-training performance for modeling stochastic non-convex functions over $5$ independent runs. The values following the ``$\pm$'' show the standard deviation of the mean across runs. \textbf{Bold} = best overall result.
    }
    \label{table:NonConvexTest}
    \resizebox{0.98\textwidth}{!}{%
    \begin{tabular}{lcccccc}
        \toprule
        \textbf{Function} & \textbf{LoRA} \citep{hu2022lora} & \textbf{DoRA} \citep{liu2024dora} & \textbf{HyRA} \citep{HyRA2025} & \textbf{RepLoRA} \citep{truong2025replora} & \textbf{DoRAN} \citep{diep2025doran} & \textbf{Ours} \\
        \midrule
        Ackley           & 0.3734\std{0.0458} & 0.3716\std{0.0475} & 39.4747\std{73.3841} & \textbf{0.3696\std{0.0482}} & 0.3753\std{0.0433} & 0.3726\std{0.0495} \\
        Dropwave          & 64.5079 \std{57.4368} & 374.8006 \std{786.5186} & 1677.5055 \std{242.8430} & 12.7792 \std{10.4422} & 357.0902 \std{795.7947} & \textbf{2.2263 \std{1.5457}} \\
        Langermann & 2.2596 \std{2.3986} & 2.5835 \std{1.9773} & 3.7396 \std{2.3199} & 2.2306 \std{2.4152} & \textbf{1.6682 \std{2.0511}} & 1.9344 \std{2.1697} \\
        Levy              & 1.3273 \std{0.2269} & 1.3395 \std{0.2296} & 219.8848 \std{113.1329} & 1.5735 \std{0.1907} & 1.4383 \std{0.2535} & \textbf{1.3028 \std{0.2163}} \\
        Matyas            & 0.0255 \std{0.0051} & 0.0231 \std{0.0018} & 1534.4436 \std{57.6542} & 0.0233 \std{0.0026} & 0.0258 \std{0.0034} & \textbf{0.0229 \std{0.0035}} \\
        Michalewicz   & 0.1301 \std{0.0076} & 0.1299 \std{0.0074} & 0.1321 \std{0.0081} & 0.1297 \std{0.0072} & \textbf{0.1277 \std{0.0087}} & 0.1288 \std{0.0076} \\
        Rastrigin & \textbf{98.7647 \std{4.4966}} & 98.8097 \std{3.6418} & 1048.7748 \std{1762.0819} & 99.2913 \std{4.0321} & 101.5798 \std{1.5538} & 99.6905 \std{4.3598} \\
        Sin-Cos           & 0.0039\std{0.0004} & 0.0040\std{0.0003} & 0.8615\std{0.2558} & 0.0039\std{0.0004} & 0.0046\std{0.0008} & \textbf{0.0039\std{0.0004}} \\
        Styblinski-Tang  &  67.2361 \std{21.6650} & 75.5779 \std{28.1219} & 30586.1356 \std{42757.6886} & 72.1448 \std{27.7898} & 70.1036 \std{21.2876} & \textbf{62.2928 \std{17.7273}} \\
        \bottomrule
    \end{tabular}
}
\end{table}

Here, inf refers to the case in which the MSE Loss was above $100000$. Overall, performance gains here do not appear to be uniformly explained by the degree of non-convexity alone. Instead, the queryable adapter appears most beneficial on targets with pronounced local heterogeneity, especially Dropwave, where a single static low-rank correction is likely too rigid. On smoother targets or highly regular repeated landscapes such as Sin-Cos, Matyas, Ackley, and Rastrigin, the advantage is smaller, suggesting that static PEFT might already capture the dominant structure. Overall, these results indicate that global queryable rank-space atoms can improve adaptation when the target function contains heterogeneous local structure.

\S\ref{s:AdditionalEmpiricalResults} shows even stronger relative performance in a deep, narrow architecture with $32$ layers and width $32$, suggesting that the queryable atomic updates may provide an optimization benefit in deeper networks as well: because the routed operator is applied as a residual transformation inside the low-rank bottleneck and its atoms are shared across blocks, gradient information from many depths can reinforce the same reusable rank-space primitives rather than being confined to isolated layer-specific adapters. 
We will return to this point below. 


\subsection{LLM Fine-Tuning Results}

We next compare our approach with representative static, routed, and generated PEFT baselines for language-model fine-tuning. The following table compares performance. The number of post-training epochs is $25$; the learning rate is $0.0002$. To ensure consistent evaluation, we selected the final checkpoint for each method based on the highest accuracy achieved on the training set. See \S\ref{s:Licenses} for more information about datasets used here.

In \textsc{General} evaluation tasks, we find that instruction-queryable routing yields consistent held-out gains: despite near-saturated training accuracy across methods, it outperforms LoRA on every benchmark and is the strongest method on six of seven tasks. This suggests that the language prior improves generalization and routing stability rather than simply adding memorization capacity.

\begin{table*}[htb]
    \centering
    \caption{Comparative post-training performance for \textsc{General} tasks}
    \label{table:4}
    \scriptsize
    \setlength{\tabcolsep}{5pt}
    \resizebox{0.8\textwidth}{!}{%
    \begin{tabular}{llccccccccc}
        \toprule
        \multirow{2}{*}{\textbf{Method}} & \multirow{2}{*}{\textbf{Type}} 
        &\multicolumn{1}{c} {\textbf{GPQA-Diamond} }
        & \multicolumn{1}{c}{\textbf{MBPP}} 
        & \multicolumn{1}{c}{\textbf{ARC}}
        &\multicolumn{1}{c}  {\textbf{Super-GLUE}}
        &\multicolumn{1}{c}  {\textbf{OpenBookQA}}
        &\multicolumn{1}{c}  {\textbf{RACE}}
        &\multicolumn{1}{c}  {\textbf{HellaSwag}}\\
        \cmidrule(lr){3-3} \cmidrule(lr){4-4} \cmidrule(l){5-5} \cmidrule(l){6-6} \cmidrule(l){7-7} \cmidrule(l){8-8} \cmidrule(l){9-9}
        & 
        & \makecell{\textbf{Qwen0.5B}}
        & \makecell{\textbf{Qwen0.5B}}
        & \makecell{\textbf{Qwen0.5B}}
        & \makecell{\textbf{Qwen0.5B}}
        & \makecell{\textbf{Qwen0.5B}}
        & \makecell{\textbf{Qwen0.5B}}
        & \makecell{\textbf{Qwen0.5B}}\\
        \midrule
        
        \multirow{1}{*}{LoRA \citep{hu2022lora}}
            & Test Accuracy  & 0.253 & 0.290 & 0.595 & 0.793 & 0.700 & 0.595 & 0.617 \\
        \midrule
        
        \multirow{1}{*}{DoRA \citep{liu2024dora}}
            & Test Accuracy  & 0.273 & 0.290 & 0.545 & 0.680 & 0.592 & 0.502 & 0.514  \\
        \midrule

        \multirow{1}{*}{RepLoRA \citep{truong2025replora}}
            & Test Accuracy & 0.273 &  0.295 & 0.585 & 0.789 & 0.696 & 0.585 &  \textbf{0.627} \\
        \midrule
        
        \multirow{1}{*}{Queryable LoRA}
            & Test Accuracy  & 0.293 & 0.290 & \textbf{0.656} & 0.795 & 0.696 & 0.585 & 0.602 \\
        \midrule
        
        \multirow{1}{*}{Instruction-Queryable LoRA}
            & Test Accuracy  & \textbf{0.323} & \textbf{0.300} & 0.651 & \textbf{0.797} & \textbf{0.708} & \textbf{0.599} & 0.623\\
        \bottomrule
    \end{tabular}
    }
\end{table*}

In \textsc{Mathematics}-related tasks, we find a somewhat more heterogeneous but still supportive pattern: queryable and instruction-queryable variants improve or match the strongest test accuracy on several reasoning benchmarks, especially GSM8K and Numina-Math. This pattern is consistent with the view that shared routed atoms are most useful when reasoning tasks benefit from reusable depth-wise adaptation: the largest gains appear on some multi-step math benchmarks.

\begin{table*}[htb]
    \centering
    \caption{Comparative post-training performance for \textsc{Mathematics} reasoning tasks.}
    \label{s:PostTrainingMath}
    \scriptsize
    \setlength{\tabcolsep}{5pt}
    \resizebox{0.8\textwidth}{!}{
    \begin{tabular}{llccccccc}
        \toprule
        \multirow{2}{*}{\textbf{Method}} & \multirow{2}{*}{\textbf{Type}} 
        &\multicolumn{2}{c}{\textbf{gsm8k}}
        & \multicolumn{2}{c}{\textbf{MATH}} 
        &\multicolumn{2}{c}{\textbf{Orca-Math}}
        & \multicolumn{1}{c}{\textbf{Numina-Math}} \\
        \cmidrule(lr){3-4} \cmidrule(lr){5-6} \cmidrule(lr){7-8} \cmidrule(l){9-9}
        & 
        & \makecell{\textbf{Qwen0.5B}}
        & \makecell{\textbf{Mistral7B}}
        & \makecell{\textbf{Qwen0.5B}}
        & \makecell{\textbf{Mistral7B}}
        & \makecell{\textbf{Qwen0.5B}}
        & \makecell{\textbf{Mistral7B}}
        & \makecell{\textbf{Qwen0.5B}}\\
        \midrule
        
        \multirow{1}{*}{LoRA \citep{hu2022lora}}
            & Test Accuracy  & 0.312& \textbf{0.375} & 0.133& 0.153  & 0.304& 0.320 & 0.203 \\
        \midrule
        
        \multirow{1}{*}{DoRA \citep{liu2024dora}}
            & Test Accuracy  & 0.312 & \textbf{0.375} & 0.078 & 0.148 & 0.234& 0.203 & 0.203\\
        \midrule

        \multirow{1}{*}{RepLoRA \citep{truong2025replora}}
            & Test Accuracy & 0.285& 0.359  & 0.141& 0.156 & 0.304& 0.289& 0.170\\
        \midrule
        
        \multirow{1}{*}{Queryable LoRA}
            & Test Accuracy  & 0.292& \textbf{0.375} & \textbf{0.156}& 0.141 & \textbf{0.344}& \textbf{0.352}& 0.219\\
        \midrule
        
        \multirow{1}{*}{Instruction-Queryable LoRA}
            & Test Accuracy  & \textbf{0.324}& \textbf{0.375} & 0.141& \textbf{0.158}  & \textbf{0.344} & 0.328& \textbf{0.234}\\
        \bottomrule
    \end{tabular}
    }
\end{table*}

For additional results on a broader battery of $\leq 1$B models, including Qwen3, LiquidAI LFM2.5 variants, AMD ReasonLite variants, IBM Granite-4.0-350M, and HuggingFaceTB SmolLM2-360M-Instruct, see Table \ref{tab:model-battery-accuracy-full-matrix}. Here, checkpoints are selected with lowest validation fine-tuning loss. Performance there is similar (equaling or exceeding LoRA in 34 out of 39 cases), indicating the queryable approach performs well relative to the canonical PEFT method. 

To unpack some of the dynamics found in these overall accuracy scores, we also examine whether instruction-regularized queryable updates improve the adapter's optimization dynamics. Figure \ref{fig:grad-norm-layer} reports the per-layer adapter gradient norms after post-training, comparing LoRA, queryable routing without instruction regularization, and the full instruction-queryable method. The instruction-queryable adapter consistently receives stronger gradient signals across a wider range of layers, especially in middle and late layers, where static LoRA and the non-instruction queryable variant often receive weaker signals. These observed gradient profiles suggest that instruction-conditioned routing keeps the dynamic adapter pathway active across depth (see Figure \ref{fig:grad-concentration}).

\begin{figure}[htb]
    \centering
    \begin{minipage}[t]{0.38\textwidth}
        \centering
        \vspace{0pt}
        \includegraphics[width=\linewidth]{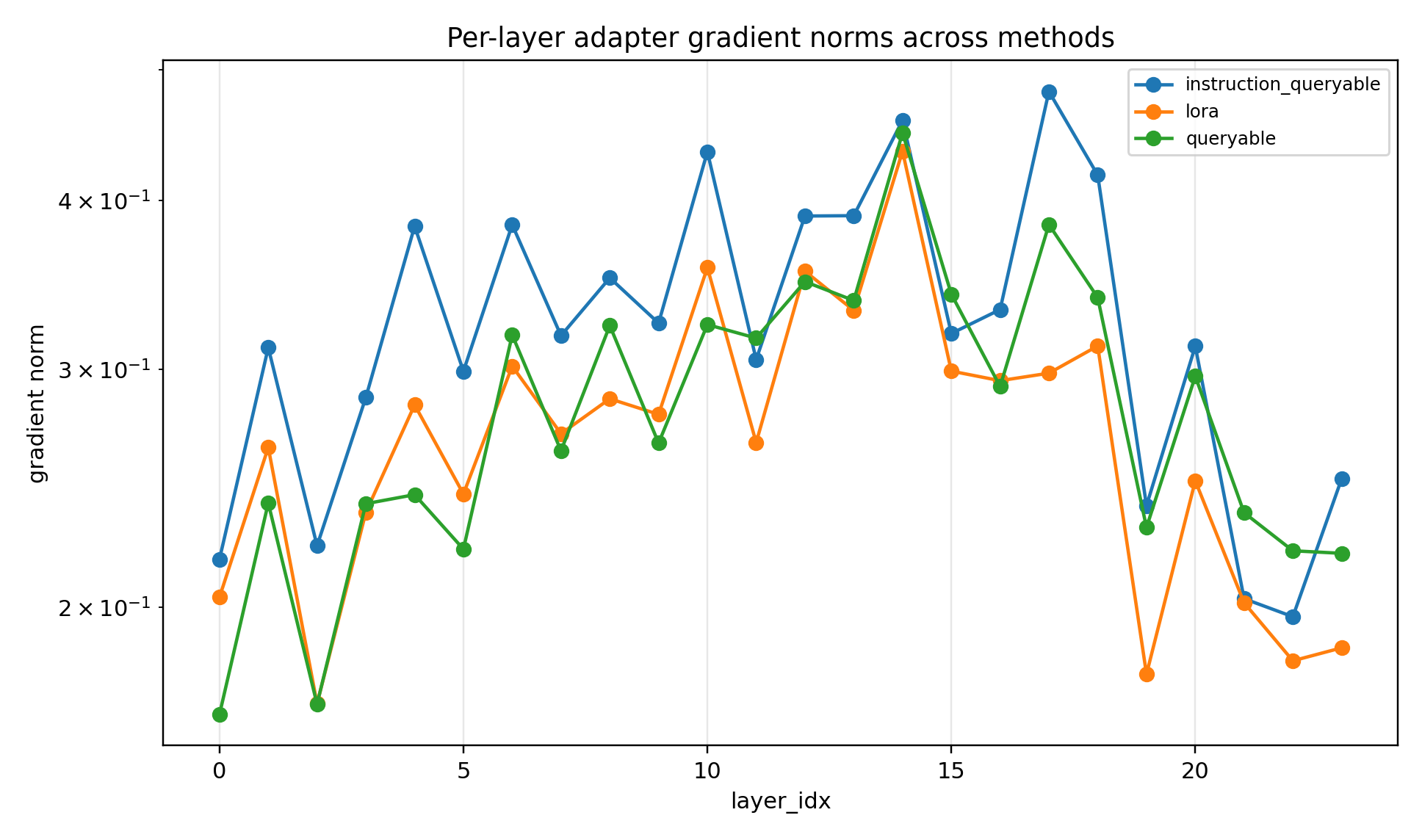}
        \caption{Per-layer adapter gradient norms across methods. The instruction-queryable method maintains stronger gradient flow across many layers.}
        \label{fig:grad-norm-layer}
    \end{minipage}%
    \hfill
    \begin{minipage}[t]{0.58\textwidth}
        \centering
        \vspace{0pt}
        \captionof{table}{Inference-time efficiency of trained adapters. Lower is better for forward latency, latency overhead, and FLOP overhead; higher is better for generation throughput. Timings exclude tokenization and dataloader construction. The proposed queryable variants are not faster than the static LoRA reference, but they are substantially faster than the more expressive RepLoRA, HyRA, and DoRAN baselines while adding only negligible adapter-side arithmetic.}
        \label{tab:inference_efficiency}
        \small
        \setlength{\tabcolsep}{2.5pt}
        \resizebox{\linewidth}{!}{%
        \begin{tabular}{lrrrrrrr}
        \toprule
        \textbf{Measurement} & \textbf{LoRA} & \textbf{DoRA} & \textbf{RepLoRA} & \textbf{HyRA} & \textbf{DoRAN} & \textbf{Queryable} & \textbf{Instr.-Queryable} \\
        \midrule
        Forward/prefill lat. (ms) $\downarrow$ & 33.1 & 30.3 & 53.0 & 83.3 & 51.0 & \textbf{42.3} & 46.5 \\
        Fwd. overhead vs LoRA (\%) $\downarrow$ & 0.0 & -8.4 & 60.2 & 151.7 & 54.1 & \textbf{27.9} & 40.5 \\
        Throughput (new toks/s) $\uparrow$ & 31.8 & 30.3 & 18.0 & 11.3 & 18.7 & \textbf{22.0} & 20.1 \\
        FLOP overhead vs LoRA (\%) $\downarrow$ & 0.0 & 0.0 & 0.0 & 0.0 & 0.0 & 0.6 & 0.6 \\
        \bottomrule
        \end{tabular}%
        }
    \end{minipage}
\end{figure}

To further explore our method's capacity for structured parameter reuse and stable adaptation across shifting domains, we evaluated the queryable updates in a sequential continual-learning setting without resetting the adapter's components between tasks. A detailed discussion of these dynamics—illustrating how the model maintains evaluation performance and controls atom route drift across consecutive benchmarks—can be found in \S\ref{s:Routing}. Overall, those results show that the instruction-queryable adapter successfully balances knowledge retention with flexibility. The model maintains a sparse, non-uniform memory access pattern to preserve reusable structures, while adapting to new tasks through localized, concentrated route drift rather than arbitrarily overwriting past atoms.

\paragraph{Inference-time Analysis.} We also measure the inference-time cost of the optimized trained adapters under matched model, target-module, rank, batch-size, sequence-length, and decoding settings. This diagnostic is included to separate two questions: plain LoRA is expected to remain the fastest static adapter, while the relevant systems question is whether the proposed dynamic routing is competitive with more expressive PEFT baselines. Table~\ref{tab:inference_efficiency} shows that the optimized queryable adapter adds moderate latency relative to LoRA, but is faster than RepLoRA, HyRA, and DoRAN in both forward/prefill latency and autoregressive generation throughput. The adapter-side arithmetic overhead remains negligible, indicating that the remaining cost comes primarily from dynamic routing and kernel dispatch rather than from the low-rank operator itself.

\paragraph{Ablations.} In \ref{s:Ablation_studies}, we vary the instruction-queryable adapter rank $r$, global memory size $M$, and routing sparsity $k$, and report the resulting accuracy-latency and accuracy-parameter Pareto frontiers. 

%% file: Chapters/5_Theoretical_results.tex
\section{Theoretical Results}\label{s:Theory}

This section establishes the structural guarantees behind the instruction-regularized queryable update mechanism. Theorem~\ref{thm:variational_instruction_retrieval} first shows that the instruction-regularized router is a principled mechanism rather than an ad hoc text-to-weight formulation: its routing weights uniquely solve a variational problem that balances state utility against a KL penalty toward the instruction-induced semantic prior. Thus, the router selects update atoms that are useful for the current hidden state while keeping retrieval close to the semantic preference encoded by the instruction. This separation showcases that language biases a stable retrieval process over reusable learned update components. 

We use the update, routed operator, block query, language prior, and joint routing logits from \S\ref{s:ApproachMain}. For a block \(b\), let \(I=I_b\subseteq[M]\) denote the active atom set used by the router; for dense softmax, \(I=[M]\). We write \(\boldsymbol{\alpha}_{b,I}(c)\in\Delta(I)\) for the routing distribution restricted to \(I\), where \(\Delta(I)\) is the probability simplex over \(I\). Since the router uses logits \(\widetilde{\zeta}_{b,m}(c)=\zeta_{b,m}+\tau_{\rm lang}\log p_m(c)\), define the corresponding tempered instruction prior on \(I\) by:  $\boldsymbol{\pi}^{(\tau)}_{I,m}(c)
=
\frac{p_m(c)^{\tau_{\rm lang}}}{\sum_{j\in I}p_j(c)^{\tau_{\rm lang}}}$ for $m \in I$. Here \(\boldsymbol{\zeta}_{b,I}\) collects the state logits on \(I\), \({\rm KL}(\cdot\|\cdot)\) is KL divergence on \(\Delta(I)\), and \(H(\boldsymbol{a})=-\sum_{m\in I}a_m\log a_m\). All norms stated are operator norms unless explicitly marked by \(\|\cdot\|_F\). The norm statement assumes bounded atoms and factors, \(\|\mathbf{C}_m\|\le R_C\), \(\|\mathbf{A}_\ell\|\le R_A\), \(\|\mathbf{B}_\ell\|\le R_B\), and bounded block summaries \(\|\bar{\mathbf{s}}_i\|\le R_s\). For the gradient identities, during a fixed forward pass we write \(\mathbf{d}_\ell=\mathbf{S}_b(c)\mathbf{s}_\ell\), \(\mathbf{t}_\ell=\mathbf{s}_\ell+g_\ell\mathbf{d}_\ell\), and \(\mathbf{r}_\ell=\nabla_{\mathbf{t}_\ell}\mathcal{L}\).

\begin{theorem}[Variational characterization of instruction-regularized retrieval]
\label{thm:variational_instruction_retrieval}
Fix a block $b$, instruction $c$, and active set $I$. The routing distribution $\boldsymbol{\alpha}_{b,I}(c)$ defined in~\eqref{eq:app_active_router} is the unique solution of
\begin{equation}
\boldsymbol{\alpha}_{b,I}(c) =
\arg\max_{\boldsymbol{a}\in\Delta(I)} \left\{ \langle \boldsymbol{a},\boldsymbol{\zeta}_{b,I}\rangle -
\mathrm{KL}\!\left(\boldsymbol{a}\,\middle\|\,\boldsymbol{\pi}^{(\tau)}_{I}(c)\right) \right\}
\label{eq:variational_router_exact}
\end{equation}
where $\Delta(I)$ is the probability simplex on $I$. Equivalently,
\begin{equation}
\boldsymbol{\alpha}_{b,I}(c) = \arg\max_{\boldsymbol{a}\in\Delta(I)}
\left\{ \langle \boldsymbol{a},\boldsymbol{\zeta}_{b,I}+\tau_{\mathrm{lang}}\log \mathbf{p}_{I}(c)\rangle +H(\boldsymbol{a}) \right\}
\label{eq:entropy_router_exact}
\end{equation}
where $H(\boldsymbol{a})=-\sum_{m\in I}a_m\log a_m$.
\end{theorem}

The state-only router may prefer one update atom because it best matches the current hidden representation, while the language prior may prefer a different atom because it better matches the instruction. Corollary \ref{cor:state_prior_tradeoff_main} demonstrates that this tension is controlled, i.e., if the instruction prior gives reasonable support to the state-preferred atom, then the instruction-regularized router cannot lose much state utility. Language can guide the adapter toward semantically meaningful updates without arbitrarily overriding the model’s state signal.

\begin{corollary}[State-prior tradeoff bound]
\label{cor:state_prior_tradeoff_main}
Let $m_b^{\star}\in\arg\max_{m\in I}\zeta_{b,m}$. Under the hypotheses of ~\ref{thm:variational_instruction_retrieval},
\begin{equation} 0 \le \max_{m\in I}\zeta_{b,m} - \langle \boldsymbol{\alpha}_{b,I}(c),\boldsymbol{\zeta}_{b,I}\rangle \le \log\frac{1}{\pi^{(\tau)}_{I,m_b^{\star}}(c)}
\label{eq:state_prior_tradeoff_main}
\end{equation}
\end{corollary}

\ref{thm:norm_controlled_updates} gives the main stability guarantee for the approach. Although the adapter is dynamic and varies across inputs, blocks, and instructions, the effective update remains bounded because it is composed of a convex combination of bounded shared atoms. The result supports the claim that the method gains flexibility without sacrificing the norm control that makes LoRA-style fine-tuning reliable. It also justifies the depth summary mechanism of attention, i.e., the attention over previous blocks can selectively reuse earlier computation without causing uncontrolled growth in the routed update.

\begin{theorem}[Norm-controlled dynamic updates and bounded depth summaries]
\label{thm:norm_controlled_updates}
Under assumption \ref{ass:app_bounded_atoms_factors}, the following statements hold for every instruction $c$ and block $b$.
\begin{enumerate}
    \item The routed operator belongs to the convex hull of the atom bank: $  \mathbf{S}_b(c)\in\operatorname{conv}\{\mathbf{C}_1,\dots,\mathbf{C}_M\}$and its operator norm is uniformly bounded: $\|\mathbf{S}_b(c)\| \le R_C$.
    \item For every adapted layer $\ell\in B_b$,
\begin{equation}
    \|\Delta\mathbf{W}^{\ell}(\mathbf{h}_{\ell};b,c)\| \le \frac{\alpha_{\mathrm{L}}}{r}\, \|\mathbf{B}_{\ell}\|\ (1+g_{\ell}R_C)\,\|\mathbf{A}_{\ell}\| \le \frac{\alpha_{\mathrm{L}}}{r}\,R_B(1+R_C)R_A
    \label{eq:update_norm_bound_main}
    \end{equation}
    Consequently, for every vector $\mathbf{x}$,
    \begin{equation}
    \|\Delta\mathbf{W}^{\ell}(\mathbf{h}_{\ell};b,c)\mathbf{x}\| \le \frac{\alpha_{\mathrm{L}}}{r}\,R_B(1+R_C)R_A\,\|\mathbf{x}\|
    \label{eq:update_vector_bound_main}
    \end{equation}
    \item If the attention-style depth summary is given by: $  \mathbf{u}^{\mathrm{att}}_{b-1} = \sum_{i=1}^{b-1}\beta_{i\to b}\bar{\mathbf{s}}_i$, where, $ \beta_{i\to b} \ge 0$, and $   \sum_{i=1}^{b-1} \beta_{i\to b} = 1$. Then
    \begin{equation}
    \|\mathbf{u}^{\mathrm{att}}_{b-1}\| \le \max_{1\le i\le b-1}\|\bar{\mathbf{s}}_i\| \le R_s
    \label{eq:u_att_bound_main}
    \end{equation}
\end{enumerate}
\end{theorem}

\ref{thm:exact_blockwise_gradient_main} specifies that every atom receives a clean gradient signal through the block in which it is retrieved, and that this signal decomposes into simple low rank contributions from the layers using that block’s routed operator. The global memory receives structured supervision from many layers and can learn reusable update directions that are useful across depth, inputs, and instructions.

\begin{theorem}[Exact blockwise gradient factorization]
\label{thm:exact_blockwise_gradient_main}
Consider $\mathbf{S}_b(c)$ as the blockwise routed operator used by all layers in $B_b$. Then
\begin{equation}
\nabla_{\mathbf{S}_b(c)}\mathcal{L} = \sum_{\ell\in B_b}g_{\ell}\mathbf{r}_{\ell}\mathbf{s}_{\ell}^{\top}
\label{eq:exact_grad_S_main}
\end{equation}
Consequently, if the atom values $\{\mathbf{C}_m\}_{m=1}^{M}$ are independent of the routing keys, then the direct value-path gradients are
\begin{equation}
\nabla_{\mathbf{C}_m}\mathcal{L} = \alpha_{b,m}(c)\nabla_{\mathbf{S}_b(c)}\mathcal{L} \text{ and } \frac{\partial \mathcal{L}}{\partial \alpha_{b,m}(c)} = \left\langle \nabla_{\mathbf{S}_b(c)}\mathcal{L},\mathbf{C}_m\right\rangle_F
\label{eq:exact_grad_C_main}
\end{equation}
$\forall \ m=1,2,3\dots,M$. The gate parameter satisfies
\begin{equation}
\frac{\partial \mathcal{L}}{\partial \eta_{\ell}} = \sigma(\eta_{\ell})(1-\sigma(\eta_{\ell})) \left\langle \mathbf{r}_{\ell},\mathbf{d}_{\ell}\right\rangle
\label{eq:exact_grad_gate_main}
\end{equation}
for $\ell\in B_b$.
\end{theorem}

A useful consequence of the blockwise gradient factorization is that it gives the router a direct learning signal. Corollary~\ref{cor:exact_logit_gradients_main} makes this signal explicit: gradient descent pushes an atom’s routing probability upward when its update direction aligns better with the block gradient than the currently retrieved average update, and pushes it downward when its alignment is worse. The router therefore learns through direct comparisons among atoms, while the shared memory is encouraged to store update directions that remain useful across contexts.

\begin{corollary}[Exact logit gradients on a fixed active set]
\label{cor:exact_logit_gradients_main}
Fix an active set $I$ and define $\psi_{b,m} := \left\langle \nabla_{\mathbf{S}_b(c)}\mathcal{L},\mathbf{C}_m\right\rangle_F$ and $\bar{\psi}_b := \sum_{j\in I}\alpha_{b,j}(c)\psi_{b,j}$ . Let $ z_{b,m}(c):=\zeta_{b,m}+\tau_{\mathrm{lang}}\log p_m(c), \ \forall \ m \in I$. Then,
\begin{equation}
\frac{\partial \mathcal{L}}{\partial z_{b,m}(c)} = \alpha_{b,m}(c)\bigl(\psi_{b,m}-\bar{\psi}_b\bigr)
\end{equation}
Equivalently, writing $p_m(c)=\operatorname{softmax}(\boldsymbol{\rho}(c))_m$, the gradients with respect to the state and raw language logits satisfy:
\begin{equation}
\frac{\partial \mathcal{L}}{\partial \zeta_{b,m}} = \alpha_{b,m}(c)\bigl(\psi_{b,m}-\bar{\psi}_b\bigr) \ \text{and} \
\frac{\partial \mathcal{L}}{\partial \rho_m(c)} = \tau_{\mathrm{lang}}\alpha_{b,m}(c)\bigl(\psi_{b,m}-\bar{\psi}_b\bigr).
\label{eq:exact_state_language_logit_grad_main}
\end{equation}
\end{corollary}

Taken together, these results explain the inner workings of our attention-based fine-tuning methodology. The variational characterization shows that language enters as a semantic prior over retrieval rather than as direct weight generation; the norm bound shows that the routed update remains inside a controlled convex hull of shared atoms; and the gradient identities show how those atoms receive blockwise supervision from the loss. Thus, within each fixed active set, the router learns to compare candidate update directions by their usefulness for the current computation, yielding a stable and interpretable mechanism for data-adaptive PEFT.

%% file: Chapters/6_Conclusion.tex

\section{Limitations \& Conclusion}\label{s:Limitations}

In conclusion, we introduced a queryable memory of rank-space atoms that lets low-rank adapters retrieve input- and depth-conditioned updates during the forward pass. This dynamic routing preserves the efficiency and stability of parameter-efficient fine-tuning while improving optimization and generalization, with optional language guidance providing an interpretable semantic prior.


But there are some limitations. The advantage of the queryable approach is not uniform across all benchmarks. The blockwise routing mechanism increases the forward-pass computational complexity compared to a static baseline, though empirical profiling shows it remains substantially faster than comparable dynamic methods like HyRA and RepLoRA (see Table \ref{tab:inference_efficiency}). Future work should refine atom retrieval in novel feature environments and reduce routing overhead. Finally, by lowering the computational barrier to effective fine-tuning, queryable memory architectures could also make harmful model adaptation easier. This motivates cautious release practices in high-stakes settings, where context-dependent routing may complicate exhaustive safety certification. \hfill $\blacksquare$

%% file: Chapters/7_Appendix.tex
\appendix

\input{Chapters_Appendix/7a_}

\input{Chapters_Appendix/7c_Experimental_details}
\input{Chapters_Appendix/7d_diagnostics}

\input{Chapters_Appendix/7b_Additional_Theoretical_Analysis}

\input{Chapters_Appendix/7x_other}

%% file: Chapters_Appendix/7a_.tex
\section{Additional Methodological Details}\label{s:MethodDetails}

\subsection{Attention Updates}
\label{sec:attention_updates}

Consider an attention block $\ell$ with token matrix $\boldsymbol{H}_{\ell}\in\mathbb{R}^{T\times d_{\ell}}$. Let the frozen attention projections be $\{ \boldsymbol{W}_{\ell}^{0,Q}, \boldsymbol{W}_{\ell}^{0,K},
\boldsymbol{W}_{\ell}^{0,V}, \boldsymbol{W}_{\ell}^{0,O} \}$. For each projection type $p\in\{Q,K,V,O\}$, the LoRA factors are \citep{hu2022lora}: $\boldsymbol{A}_{\ell}^{p}\in
\mathbb{R}^{r\times d_{\ell}^{\mathrm{in},p}}$and $\boldsymbol{B}_{\ell}^{p}\in
\mathbb{R}^{d_{\ell}^{\mathrm{out},p}\times r}$. For the query, key, and value projections, the projection input is $\boldsymbol{H}_{\ell}$. We define the corresponding token-level rank-state by $\boldsymbol{R}_{\ell}^{p}
=
\boldsymbol{H}_{\ell}
(\boldsymbol{A}_{\ell}^{p})^{\top}
\in \mathbb{R}^{T\times r}
$ for $p\in\{Q,K,V\}$. Its token average is
\begin{equation}
\bar{\boldsymbol{r}}_{\ell}^{p}
=
\frac{1}{T}\mathbf{1}_{T}^{\top}
\boldsymbol{R}_{\ell}^{p}
\end{equation}
For attention blocks, we use a single shared router state obtained by averaging the query, key, and value summaries:
\begin{equation}
\bar{\boldsymbol{r}}_{\ell}^{\mathrm{attn}} = \frac{1}{3} \left( \bar{\boldsymbol{r}}_{\ell}^{Q} + \bar{\boldsymbol{r}}_{\ell}^{K} + \bar{\boldsymbol{r}}_{\ell}^{V} \right)
\end{equation}
This vector is used in place of the block-entry rank-state $\boldsymbol{s}_{\ell_b}^{\mathrm{entry}}$ in Eq.~\eqref{eq:app_prequery} for attention blocks. Thus, the router receives a compact summary of the low-rank attention state, while avoiding a circular dependence on the output projection. Given the routed rank-space operator $\boldsymbol{S}_b(\boldsymbol{c})$,
the adapted update for each projection type $p\in\{Q,K,V,O\}$ is
\begin{align}
\Delta \boldsymbol{W}_{\ell}^{p}(\boldsymbol{H}_{\ell};b,\boldsymbol{c}) &= \frac{\alpha}{r} \boldsymbol{B}_{\ell}^{p} \left( \boldsymbol{I}_r + g_{\ell}^{p}\boldsymbol{S}_b(\boldsymbol{c}) \right) \boldsymbol{A}_{\ell}^{p}\\ g_{\ell}^{p} &= \sigma(\eta_{\ell}^{p})
\end{align}
Here, $g_{\ell}^{p} \in(0,1)$. The adapted attention projections are then:
\begin{align}
\boldsymbol{Q}_{\ell} &=
\boldsymbol{H}_{\ell} \left( \boldsymbol{W}_{\ell}^{0,Q} + \Delta \boldsymbol{W}_{\ell}^{Q} \right)^{\top} \\ \boldsymbol{K}_{\ell} &= \boldsymbol{H}_{\ell} \left( \boldsymbol{W}_{\ell}^{0,K} + \Delta \boldsymbol{W}_{\ell}^{K} \right)^{\top} \\
\boldsymbol{V}_{\ell} &= \boldsymbol{H}_{\ell} \left( \boldsymbol{W}_{\ell}^{0,V} + \Delta \boldsymbol{W}_{\ell}^{V} \right)^{\top} \\
\boldsymbol{O}_{\ell} &= \operatorname{Attention} (\boldsymbol{Q}_{\ell},\boldsymbol{K}_{\ell},\boldsymbol{V}_{\ell}) \left( \boldsymbol{W}_{\ell}^{0,O} + \Delta \boldsymbol{W}_{\ell}^{O} \right)^{\top}
\end{align}
The same routed operator $\boldsymbol{S}_b(\boldsymbol{c})$ is reused across $Q,K,V,O$ within the block, whereas the local LoRA factors $\boldsymbol{A}_{\ell}^{p},\boldsymbol{B}_{\ell}^{p}$ and gates $g_{\ell}^{p}$ remain projection-specific. Hence, the memory provides a global reusable rank-space transformation, while the projection-specific factors preserve the inductive bias of each attention projection.

For algorithmic details, see Algorithm  \ref{alg:instruction_queryable_compact} and \ref{alg:6}.

\begin{algorithm}[htb]
\caption{Instruction-regularized queryable low-rank adaptation; see Algorithm \ref{alg:6} for details.}
\label{alg:instruction_queryable_compact}
\begin{algorithmic}[1]
\Require Frozen weights $\{\boldsymbol W_\ell^0\}_{\ell=1}^L$; low-rank factors
$\{\boldsymbol A_\ell,\boldsymbol B_\ell,\eta_\ell\}_{\ell=1}^L$; shared atoms
$\{\boldsymbol C_m,\boldsymbol k_m\}_{m=1}^M$; routing parameters
$\{\boldsymbol w_\ell\}_{\ell=1}^L,\boldsymbol Q_{\rm cur},\boldsymbol Q_{\rm dep},
\boldsymbol Q_{\rm ctx},\boldsymbol R_{\rm ctx}$; depth-summary projections
$\boldsymbol Q_{{\rm dep},q},\boldsymbol Q_{{\rm dep},k}$; blocks
$\{\mathcal B_b\}_{b=1}^B$; instruction embedding $\boldsymbol e(\boldsymbol c)$;
input $\boldsymbol h_1=\boldsymbol x$.
\Ensure Output $\widehat{\boldsymbol y}$.
\State Initialize completed block summaries $\mathcal U\leftarrow[\,]$.
\For{$b=1,\dots,B$}
    \State Let $\ell_b$ be the first layer in $\mathcal B_b$ and compute
    $\boldsymbol s_{\ell_b}^{\rm entry}
    \leftarrow \boldsymbol A_{\ell_b}\mathcal D_{\ell_b}(\boldsymbol h_{\ell_b})$.
    \State Form the instruction-conditioned pre-query
    \[
    \boldsymbol q_b^{(0)}
    \leftarrow
    \boldsymbol w_{\ell_b}
    +\boldsymbol Q_{\rm cur}\boldsymbol s_{\ell_b}^{\rm entry}
    +\lambda_{\rm ctx}\boldsymbol Q_{\rm ctx}\boldsymbol e(\boldsymbol c).
    \]
    \State Retrieve a depth summary from previous blocks:
    \[
    \boldsymbol u_{b-1}^{\rm att}
    \leftarrow
    \begin{cases}
    \boldsymbol 0, & b=1,\\[0.5ex]
    \sum_{i<b}\beta_{i\to b}\bar{\boldsymbol s}_i, & b>1,
    \end{cases}
    \qquad
    \beta_{i\to b}
    \propto
    \exp\!\left(
    \frac{
    \langle
    \operatorname{RMSNorm}(\boldsymbol Q_{{\rm dep},q}\boldsymbol q_b^{(0)}),
    \operatorname{RMSNorm}(\boldsymbol Q_{{\rm dep},k}\bar{\boldsymbol s}_i)
    \rangle
    }{\sqrt{d_k}T_{\rm dep}}
    \right).
    \]
    \State Set the final block query
    $\boldsymbol q_b
    \leftarrow
    \boldsymbol q_b^{(0)}
    +\boldsymbol Q_{\rm dep}\boldsymbol u_{b-1}^{\rm att}$.
    \State Compute state-routing logits and language-prior logits:
    \[
    \zeta_{b,m}
    \leftarrow
    \frac{
    \langle
    \operatorname{RMSNorm}(\boldsymbol q_b),
    \operatorname{RMSNorm}(\boldsymbol k_m)
    \rangle
    }{\sqrt{d_k}T_{\rm attn}},
    \qquad
    \rho_m(\boldsymbol c)
    \leftarrow
    \frac{
    \langle
    \operatorname{RMSNorm}(\boldsymbol R_{\rm ctx}\boldsymbol e(\boldsymbol c)),
    \operatorname{RMSNorm}(\boldsymbol k_m)
    \rangle
    }{\sqrt{d_k}T_{\rm lang}} .
    \]
    \State Convert $\rho(\boldsymbol c)$ to a prior
    $p(\boldsymbol c)=\operatorname{softmax}(\rho(\boldsymbol c))$ and form
    \[
    \widetilde\zeta_{b,m}(\boldsymbol c)
    \leftarrow
    \zeta_{b,m}
    +\tau_{\rm lang}\log p_m(\boldsymbol c).
    \]
    \State Route over the shared atoms:
    \[
    \alpha_b(\boldsymbol c)
    \leftarrow
    \operatorname{TopKSoftmax}(\widetilde\zeta_b(\boldsymbol c)),
    \qquad
    \boldsymbol S_b(\boldsymbol c)
    \leftarrow
    \sum_{m=1}^M \alpha_{b,m}(\boldsymbol c)\boldsymbol C_m .
    \]
    \For{$\ell\in\mathcal B_b$}
        \State Compute
        $\boldsymbol s_\ell\leftarrow
        \boldsymbol A_\ell\mathcal D_\ell(\boldsymbol h_\ell)$,
        reusing $\boldsymbol s_{\ell_b}^{\rm entry}$ when $\ell=\ell_b$.
        \State Apply the routed rank-space adapter
        \[
        \boldsymbol t_\ell
        \leftarrow
        \bigl(\boldsymbol I_r+\sigma(\eta_\ell)\boldsymbol S_b(\boldsymbol c)\bigr)
        \boldsymbol s_\ell .
        \]
        \If{$\ell$ is a standard hidden layer}
            \State Update
            \[
            \boldsymbol h_{\ell+1}
            \leftarrow
            \phi\!\left(
            \boldsymbol W_\ell^0\boldsymbol h_\ell
            +\frac{\alpha}{r}\boldsymbol B_\ell\boldsymbol t_\ell
            \right).
            \]
        \Else
            \State For each attention projection
            $p\in\{Q,K,V,O\}$, use
            \[
            \Delta\boldsymbol W_\ell^p(\boldsymbol c)
            \leftarrow
            \frac{\alpha}{r}
            \boldsymbol B_\ell^p
            \bigl(\boldsymbol I_r+g_\ell^p\boldsymbol S_b(\boldsymbol c)\bigr)
            \boldsymbol A_\ell^p,
            \]
            and continue the Transformer forward pass with the adapted projections.
        \EndIf
    \EndFor
    \State Store the block summary
    \[
    \bar{\boldsymbol s}_b
    \leftarrow
    \frac{1}{|\mathcal B_b|}
    \sum_{\ell\in\mathcal B_b}\boldsymbol s_\ell,
    \qquad
    \mathcal U\leftarrow\mathcal U\cup\{\bar{\boldsymbol s}_b\}.
    \]
\EndFor
\State \Return $\widehat{\boldsymbol y}\leftarrow \operatorname{Head}(\boldsymbol h_{L+1})$.
\end{algorithmic}
\end{algorithm}

\begin{algorithm}[htb]
\caption{Instruction-regularized global queryable update}
\label{alg:6}
\begin{algorithmic}[1]
\Require Frozen hidden-layer weights $\{\boldsymbol{W}_\ell^0\}_{\ell=1}^{L}$; trainable factors $\{\boldsymbol{A}_\ell,\boldsymbol{B}_\ell,\eta_\ell\}_{\ell=1}^{L}$; shared atoms $\{\boldsymbol{C}_m,\boldsymbol{k}_m\}_{m=1}^{M}$; routing parameters $\{\boldsymbol{w}_\ell\}_{\ell=1}^{L},\boldsymbol{Q}_{\mathrm{cur}},\boldsymbol{Q}_{\mathrm{dep}}$; block partition $\{\boldsymbol{B}_b\}_{b=1}^{B}$; input representation $\boldsymbol{h}_1 = \boldsymbol{x}$, external instruction $\boldsymbol{c}$ with embedding $\boldsymbol{e}(\boldsymbol{c})$, $Q_{\mathrm{cur}}$, $Q_{\mathrm{dep}}$, $Q_{\mathrm{ctx}}$, $R_{\mathrm{ctx}}$; temperatures $T_{\mathrm{attn}},T_{\mathrm{lang}},T_{\mathrm{dep}}$; language-prior strength $\tau_{\mathrm{lang}}$; language-query weight $\lambda_{\mathrm{ctx}}$.
\Ensure Model output $\widehat{\boldsymbol{y}}$
\State Initialize the list of completed block summaries $\mathcal U \leftarrow [\,]$.
\State $\boldsymbol{u}_0 \gets \mathrm{None}$
\For{$b=1,2,\dots,B$}
    \State $\ell_b \gets$ first index in block $B_b$
    \State $\boldsymbol{s}_{\ell_b}^{\mathrm{entry}} \gets \boldsymbol{A}_{\ell_b} \mathcal{D}_{\ell_b}(\boldsymbol{h}_{\ell_b})$
    \State Form the query $\boldsymbol{q}_b^{(0)} \leftarrow \boldsymbol{w}_{\ell_b} + \boldsymbol{Q}_{\mathrm{cur}} \boldsymbol{s}^{\mathrm{entry}}_{\ell_b}+\lambda_{\mathrm{ctx}} \boldsymbol{Q}_{\mathrm{ctx}} \boldsymbol{e}(\boldsymbol{c})$
    \If{$b = 1$}
        \State $\boldsymbol{u}_{b-1}^{\mathrm{att}} \leftarrow \boldsymbol{0}$
    \Else
        \For{each completed block summary $\bar{\boldsymbol{s}}_i \in \mathcal U$, $i=1,\dots,b-1$}
            \State Compute the depth-summary key $\widehat{\boldsymbol{\kappa}}_i \leftarrow \operatorname{RMSNorm}(\boldsymbol{Q}_{\mathrm{dep},k} \ \bar{\boldsymbol{s}}_i)$
            \State Compute the depth-summary logit $ \xi_{i \to b} \leftarrow \frac{\langle \operatorname{RMSNorm}(\boldsymbol{Q}_{\mathrm{dep},q} \ \boldsymbol{q}_b^{(0)}), \ \widehat{\boldsymbol{\kappa}}_i\rangle}{\sqrt{d_k} T_{\mathrm{dep}}}$
        \EndFor
        \State Convert $\{\xi_{i\to b}\}_{i=1}^{b-1}$ to softmax weights $\{\beta_{i\to b}\}_{i=1}^{b-1}$.
        \State Complete the block summary $\boldsymbol{u}_{b-1}^{\mathrm{att}} \leftarrow \sum_{i=1}^{b-1} \beta_{i\to b} \bar{\boldsymbol{s}}_i$
    \EndIf
    \State Compute the final block summary $\boldsymbol{q}_b \leftarrow \boldsymbol{q}_b^{(0)} + \boldsymbol{Q}_{\mathrm{dep}} \boldsymbol{u}_{b-1}^{\mathrm{att}}$
    \For{$m=1,\dots,M$}
        \State Compute the routing logit  $\zeta_{b,m} \leftarrow \frac{\langle \operatorname{RMSNorm}(\boldsymbol{q}_b),\ \operatorname{RMSNorm}(\boldsymbol{k}_m)\rangle}{\sqrt{d_k} T_{\mathrm{attn}}}$
        \State Compute the language prior logit $ \rho_m(c) \leftarrow \frac{\langle \operatorname{RMSNorm}(\boldsymbol{R}_{\mathrm{ctx}} \boldsymbol{e}(\boldsymbol{c})),\ \operatorname{RMSNorm}(\boldsymbol{k}_m)\rangle}{\sqrt{d_k} T_{\mathrm{lang}}}$
    \EndFor
    \State Convert $\{\rho_m(c)\}_{m=1}^M$ to the language prior $p(\boldsymbol{c})$
    \State Form full joint routing logits $\widetilde \zeta_{b,m}(\boldsymbol{c}) \leftarrow \zeta_{b,m} + \tau_{\mathrm{lang}}\log p_m(\boldsymbol{c})$ $ \forall \ m = 1,2, 3, ..., M$
    \State Apply top-$k$ softmax to obtain routing weights $\alpha_b(\boldsymbol{c})=\{\alpha_{b,m}(\boldsymbol{c})\}_{m=1}^M$
    \State Construct the routing operator $ \boldsymbol{S}_b(\boldsymbol{c}) \leftarrow \sum_{m=1}^M \alpha_{b,m}(\boldsymbol{c}) \boldsymbol{C}_m$
    \For{each $\ell\in \boldsymbol{B}_b$ in order}
        \If{$\ell=\ell_b$}
            \State $\boldsymbol{s}_\ell \gets \boldsymbol{s}_{\ell_b}^{\mathrm{entry}}$
        \Else
            \State $\boldsymbol{s}_\ell \gets \boldsymbol{A}_\ell \mathcal{D}_\ell (\boldsymbol{h}_\ell)$
        \EndIf
        \State $\boldsymbol{d}_\ell \gets \boldsymbol{S}_b \boldsymbol{s}_\ell$
        \State $g_\ell \gets \text{sigmoid}(\eta_\ell)$
        \State $\boldsymbol{t}_\ell \gets \boldsymbol{s}_\ell + g_\ell \boldsymbol{d}_\ell$
        \If{layer $\ell$ is a standard hidden layer}
            \State $\delta \boldsymbol{h}_\ell \gets \frac{\alpha}{r} \boldsymbol{B}_\ell \boldsymbol{t}_\ell$
            \State $\boldsymbol{y}_\ell \gets \boldsymbol{W}_\ell^0 \boldsymbol{h}_\ell + \delta \boldsymbol{h}_\ell$
            \State $\boldsymbol{h}_{\ell+1} \gets \phi(\boldsymbol{y}_\ell)$
        \Else
            \State $\Delta \boldsymbol{W}_{\ell}^{p} (\boldsymbol{H}_{\ell};b,\boldsymbol{c}) \leftarrow \frac{\alpha}{r} \boldsymbol{B}_{\ell}^{p}\bigl(\boldsymbol{I}_r + g_{\ell}^{p} \boldsymbol{S}_b(\boldsymbol{c})\bigr) \boldsymbol{A}_{\ell}^{p}$ , $\forall \ p\in\{\boldsymbol{Q},\boldsymbol{K},\boldsymbol{V},\boldsymbol{O}\} $.
            \State Form the adapted attention projections and continue the Transformer forward pass.
        \EndIf
        \State $\boldsymbol{z}_b \gets \boldsymbol{z}_b + \boldsymbol{s}_\ell$ and $n_b \gets n_b + 1$
    \EndFor
    \State $\bar{\boldsymbol{s}}_b \gets \frac{\boldsymbol{z}_b}{n_b}$
    \State Append $\bar{\boldsymbol{s}}_b$ to $\mathcal U$.
\EndFor
\State \Return $\widehat{\boldsymbol{y}} \gets \mathrm{Head}(\boldsymbol{h}_{L+1})$
\end{algorithmic}
\end{algorithm}

%% file: Chapters_Appendix/7c_Experimental_details.tex
\section{Additional Experimental Results}\label{s:AdditionalEmpiricalResults}

Figure \ref{fig:grad-concentration} summarizes Figure \ref{fig:grad-norm-layer} results through the gradient concentration index, defined as the ratio between the maximum and mean per-layer adapter gradient norm. Lower values indicate that optimization is less dominated by a small number of layers. Across epochs, the instruction-queryable model has the lowest or near-lowest concentration index, while LoRA remains more concentrated. The language-regularized router distributes learning more evenly across the adapted network; hence, the shared atom memory should be trained with many blockwise queries.

\begin{figure}[htb]
    \centering
    \includegraphics[width=0.52\linewidth]{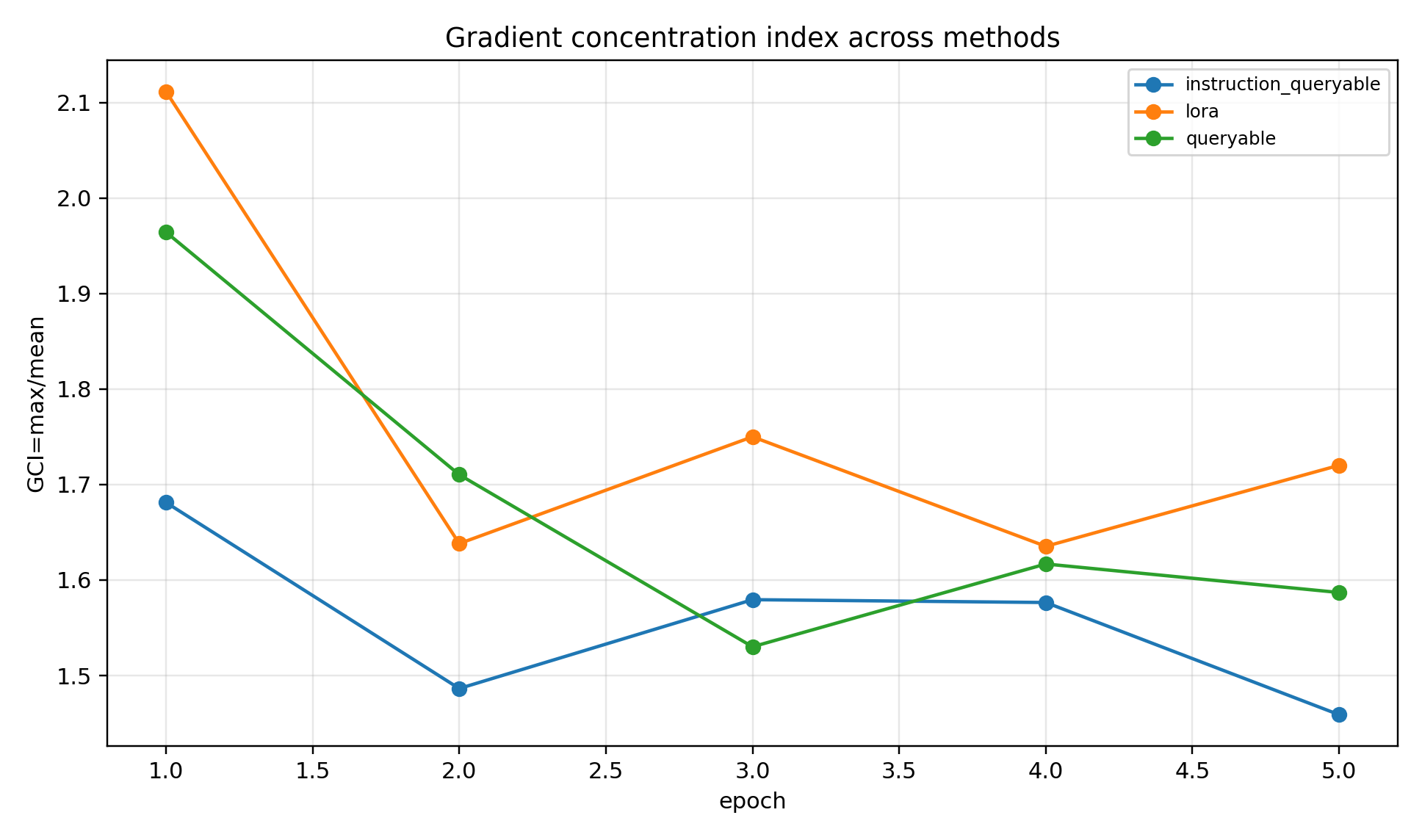}
    \caption{
    Gradient concentration index across epochs. The instruction-queryable method exhibits lower gradient concentration than LoRA, suggesting less layerwise gradient collapse and more distributed adapter learning.
    }
    \label{fig:grad-concentration}
\end{figure}

\subsection{Adapter Performance Comparisons}

Tables~\ref{tab:adapter_design_comparison} and~\ref{tab:dynamic_adapter_accuracy} compare adapter design choices with downstream generalization on AllenAI ARC using Qwen2.5-0.5B-Instruct as a frozen backbone. All methods post-train only adapter or update-memory parameters under the same broad reasoning instruction and are averaged over three independent seeds.  The run adapts the attention and MLP projection modules with rank-$8$ updates for ten epochs. For the queryable variants, the router selects a sparse top-$4$ subset from $16$ shared update atoms; the instruction-conditioned variant additionally encodes the task instruction as an external routing prior.

Table~\ref{tab:adapter_design_comparison} highlights that the main difference is not raw parameter count.  LoRA and DoRA use static layer-local updates, while Text-to-LoRA and Doc-to-LoRA use generated adapters with roughly $145$-$148$M trainable parameters.  MoE-LoRA introduces routing but increases the trainable budget to $18.58$M.  In contrast, \textsc{Queryable LoRA} and \textsc{Instruction-Queryable LoRA} use only $4.456$M trainable parameters, adding about $1.3\%$ overhead relative to LoRA while introducing input-dependent update selection, and in the instruction-conditioned case, an explicit language prior.  Thus, the queryable methods test whether structured routing over a compact shared update memory can improve generalization without relying on parameter-heavy adapter generation. Table~\ref{tab:dynamic_adapter_accuracy} shows that all evaluated methods reach $1.0000$ train accuracy, so the meaningful distinction is evaluation accuracy rather than train-set fitting.  The generated Text-to-LoRA and Doc-to-LoRA baselines obtain $0.5599 \pm 0.0495$ evaluation accuracy despite using the largest parameter budgets. The routed methods perform substantially better: \textsc{Queryable LoRA} achieves the highest mean accuracy, $0.6576 \pm 0.0176$, while \textsc{Instruction-Queryable LoRA} remains competitive at $0.6510 \pm 0.0113$.  MoE-LoRA and Lily also benefit from routing, reaching $0.6458 \pm 0.0254$ and $0.6406 \pm 0.0103$, respectively, but neither exceeds the compact queryable update bank in mean performance.

Overall, in this controlled Qwen-0.5B ARC study, generalization improves more from structured, input-dependent retrieval over reusable update atoms than from simply increasing adapter-generator capacity.  The evidence is strongest for the separation between routed/queryable methods and the high-parameter generated baselines.

\begin{table}[t]
\centering
\tiny
\setlength{\tabcolsep}{7pt}
\renewcommand{\arraystretch}{1.15}
\caption{Qualitative comparison of static, routed, and instruction-conditioned LoRA-style adapters.}
\label{tab:adapter_design_comparison}
\begin{tabular}{p{0.38\linewidth}ccc}
\toprule
\textbf{Method} & \textbf{Trainable params} & \textbf{Routing?} & \textbf{Instruction?} \\
\midrule
LoRA \citep{hu2022lora} & 4,399,104& No& No\\
DoRA \citep{liu2024dora} & 4,644,864& No& No\\
RepLoRA \citep{truong2025replora} & 145,872,000& Shared static& No\\
HyRA \citep{HyRA2025} & 147,922,752& Generated&No\\
DoRAN \citep{diep2025doran} & 145,270,056& Generated&No\\
MoE-LoRA \citep{luo2024moelora} & 18,579,456& Yes& No\\
Text-to-LoRA \citep{TextToLoRA2025} & 148,127,232& Generated& Yes\\
Lily \citep{Lily2024} & 2,159,872& Yes&No\\
Doc-to-LoRA \citep{DocToLoRA2026} & 148,127,232& Generated&Yes\\
Queryable LoRA & 4,456,936& Yes& No\\
Instruction-Queryable LoRA& 4,456,936& Yes& Yes\\
\bottomrule
\end{tabular}
\end{table}
\begin{table}[t]
\centering
\tiny
\setlength{\tabcolsep}{7pt}
\renewcommand{\arraystretch}{1.15}
\caption{Best train and evaluation accuracy across three independent runs for dynamic adapter-based approaches. We report mean $\pm$ standard deviation.}
\label{tab:dynamic_adapter_accuracy}
\begin{tabular}{lcc}
\toprule
\textbf{Method} 
& \textbf{Train Accuracy}& \textbf{Eval Accuracy}\\
\midrule
Text-to-LoRA \citep{TextToLoRA2025}
& $1.0000 \pm 0.0000$ 
& $0.5599 \pm 0.0495$ \\

Instruction-Queryable LoRA& $1.0000 \pm 0.0000$ 
& $\boldsymbol{0.6510} \pm \boldsymbol{0.0113}$ \\

MoE-LoRA \citep{luo2024moelora}
& $1.0000 \pm 0.0000$ 
& $0.6458 \pm 0.0254$ \\

Lily \citep{Lily2024}
& $1.0000 \pm 0.0000$ 
& $0.6406 \pm 0.0103$ \\

Doc-to-LoRA \citep{DocToLoRA2026}
& $1.0000 \pm 0.0000$ 
& $0.5599 \pm 0.0495$ \\
\bottomrule
\end{tabular}
\end{table}

\subsection{Ablation Studies} \label{s:Ablation_studies}

The ablation results in Table~\ref{tab:instruction_router_ablation} were obtained by fine-tuning
Qwen2.5-0.5B-Instruct on the GPQA-Diamond task for $10$ epochs, using the same instruction-queryable adapter configuration across all variants. Each row keeps the trainable parameter count fixed at $4.46$M and changes only the source or structure of the routing signal: no instruction, a generic instruction, the correct task instruction, shuffled or adversarial instructions, a random embedding with matched norm, an instruction-only router without the state query, and a state-query-only router without the instruction prior. We selected the optimal checkpoint for each setting based on early stopping on the highest training accuracy;  in the event of $1.0$ training accuracy across multiple epochs, we selected the epoch with the lowest training loss. The reported train and evaluation accuracies correspond to this specific checkpoint, while the timing columns report the mean evaluation time and the corresponding per-example evaluation latency. Overall, the table isolates whether performance gains come from meaningful semantic conditioning, from the model's internal state query, or merely from added routing capacity.

\begin{table}[htb]
\centering
\small
\setlength{\tabcolsep}{4pt}
\renewcommand{\arraystretch}{1.15}
\caption{Instruction-router ablation results. Each ablation uses the same number of trainable parameters; the comparison isolates the effect of the instruction prior, state query, and semantic alignment.}
\label{tab:instruction_router_ablation}
\resizebox{\textwidth}{!}{
\begin{tabular}{llrrr}
\toprule
\textbf{Ablation} & \textbf{Our interpretation} 
& \textbf{Best eval acc.} 
& \textbf{Mean eval sec.} 
& \textbf{Latency / sample} \\
\midrule
 No instruction& base dynamic routing& 0.2420& 1.6362&0.0739\\
Generic instruction 
& Generic non-task-specific language control
& 0.2500 & 1.5253 & 0.0953 \\

Correct instruction 
& Semantic prior benefit
& 0.3125 & 1.5571 & 0.0973 \\

Shuffled instruction 
& Tests semantic mismatch
& 0.1875 & 1.4997 & 0.0937 \\

Adversarial wrong instruction 
& Tests robustness
& 0.2500 & 1.5432 & 0.0964 \\

Random embedding, same norm 
& Tests capacity versus semantics
& 0.1875 & 1.5392 & 0.0962 \\

State query only, no prior 
& Tests whether instruction helps
& 0.2500 & 1.5169 & 0.0948 \\
\bottomrule
\end{tabular}
}
\end{table}

\begin{figure}[htb]
    \centering
    \begin{subfigure}[t]{0.49\linewidth}
        \centering
        \includegraphics[width=\linewidth]{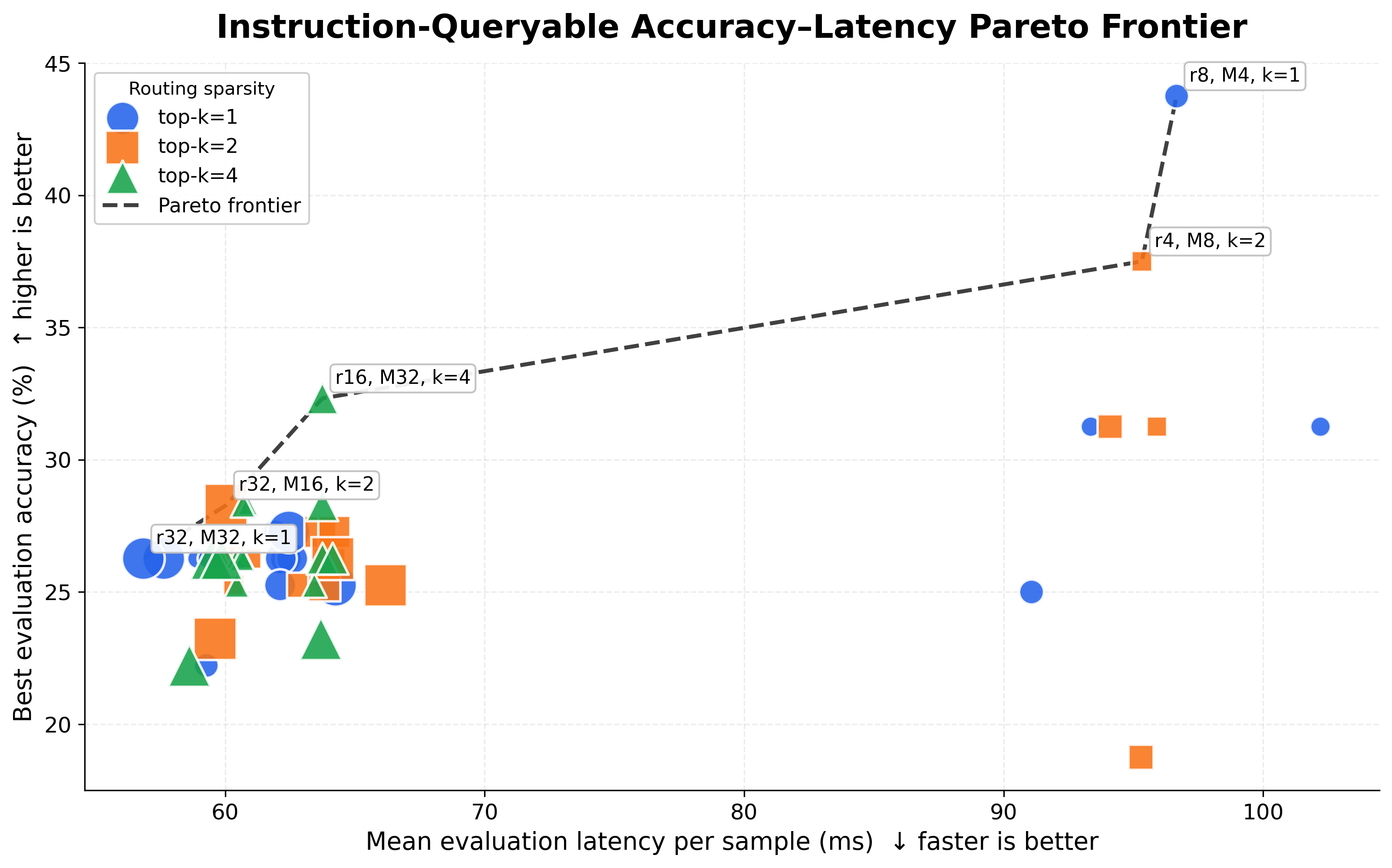}
        \caption{Accuracy--latency frontier.}
        \label{fig:iq_gpqa_pareto_latency}
    \end{subfigure}
    \hfill
    \begin{subfigure}[t]{0.49\linewidth}
        \centering
        \includegraphics[width=\linewidth]{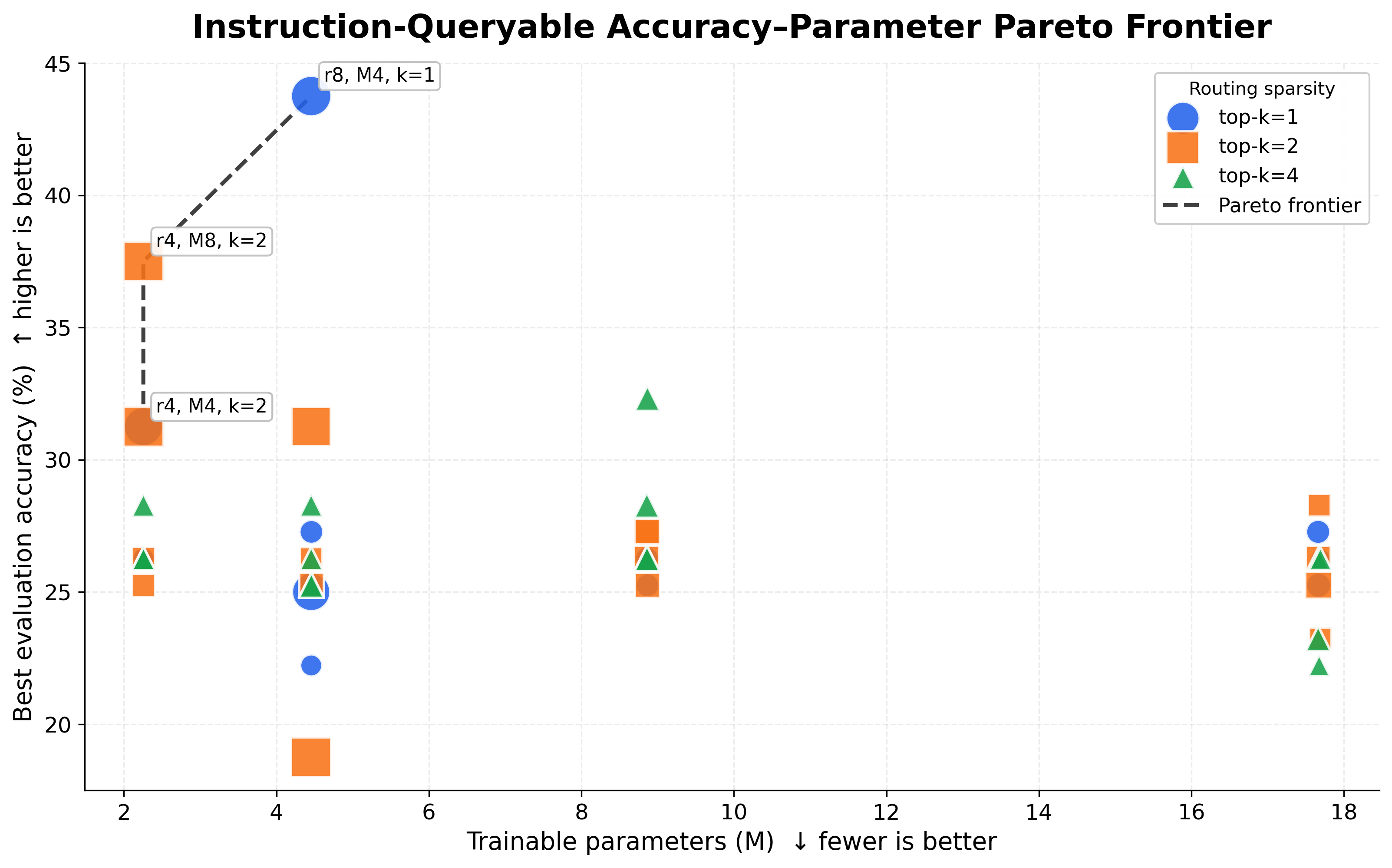}
        \caption{Accuracy--parameter frontier.}
        \label{fig:iq_gpqa_pareto_params}
    \end{subfigure}
    \caption{\textbf{Instruction-queryable Pareto tradeoffs on GPQA-Diamond.}
    We sweep the instruction-queryable adapter design over rank $r$, number of global update atoms $M$, and routing sparsity $k$, and report the best evaluation accuracy against two efficiency costs: mean evaluation latency per sample and number of trainable parameters. Higher accuracy is better, while lower latency and fewer trainable parameters are better. Markers indicate the routing sparsity $k$, marker size reflects trainable parameter count, and the dashed curve shows the non-dominated Pareto frontier.}
    \label{fig:iq_gpqa_pareto}
\end{figure}

\paragraph{Accuracy-efficiency tradeoffs on GPQA-Diamond.}
We evaluate the instruction-queryable adapter on GPQA-Diamond by sweeping the adapter rank $r \in \{4,8,16,32\}$, the number of shared update atoms $M \in \{4,8,16,32\}$, and the routing sparsity $k \in \{1,2,4,8\}$, excluding invalid settings with $k>M$, and with the total number of epochs not exceeding $10$. All configurations use the same frozen Qwen2.5-0.5B-Instruct backbone, task instruction, training procedure, and evaluation protocol; only the instruction-queryable hyperparameters are varied. For each setting, we select the final checkpoint based on the highest observed training accuracy and report its corresponding held-out evaluation accuracy. The dashed curves in Figure~\ref{fig:iq_gpqa_pareto} denote the non-dominated Pareto frontier, where a configuration is retained only if no other setting achieves both higher accuracy and lower cost. The results show that most configurations cluster in a low-latency regime with modest accuracy, while a few sparse routed configurations achieve substantially better accuracy without requiring the largest parameter budgets. In particular, the best Pareto points arise from compact or moderately sized memories with sparse top-$k$ routing, rather than from uniformly increasing rank, atom count, or active atoms per block. The dominated high-parameter and high-latency settings indicate that additional adapter capacity alone is insufficient on GPQA-Diamond; effective atom selection and sparse reuse of the global update memory are critical. Overall, the sweep supports the main design principle of instruction-queryable adaptation: a small shared memory of reusable rank-space update atoms, biased by the task instruction and selected by the router, can yield a stronger accuracy--cost tradeoff than simply enlarging the adapter.

\subsection{Optimization Stress Test}

Tables~\ref{table:2} and~\ref{table:3} report an additional stress test designed to isolate optimization behavior in a deeper and narrower regression model.  This experiment uses the same source-to-target non-convex regression protocol as the main synthetic study, but changes the regressor to a GELU multilayer perceptron with 32 hidden layers of width 32 and a scalar output head.  The point of this configuration is to create a difficult adaptation regime in which useful update directions must propagate through many narrow layers.  This is typically the setting where static, layer-local low-rank updates may be brittle, since each adapted layer receives only its own local correction, whereas the queryable adapter can reuse a shared memory of rank-space atoms across blocks.

For each function, source and target examples are generated independently from noisy two-dimensional regression tasks. The source stage uses $3,000$ samples, split into $80\%$ source-train and $20\%$ source-test, to pretrain the base regressor for 300 epochs.  The target stage then samples $1,200$ fresh examples from a shifted and perturbed target version of the function; $400$ examples are used for adaptation and the remaining $800$ examples are held out for testing.  Inputs are sampled uniformly over the same box used by the experiment driver, and labels include additive Gaussian noise with standard deviation $0.05$.  Adapter methods are trained on the target split for $5,000$ epochs with batch size $64$, adapter learning rate $5\times 10^{-4}$, rank $8$, scaling $16$, and no adapter dropout.  The queryable method uses $8$ shared rank-space atoms, top-$2$ routing, and $4$ blocks.  The PEFT baselines update only their adapter parameters under the default adapter-only setting; the column marked SOAP$^\ast$ is the full-fine-tuning baseline and updates the whole pretrained regressor with the SOAP optimizer \cite{vyas2024soap}.

The tables report post-training mean-squared error.  Table~\ref{table:2} records the best target-training MSE attained during adaptation, while Table~\ref{table:3} records the best held-out target-test MSE checkpoint.  The main pattern is that the queryable adapter achieves low-loss target solutions across nearly all functions, despite its deep, narrow architecture.  In contrast, several static or hypernetwork-style baselines remain close to high-loss plateaus or become numerically unstable.  This is particularly visible on Ackley, Langermann, Levy, Sin-Cos, and Styblinski-Tang, where SOAP$^\ast$ and most PEFT baselines remain far from the best attained MSE while the queryable model reaches a substantially lower value.  Dropwave is also informative, since, even though full fine-tuning fits the target training split well, the queryable adapter yields the strongest held-out test result, suggesting that the shared routed atom memory improves the adapted solution in this hard local-heterogeneity regime. Across the table, the queryable method is the only adapter that consistently avoids both plateau behavior and catastrophic instability, and it obtains the best PEFT result on almost every reported row.  We observe that depth makes isolated layer-wise adaptation difficult. Sharing routed update atoms across blocks can preserve useful descent directions and make the adaptation problem easier to optimize.

\begin{table}[htbp]
    \centering
    \caption{Summary -- \textbf{MSE Loss}. Comparative post-training performance for modeling stochastic non-convex functions (\cite{simulationlib}). 
    \textbf{Bold} = best overall result; \textit{Italic} = best PEFT result (excluding full fine-tuning). 
    $^*$SOAP column shows full fine-tuning results.
These results use a narrow but deep architecture with 32 depth and 32 width (unlike the main text, which uses a wide but shallower design). 
    }
    \label{table:2}
    
    \begin{tabular}{lccccccc}
        \toprule
        \textbf{Function} & \textbf{SOAP*} & \textbf{LoRA} & \textbf{DoRA} & \textbf{HyRA} & \textbf{RepLoRA} & \textbf{DoRAN} & \textbf{Ours} \\
        \midrule
        Ackley& 13.124  & 13.124  & 13.123  & 13.108  & 13.123  & 13.119  & \textit{\textbf{0.113}} \\
        Dropwave          & \textbf{0.582} & 223.643 & 234.136 & 1829.286 & 1828.609 & 1805.232 & \textit{1.291} \\
        Langermann        & 5.502   & 5.502   & 5.502   & 5.499   & 5.502   & 5.502   & \textit{\textbf{0.021}} \\
        Levy              & 324.969 & 324.956 & 324.971 & 324.447 & 324.966 & 324.875 & \textit{\textbf{7.173}} \\
        Matyas            & \textbf{0.017} & 0.135 & 753.808 & 753.813 & 753.459 & 753.748 & \textit{0.087} \\
        Michalewicz       & 0.125   & 0.125   & 0.125   & 0.125   & 0.125   & 0.125   & \textit{\textbf{0.054}} \\
        Rastrigin         & \textbf{1.270} & 49.880 & 77.443 & 4482.271 & 79.764 & 4197.973 & \textit{22.812} \\
        Sin-Cos           & 1.011   & 1.011   & 1.011   & 1.011   & 1.011   & 1.011   & \textbf{\textit{0.161}} \\
        Styblinski-Tang   & 21.271  & 11.843  & 24.708  & inf     & 20.744  & 385.618 & \textit{\textbf{10.898}} \\
        \bottomrule
    \end{tabular}
\end{table}

\begin{table}[htbp]
    \centering
    \caption{Summary -- \textbf{MSE Loss}. Comparative post-training \textbf{test} time performance for modeling stochastic non-convex functions (\cite{simulationlib}). 
    \textbf{Bold} = best overall result; \textit{Italic} = best PEFT result (excluding full fine-tuning). 
    $^*$SOAP column shows full fine-tuning results.
    These results use a narrow but deep architecture of 32 depth, 32 width (unlike in the main text, which uses a wide but shallower design). 
    }
    \label{table:3}
    
    \begin{tabular}{lccccccc}
        \toprule
        \textbf{Function} & \textbf{SOAP*} & \textbf{LoRA} & \textbf{DoRA} & \textbf{HyRA} & \textbf{RepLoRA} & \textbf{DoRAN} & \textbf{Ours} \\
        \midrule
        Ackley& 13.230  & 13.227  & 13.220  & 13.215  & 13.222  & 13.215  & \textit{\textbf{0.673}} \\
        Dropwave          & 25.096  & 193.122 & 246.421 & 1684.105 & 1683.752 & 1608.287 & \textit{\textbf{9.828}} \\
        Langermann        & 5.815   & 5.815   & 5.815   & 5.811   & 5.814   & 5.815   & \textit{\textbf{0.195}} \\
        Levy              & 337.855 & 337.855 & 337.855 & 337.855 & 337.855 & 337.855 & \textit{\textbf{136.888}} \\
        Matyas            & \textbf{0.068} & 0.870 & 749.881 & 749.881 & 749.881 & 749.881 & \textit{0.345} \\
        Michalewicz       & 0.136   & 0.136   & 0.136   & 0.136   & 0.136   & 0.136   & \textit{\textbf{0.135}} \\
        Rastrigin         & 156.795 & 113.054 & 1172.995 & 4772.434 & \textit{\textbf{108.480}} & 4127.224 & 127.405 \\
        Sin-Cos           & 0.924   & 0.924   & 0.924   & 0.923   & 0.924   & 0.923   & \textbf{\textit{0.920}} \\
        Styblinski-Tang   & 14.063  & 11.736  & 24.460  & inf     & 21.886  & 384.235 & \textit{\textbf{8.953}} \\
        \bottomrule
    \end{tabular}
\end{table}

\subsection{LLM Batch Experimentation}

\begin{table*}[htbp]
\centering
\caption{
Test accuracy for a large-scale battery of $\leq$ 1B models. Higher is better; bold indicates strategies with or tied for the best available accuracy for that task/model.
}
\label{tab:model-battery-accuracy-full-matrix}
\scriptsize
\setlength{\tabcolsep}{3pt}
\begin{tabular}{llrrr}
\toprule
\textbf{Task} & \textbf{Model} & \textbf{LoRA} & \textbf{Queryable} & \textbf{Instruction Queryable} \\
\midrule
 Orca-Math & LiquidAI/LFM2-700M & 25.0\% & 21.9\% & \textbf{26.6\%} \\
 & LiquidAI/LFM2.5-350M & 6.2\% & \textbf{9.4\%} & 6.2\% \\
 & Qwen/Qwen2.5-Coder-0.5B-Instruct & 9.4\% & \textbf{14.1\%} & 9.4\% \\
 & Qwen/Qwen3-0.6B & 25.0\% & 20.3\% & \textbf{34.4\%} \\
 & amd/ReasonLite-0.6B & 20.3\% & 20.3\% & \textbf{21.9\%} \\
 & amd/ReasonLite-0.6B-Turbo & 23.4\% & \textbf{28.1\%} & \textbf{28.1\%} \\
 & ibm-granite/granite-4.0-350m & 7.8\% & \textbf{17.2\%} & 10.9\% \\
\midrule
GSM8K & HuggingFaceTB/SmolLM2-360M-Instruct & 10.9\% & 7.8\% & \textbf{17.2\%} \\
 & LiquidAI/LFM2-700M & 42.2\% & \textbf{45.3\%} & 37.5\% \\
 & LiquidAI/LFM2.5-350M & \textbf{23.4\%} & 6.2\% & 15.6\% \\
 & Qwen/Qwen2.5-0.5B-Instruct & 21.9\% & 21.9\% & \textbf{23.4\%} \\
 & Qwen/Qwen2.5-Coder-0.5B-Instruct & 23.4\% & 25.0\% & \textbf{28.1\%} \\
 & Qwen/Qwen3-0.6B & 34.4\% & 39.1\% & \textbf{42.2\%} \\
 & amd/ReasonLite-0.6B & 29.7\% & 29.7\% & \textbf{34.4\%} \\
 & ibm-granite/granite-4.0-350m & 25.0\% & 17.2\% & \textbf{26.6\%} \\
\midrule
MATH & HuggingFaceTB/SmolLM2-360M-Instruct & 4.7\% & \textbf{9.4\%} & 3.1\% \\
 & LiquidAI/LFM2-700M & 23.4\% & \textbf{25.0\%} & \textbf{25.0\%} \\
 & LiquidAI/LFM2.5-350M & \textbf{7.8\%} & 6.2\% & \textbf{7.8\%} \\
 & Qwen/Qwen2.5-Coder-0.5B-Instruct & \textbf{10.9\%} & 9.4\% & 4.7\% \\
 & Qwen/Qwen3-0.6B & 10.9\% & 17.2\% & \textbf{20.3\%} \\
 & amd/ReasonLite-0.6B & 26.6\% & 17.2\% & \textbf{28.1\%} \\
 & ibm-granite/granite-4.0-350m & 7.8\% & \textbf{14.1\%} & 9.4\% \\
\midrule
BoolQ & HuggingFaceTB/SmolLM2-360M-Instruct & 62.5\% & \textbf{78.1\%} & \textbf{78.1\%} \\
 & LiquidAI/LFM2-700M & 81.2\% & \textbf{84.4\%} & 76.6\% \\
 & LiquidAI/LFM2.5-350M & 64.1\% & 59.4\% & \textbf{68.8\%} \\
 & Qwen/Qwen2.5-0.5B-Instruct & \textbf{70.3\%} & \textbf{70.3\%} & \textbf{70.3\%} \\
 & Qwen/Qwen2.5-Coder-0.5B-Instruct & \textbf{65.6\%} & 59.4\% & \textbf{65.6\%} \\
 & Qwen/Qwen3-0.6B & 78.1\% & \textbf{84.4\%} & 78.1\% \\
 & amd/ReasonLite-0.6B & \textbf{59.4\%} & \textbf{59.4\%} & 56.2\% \\
 & amd/ReasonLite-0.6B-Turbo & 57.8\% & 62.5\% & \textbf{70.3\%} \\
 & ibm-granite/granite-4.0-350m & 71.9\% & 65.6\% & \textbf{76.6\%} \\
\midrule
ARC-Challenge & LiquidAI/LFM2-700M & \textbf{68.8\%} & 67.2\% & 62.5\% \\
 & LiquidAI/LFM2.5-350M & 48.4\% & \textbf{62.5\%} & 57.8\% \\
 & Qwen/Qwen2.5-0.5B-Instruct & 50.0\% & 51.6\% & \textbf{53.1\%} \\
 & Qwen/Qwen2.5-Coder-0.5B-Instruct & \textbf{42.2\%} & 32.8\% & 39.1\% \\
 & Qwen/Qwen3-0.6B & 60.9\% & 57.8\% & \textbf{65.6\%} \\
 & amd/ReasonLite-0.6B & 23.4\% & 29.7\% & \textbf{32.8\%} \\
 & amd/ReasonLite-0.6B-Turbo & 21.9\% & \textbf{25.0\%} & 18.8\% \\
 & ibm-granite/granite-4.0-350m & \textbf{48.4\%} & 37.5\% & 43.8\% \\
\bottomrule
\end{tabular}
\end{table*}

\subsection{Datasets}

We use the following datasets for post-training and evaluation: GPQA-Diamond \citep{rein2023gpqa}, MBPP \citep{austin2021program}, AI2 ARC \citep{allenai:arc}, SuperGLUE \citep{wang2019superglue, clark2019boolq}, OpenBookQA \citep{OpenBookQA2018}, RACE \citep{lai-etal-2017-race}, HellaSwag \citep{zellers2019hellaswag}, GSM8K \citep{cobbe2021gsm8k}, MATH \citep{hendrycksmath2021}, Orca-Math \citep{mitra2024orcamath}, and Numina-Math \citep{numina_math_datasets}. These correspond to the general-task benchmarks in Table \ref{table:3} and the mathematics-reasoning benchmarks in Table \ref{table:4}.

%% file: Chapters_Appendix/7d_diagnostics.tex
\section{Continual Routing Analysis}\label{s:Routing}

We evaluate instruction-regularized queryable updates in a sequential continual-learning setting. For each run, the base model and adapter are initialized once, then trained on MBPP, GSM8K, and GPQA-Diamond, without resetting the LoRA factors, router, gates, keys, or shared rank-space atoms between tasks. The backbone remains frozen throughout. After each stage, we evaluate all three tasks and record the block-averaged atom-routing distribution for each evaluation set. This experiment is intended to test the central mechanism of our method: the effective adapter is a routed update of the form $\boldsymbol{B}_\ell(\boldsymbol{I}_r + g_\ell S_b(\boldsymbol{c})) \boldsymbol{A}_\ell$, where $S_b(\boldsymbol{c})$ is assembled from a shared memory of rank-space atoms. Hence, good continual behavior should appear as retained evaluation performance, structured reuse, and controlled drift of the atom routes.

Figure~\ref{fig:continual-atom-usage} shows the atom-usage distributions after each training stage. Each row is an evaluation task, and each column is a shared rank-space atom. The important pattern is sparse, non-uniform memory access: the router repeatedly uses a small subset of atoms without spreading mass uniformly across all atoms or collapsing to a single atom. This behavior is expected from a queryable memory. Atoms, such as a few high-mass columns, are reused across tasks, suggesting a shared low-rank correction structure, but their weights vary across MBPP, GSM8K, and GPQA-Diamond, showcasing that the instruction-conditioned state query still changes the retrieved operator. In other words, the same global memory is being reused, but not in a task-blind way.

\begin{figure*}[htb]
    \centering
    \begin{minipage}{0.32\linewidth}
        \centering
        \includegraphics[width=\linewidth]{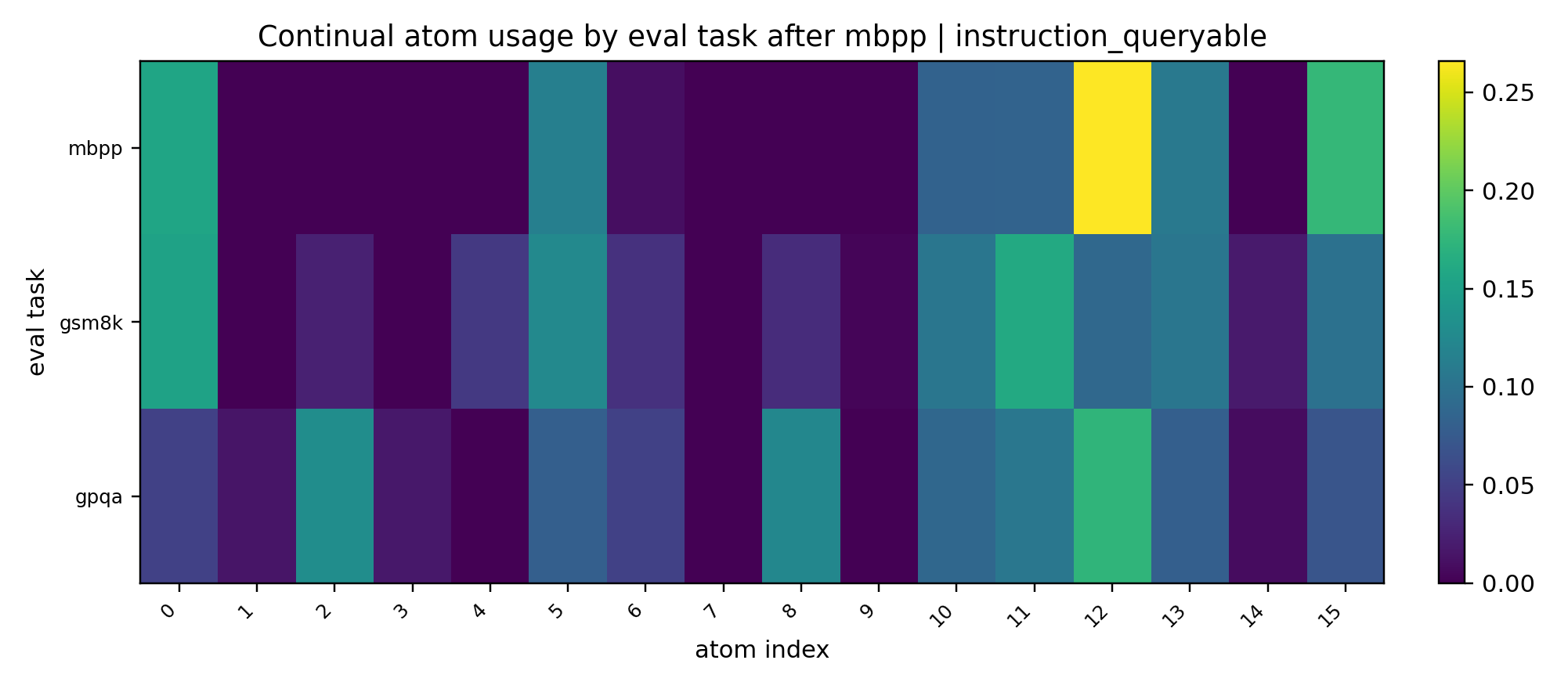}
        \vspace{-1mm}
        \centerline{\small after MBPP}
    \end{minipage}
    \hfill
    \begin{minipage}{0.32\linewidth}
        \centering
        \includegraphics[width=\linewidth]{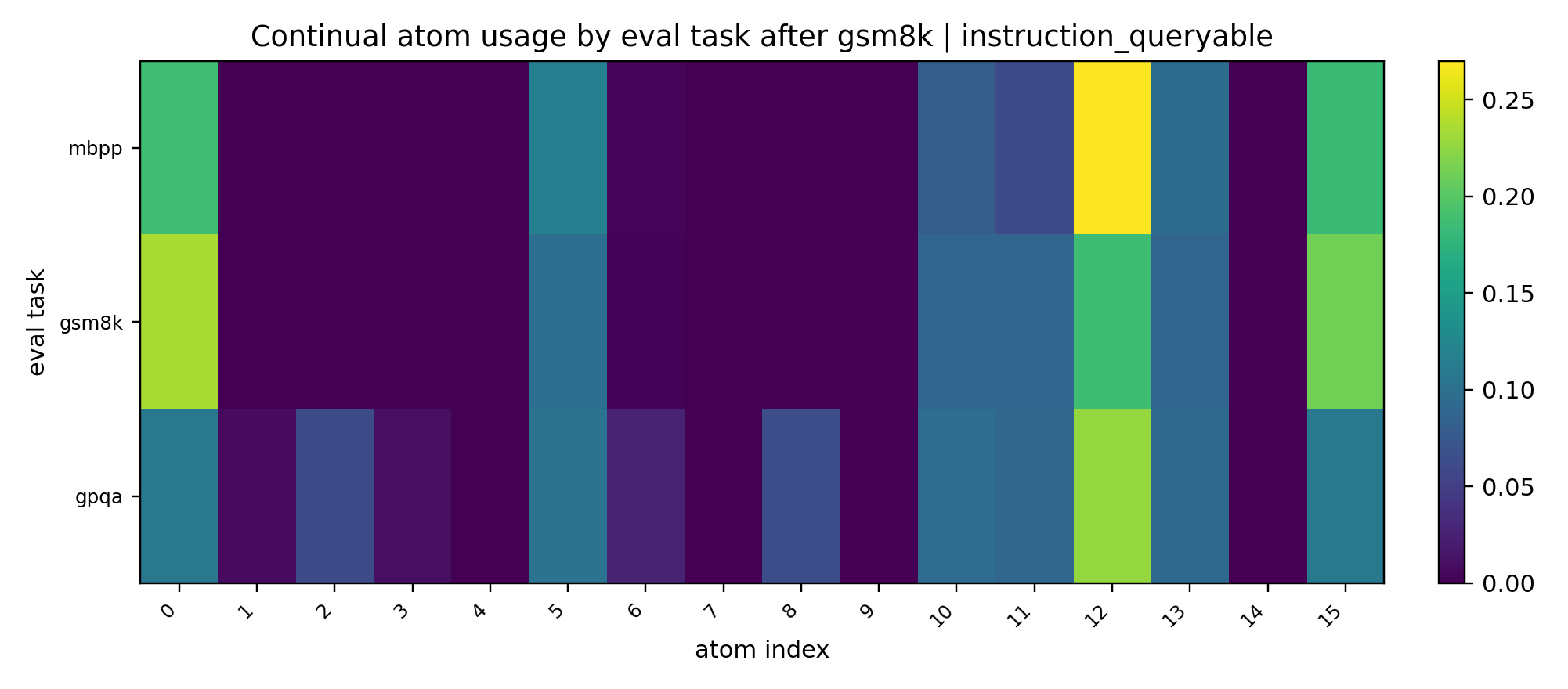}
        \vspace{-1mm}
        \centerline{\small after GSM8K}
    \end{minipage}
    \hfill
    \begin{minipage}{0.32\linewidth}
        \centering
        \includegraphics[width=\linewidth]{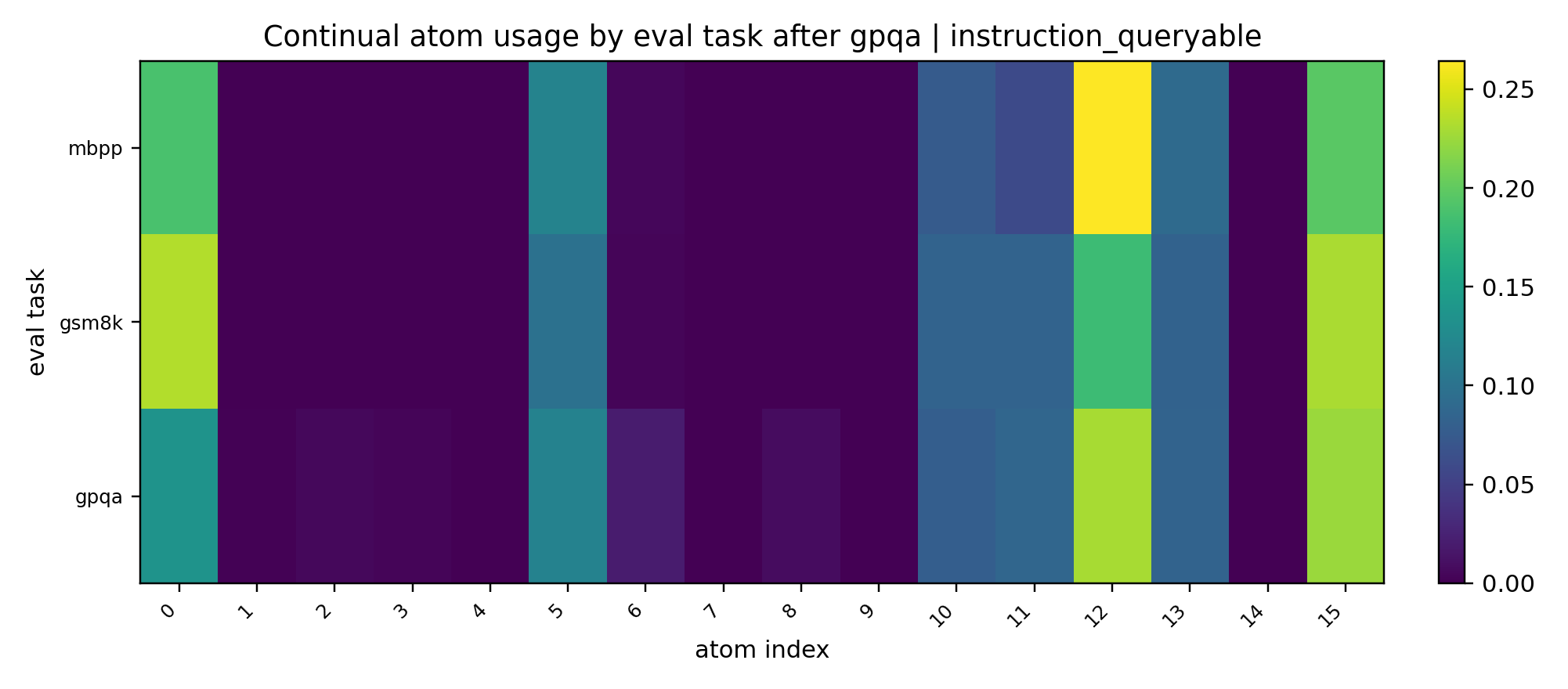}
        \vspace{-1mm}
        \centerline{\small after GPQA-Diamond}
    \end{minipage}
    \caption{Atom usage by evaluation task after each continual training stage. Sparse, repeated high-mass atoms indicate reusable shared memory, while task-dependent changes in mass indicate adaptive routing.}
    \label{fig:continual-atom-usage}
\end{figure*}

Figure~\ref{fig:continual-atom-drift} fixes the evaluation task and tracks how its atom usage changes as new tasks are learned. This view is essential since continual adaptation can fail either by being too rigid or by overwriting previous routes. The observed behavior is intermediate. MBPP keeps a stable high-mass route pattern across later stages, which is consistent with retention. GSM8K and GPQA-Diamond show more localized changes in a few atoms, which is expected because these tasks are introduced later and require new routing structure. The key point is that route drift is concentrated, which ensures that the model reallocates the probability mass among a small active set.

\begin{figure*}[htb]
    \centering
    \begin{minipage}{0.32\linewidth}
        \centering
        \includegraphics[width=\linewidth]{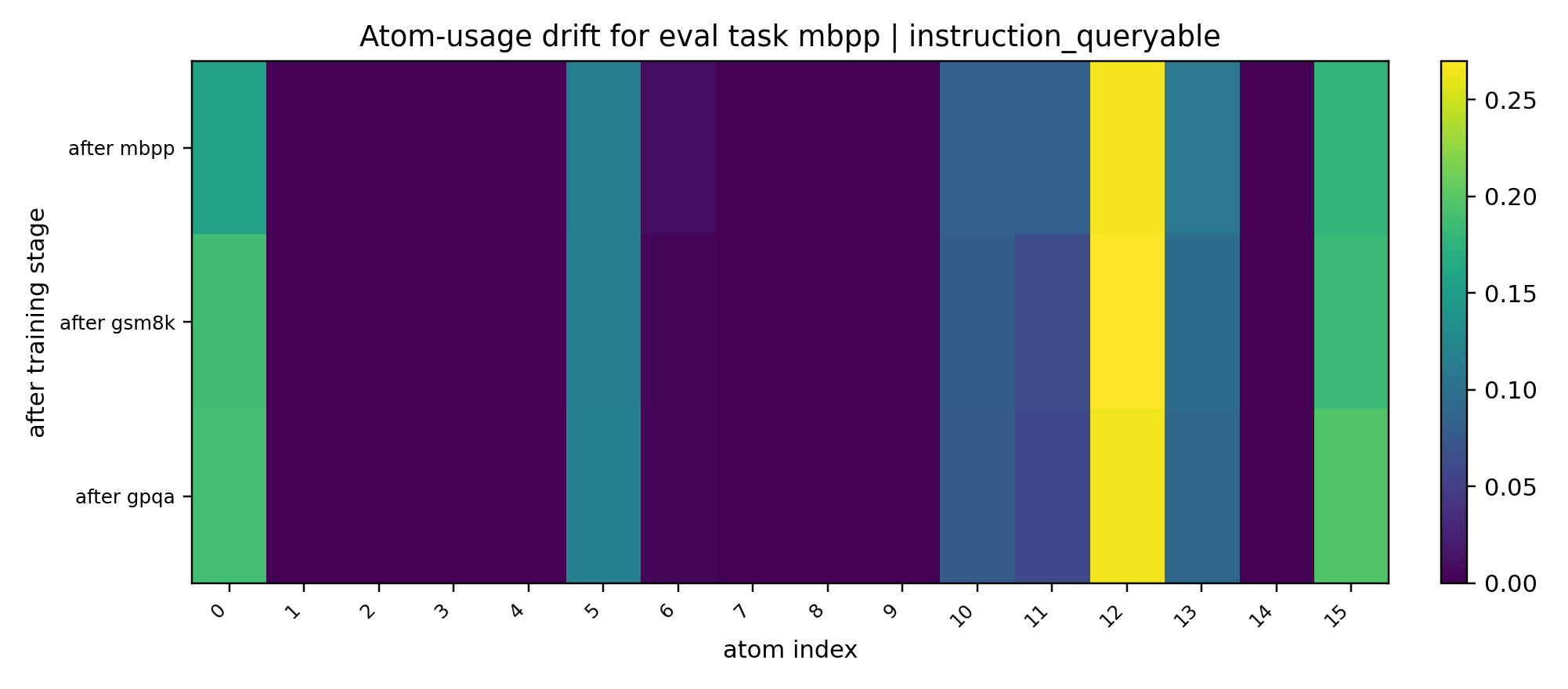}
        \vspace{-1mm}
        \centerline{\small eval MBPP}
    \end{minipage}
    \hfill
    \begin{minipage}{0.32\linewidth}
        \centering
        \includegraphics[width=\linewidth]{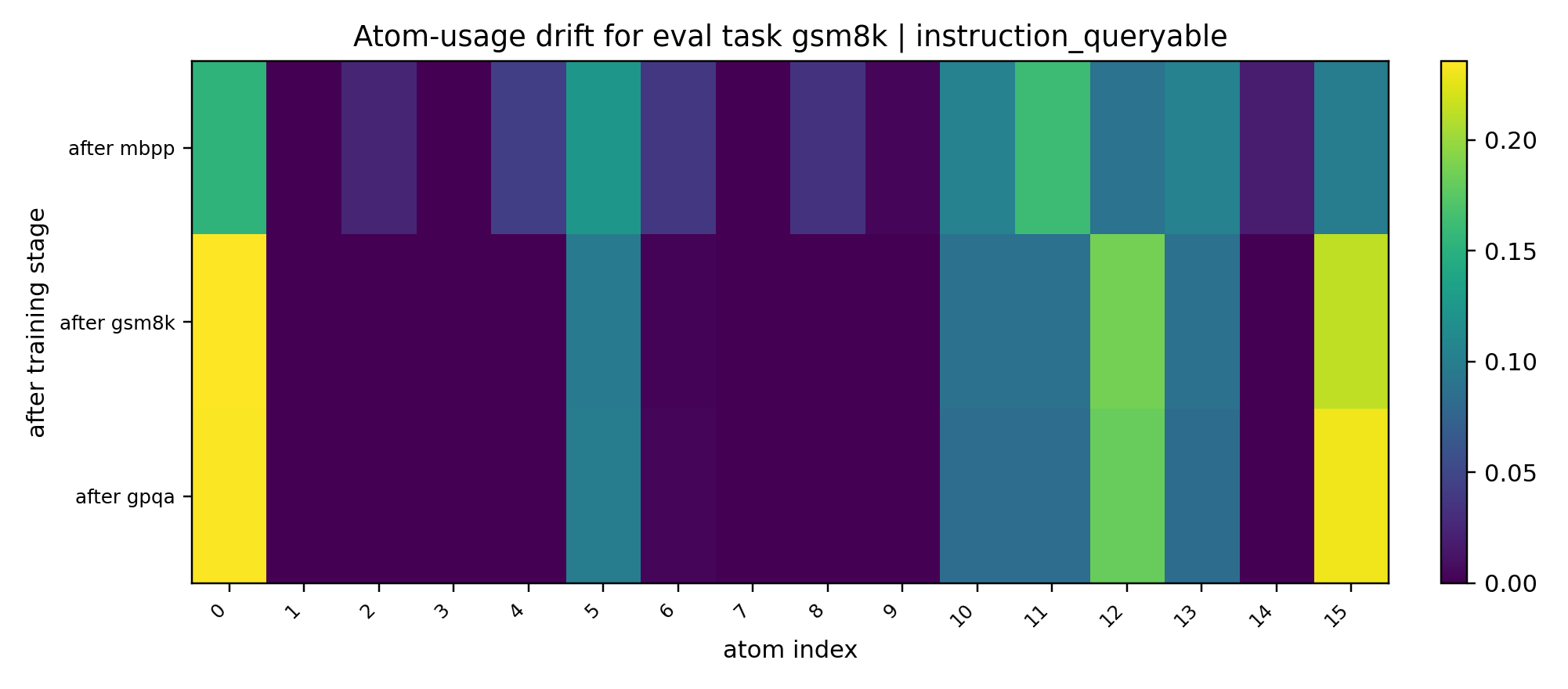}
        \vspace{-1mm}
        \centerline{\small eval GSM8K}
    \end{minipage}
    \hfill
    \begin{minipage}{0.32\linewidth}
        \centering
        \includegraphics[width=\linewidth]{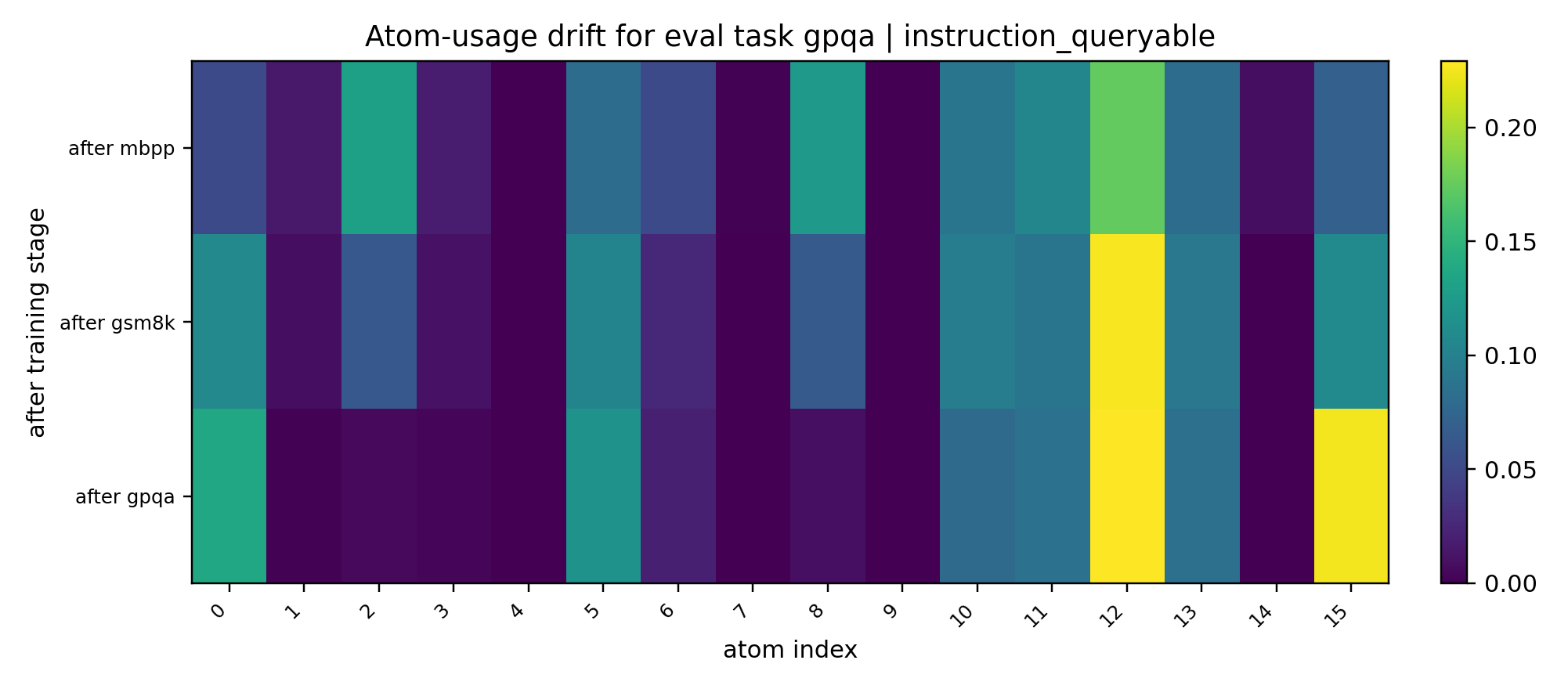}
        \vspace{-1mm}
        \centerline{\small eval GPQA-Diamond}
    \end{minipage}
    \caption{Atom-usage drift for each fixed evaluation task across training stages. The maps show that later tasks induce localized rerouting while preserving much of the earlier sparse route structure.}
    \label{fig:continual-atom-drift}
\end{figure*}

Figure~\ref{fig:continual-atom-specialization-entropy} summarizes the final-stage task specialization of each atom by measuring the entropy of its usage distribution across the evaluation tasks. High entropy indicates an atom that is reused broadly across MBPP, GSM8K, and GPQA, whereas a low entropy indicates an atom whose usage is concentrated on a smaller subset of tasks. Several atoms have near-maximal entropy and behave like shared reusable operators, and a small number of atoms have much lower entropy and hence appear more specialized.

\begin{figure}[htb]
    \centering
    \includegraphics[width=0.78\linewidth]{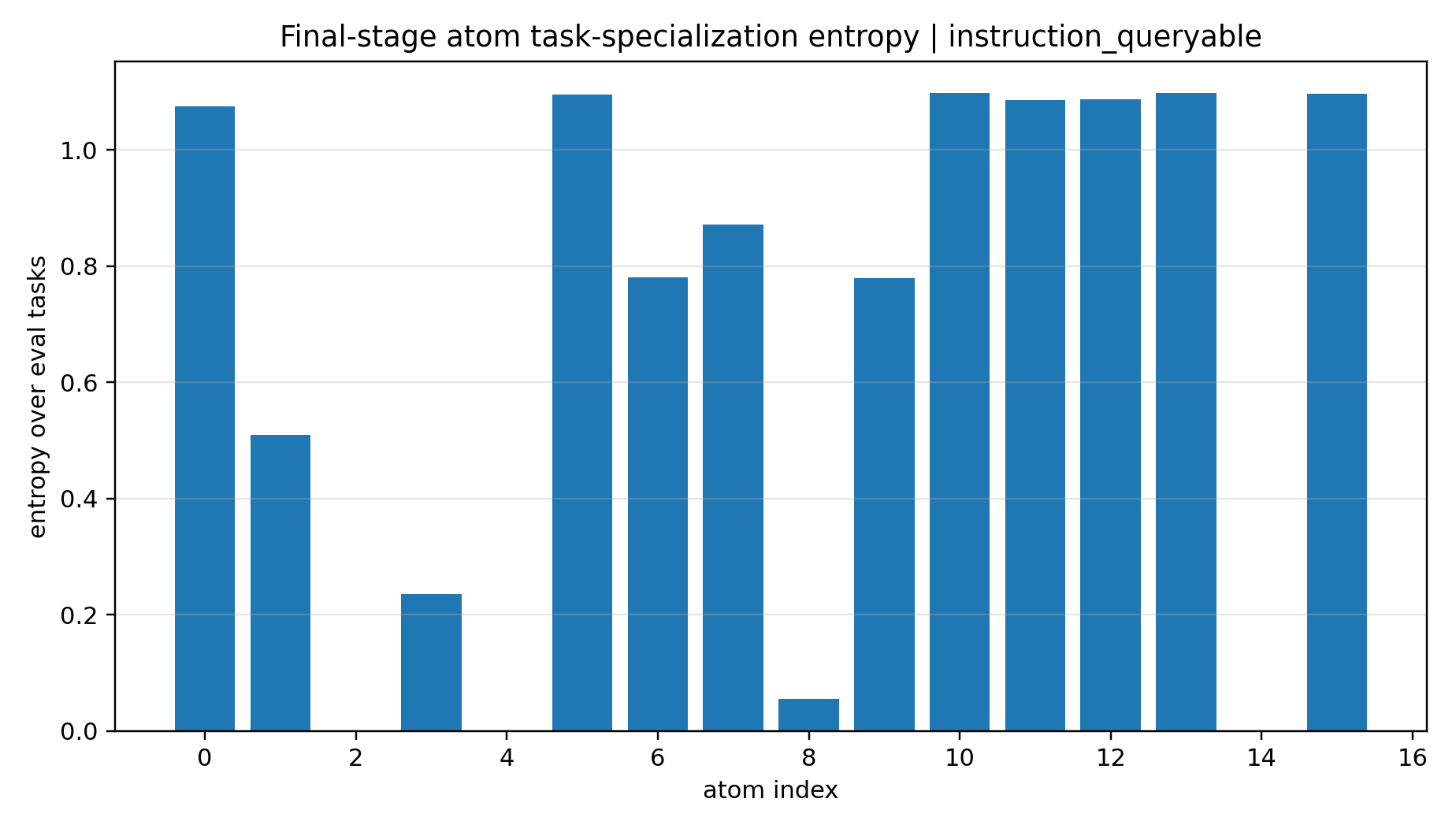}
    \caption{
    A mixture of high and low entropy atoms suggests that the shared memory bank supports both reusable and selective rank-space corrections.
    }
    \label{fig:continual-atom-specialization-entropy}
\end{figure}

Figure~\ref{fig:final-kl-and-stagewise-drift} gives two compact summaries of the same mechanism. The symmetric KL matrix compares final atom-usage distributions across evaluation tasks after the full sequence. MBPP and GSM8K are relatively close, whereas GPQA-Diamond is more distinct, aligning with the intuition that GPQA requires a different mixture over the shared memory. The stagewise route-drift matrix measures how much each evaluation route changes after a new task is learned. The largest drift occurs when GSM8K is introduced, especially on GSM8K itself, whereas the subsequent GSM8K$\rightarrow$GPQA-Diamond transition produces smaller changes on earlier tasks. Hence, we observe that the router can adapt effectively to a new task.

\begin{figure*}[htb]
    \centering
    \begin{minipage}{0.47\linewidth}
        \centering
        \includegraphics[width=\linewidth, trim={ 0 0 0 0.85cm},clip]{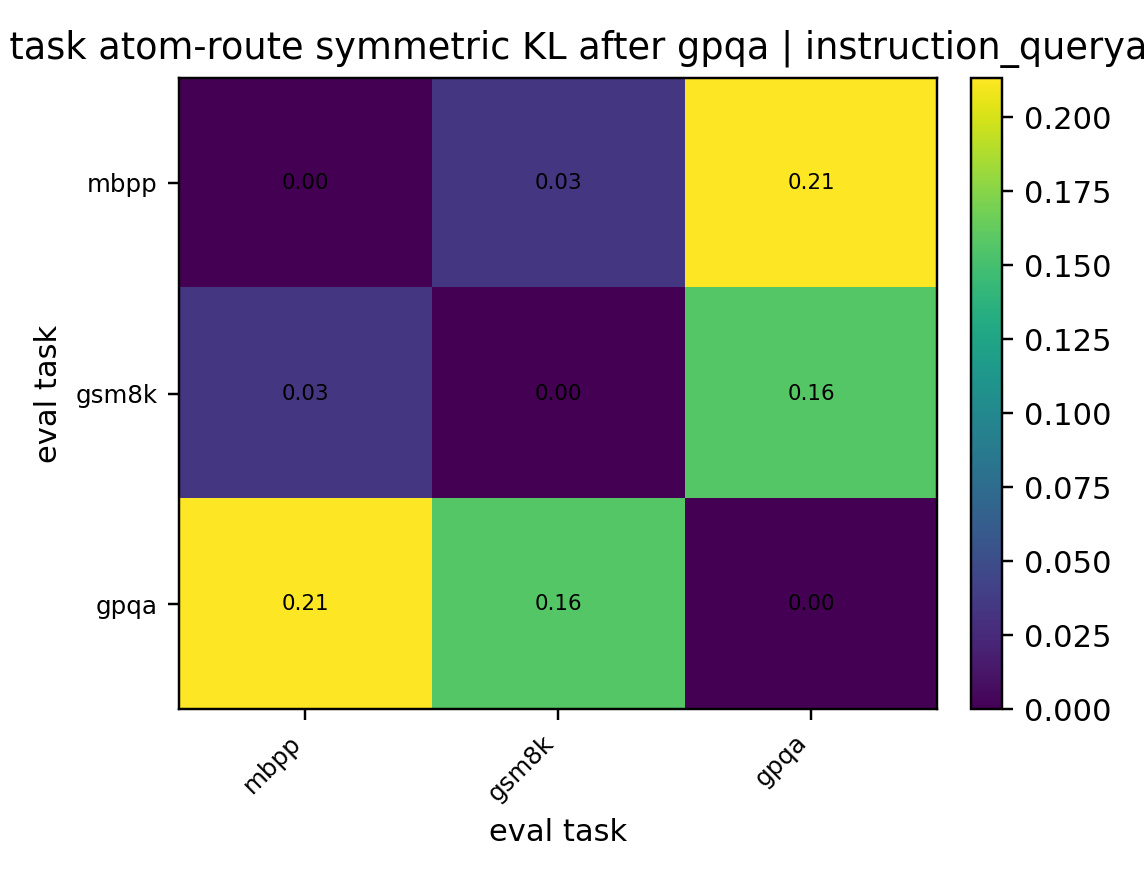}
        \vspace{-1mm}
        \centerline{\small \sc Final Symmetric KL}
    \end{minipage}
    \hfill
    \begin{minipage}{0.47\linewidth}
        \centering
        \includegraphics[width=\linewidth, trim={ 0 0 0 0.85cm},clip]{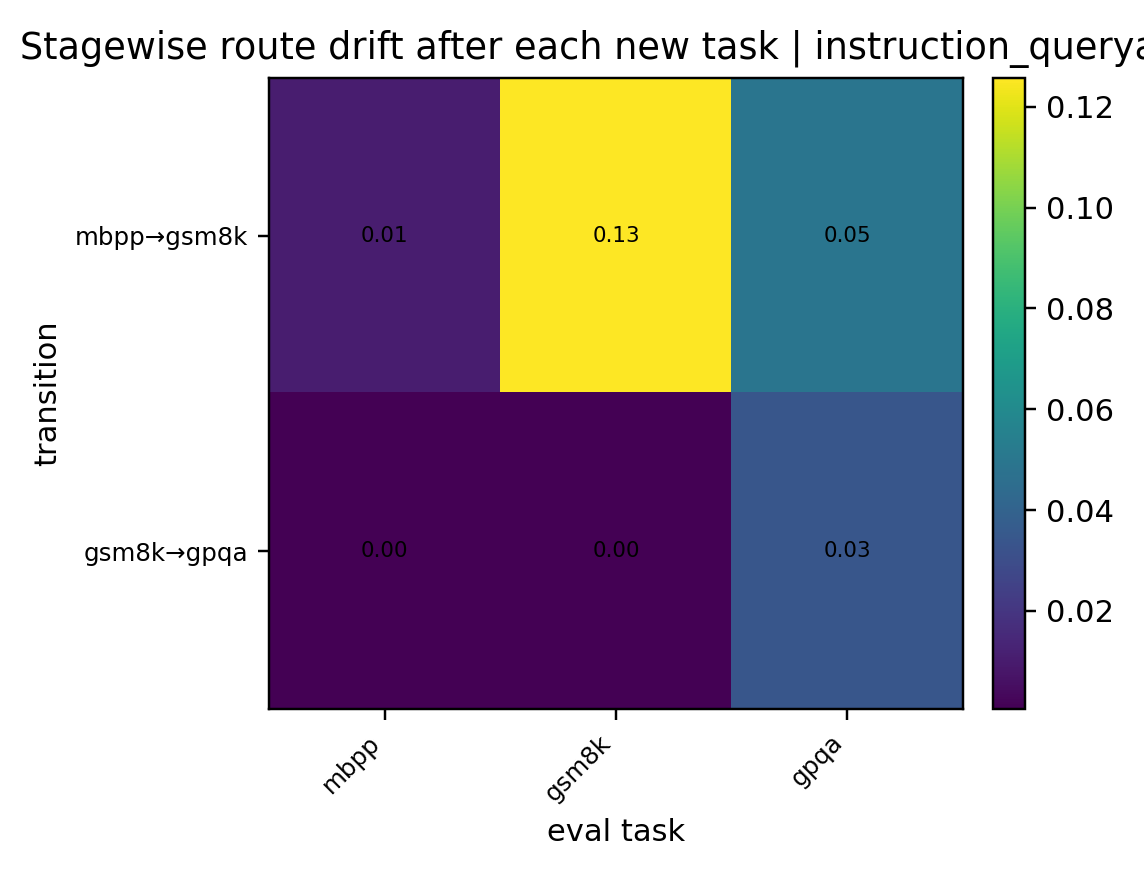}
        \vspace{-1mm}
        \centerline{\small \sc Stagewise Route Drift}
    \end{minipage}
    \caption{Final task separation and stagewise route drift. The KL matrix reflects that GPQA-Diamond uses a more distinct atom mixture after the full sequence. The drift matrix shows that new-task learning changes routes in a concentrated way.}
    \label{fig:final-kl-and-stagewise-drift}
\end{figure*}

%% file: Chapters_Appendix/7b_Additional_Theoretical_Analysis.tex
\section{Additional Theoretical Analysis}

\subsection{Setup}

We use the notations as stated in Sec.~\ref{s:ApproachMain}. For every block $b$, the router produces a probability vector
\[
\boldsymbol{\alpha}_b(c)=(\alpha_{b,1}(c),\ldots,\alpha_{b,M}(c))\in\Delta_M
\]
over a shared memory of rank-space atoms $\{\mathbf{C}_m\}_{m=1}^{M}\subset \mathbb{R}^{r\times r}$. The routed rank-space operator is
\begin{equation}
\mathbf{S}_b(c) := \sum_{m=1}^{M}\alpha_{b,m}(c)\mathbf{C}_m
\label{eq:app_routed_operator}
\end{equation}
For a layer $\ell\in\mathcal{B}_b$, the realized update is
\begin{align}
\Delta\mathbf{W}^{\ell}(\mathbf{h}_{\ell};b,c)
&= \frac{\alpha_{\mathrm{L}}}{r}\,
\mathbf{B}_{\ell}
\bigl(\mathbf{I}_r + g_{\ell}\mathbf{S}_b(c)\bigr)
\mathbf{A}_{\ell} \\
g_{\ell}&=\sigma(\eta_{\ell})\in(0,1)
\label{eq:app_layer_update}
\end{align}
where $\alpha_{\mathrm{L}}$ is the LoRA scaling factor and $r$ is the rank-space dimension. The symbol $\alpha_{\mathrm{L}}$ is used to avoid confusion with the routing weights $\alpha_{b,m}(c)$. The instruction encoder maps a natural-language instruction $c$ to a fixed embedding $\mathbf{e}(c)\in\mathbb{R}^{d_c}$. The state-dependent block query is:
\begin{equation}
\mathbf{q}_b^{(0)} = \mathbf{w}_{\ell_b} + \mathbf{Q}_{\mathrm{cur}}\mathbf{s}^{\mathrm{entry}}_{\ell_b} + \lambda_{\mathrm{ctx}}\mathbf{Q}_{\mathrm{ctx}}\mathbf{e}(c)
\label{eq:app_prequery}
\end{equation}
where $\ell_b$ is the first adapted layer in block $b$. The attention-style depth summary is:
\begin{align}
\mathbf{u}^{\mathrm{att}}_{b-1} &= \sum_{i=1}^{b-1}\beta_{i\to b}\bar{\mathbf{s}}_i \\
\beta_{i\to b} &= \frac{\exp(\xi_{i\to b})}{\sum_{j=1}^{b-1}\exp(\xi_{j\to b})}
\label{eq:app_depth_summary}
\end{align}
with the convention $\mathbf{u}^{\mathrm{att}}_0=\mathbf{0}$. The final query is:
\begin{equation}
\mathbf{q}_b = \mathbf{q}_b^{(0)}+\mathbf{Q}_{\mathrm{dep}}\mathbf{u}^{\mathrm{att}}_{b-1}
\label{eq:app_final_query}
\end{equation}
The depth-summary logits are
\begin{align}
\xi_{i\to b} &=
\frac{\langle \widehat{\mathbf{q}}^{(0)}_b,\widehat{\boldsymbol{\kappa}}_i\rangle}{\sqrt{d_k}T_{\mathrm{dep}}} \\
\widehat{\mathbf{q}}^{(0)}_b &=\text{RMSNorm}(\mathbf{Q}_{\mathrm{dep},q}\mathbf{q}^{(0)}_b) \\
\widehat{\boldsymbol{\kappa}}_i &= \text{RMSNorm}(\mathbf{Q}_{\mathrm{dep},k}\bar{\mathbf{s}}_i).
\label{eq:app_depth_logits}
\end{align}
The language-prior logits and state logits are
\begin{align}
\rho_m(c)
&= \frac{\langle \text{RMSNorm}(\mathbf{R}_{\mathrm{ctx}}\mathbf{e}(c)),\widehat{\mathbf{k}}_m\rangle}{\sqrt{d_k}T_{\mathrm{lang}}} \\ 
\zeta_{b,m} &= \frac{\langle \text{RMSNorm}(\mathbf{q}_b),\widehat{\mathbf{k}}_m\rangle}{\sqrt{d_k}T_{\mathrm{attn}}}
\label{eq:app_state_language_logits}
\end{align}
where $\widehat{\mathbf{k}}_m:=\text{RMSNorm}(\mathbf{k}_m)$. The language prior distribution is:
\begin{equation}
 p_m(c)=\frac{\exp(\rho_m(c))}{\sum_{j=1}^{M}\exp(\rho_j(c))}
\label{eq:app_language_prior}
\end{equation}
The router in the main paper uses fused logits
\begin{equation}
\widetilde{\zeta}_{b,m}(c) = \zeta_{b,m}+\tau_{\mathrm{lang}}\log p_m(c)
\label{eq:app_tilde_logits}
\end{equation}
For full routing, the active set is $I=\{1,\ldots,M\}$. For top-$k$ routing, $I=I_b(c)$ denotes the selected active set. On a fixed active set $I$, the routed weights are
\begin{equation}
\alpha_{b,m}(c) = \frac{\exp(\widetilde{\zeta}_{b,m}(c))}{\sum_{j\in I}\exp(\widetilde{\zeta}_{b,j}(c))}
\label{eq:app_active_router}
\end{equation}
for $m\in I,$ and $\alpha_{b,m}(c)=0$ for $m\notin I$ in the sparse case. The active set \(I_b\) is treated as fixed inside differentiability statements. This is automatic away from top-\(k\) switching boundaries and is the only local smoothness convention needed for logit-gradient calculations. Matrix inner products are Frobenius inner products, \(\langle \mathbf{U},\mathbf{V}\rangle_F={\rm tr}(\mathbf{U}^{\top}\mathbf{V})\). We consider the following assumptions:

\begin{assumption}
\label{ass:app_bounded_atoms_factors}
$\exists$ finite constants $R_C,R_A,R_B>0$ such that $\|\mathbf{C}_m\|\le R_C$, $
\|\mathbf{A}_{\ell}\|\le R_A$, $\|\mathbf{B}_{\ell}\|\le R_B$; $\forall \ $atom $m$ and every adapted layer $\ell$. Matrix norms are operator norms unless explicitly stated otherwise.
\end{assumption}

\begin{assumption}
\label{ass:app_bounded_states_instructions}
$\exists$ finite constants $R_s,R_e>0$ such that $\|\bar{\mathbf{s}}_i\|\le R_s$, 
$\|\mathbf{s}^{\mathrm{entry}}_{\ell_b}\|\le R_s$, $
\|\mathbf{e}(c)\|\le R_e$; $\forall \ $ completed block $i$, block entry layer $\ell_b$, and instruction $c$ in the region of interest.
\end{assumption}

\begin{assumption}
\label{ass:app_rms_region}
Let $\text{RMS}_{\varepsilon}(\mathbf{x}) := \sqrt{\frac{1}{d}\|\mathbf{x}\|_2^2+\varepsilon}$ and $\text{RMSNorm}(\mathbf{x}) :=
\frac{\mathbf{x}}{\text{RMS}_{\varepsilon}(\mathbf{x})}$.  $\exists \ \rho_{\min}>0$ such that every vector to which RMS normalization is applied in the router satisfies $\text{RMS}_{\varepsilon}(\mathbf{x})\ge \rho_{\min}$.
\end{assumption}

\begin{assumption}
\label{ass:app_fixed_active_set}
For the top-$k$ sparse router, all local differentiability and Lipschitz statements are restricted to a region where the active set $I_b(c)$ is constant. Equivalently, the $k$th and $(k+1)$st largest fused logits are separated by a positive margin throughout the region.
\end{assumption}

\subsection{Results}

Section \ref{s:ApproachMain} defines the language prior as:
\begin{align}
p_m(\boldsymbol{c}) &= \frac{\exp(\rho_m(\boldsymbol{c}))}{\sum_{j=1}^{M}\exp(\rho_j(\boldsymbol{c}))} \\
\tilde\zeta_{b,m}(\boldsymbol{c}) &= \zeta_{b,m} + \tau_{\mathrm{lang}}\log p_m(\boldsymbol{c})
\end{align}
where $p_m(c)$ is the language-prior distribution and $\tilde\zeta_{b,m}(c)$ are the fused routing logits. Lemma \ref{lem:app_equivalent_raw_logit_fusion} represents this formulation in an equivalent and simplified sense.

\begin{lemma} \label{lem:app_equivalent_raw_logit_fusion}
Define $z_{b,m}$ as $z_{b,m}(c) = \zeta_{b,m}+\tau_{\mathrm{lang}}\rho_m(c)$, where, $m = 1,2,3,..., M$. Then, both the usual softmax router and the top-$k$ softmax router obtained from the logits $\{ \tilde\zeta_{b,m}\}_{m=1}^M$ are equivalent to those obtained from $\{ z_{b,m}\}_{m=1}^M$. 
\end{lemma}

\begin{proof}
We observe that:
\begin{align}
\log p_m(\boldsymbol{c}) &= \rho_m(\boldsymbol{c}) - \log\Bigl(\sum_{j=1}^{M} e^{\rho_j(\boldsymbol{c})}\Bigr)\\
\implies \tilde\zeta_{b,m}(\boldsymbol{c}) &= \zeta_{b,m}+\tau_{\mathrm{lang}} \rho_m(\boldsymbol{c})
- \tau_{\mathrm{lang}}\log\Bigl(\sum_{j=1}^{M} e^{\rho_j(\boldsymbol{c})}\Bigr)
\end{align}
Hence, $\{ \tilde\zeta_{b,m}\}_{m=1}^M$ and $\{ z_{b,m}\}_{m=1}^M$ differ by only a scalar quantity, and thus, the softmax probabilities obtained from either logits are the same.
\end{proof}

The router can be understood as solving a precise retrieval problem over the shared atom memory. It selects atoms that match the current hidden state while remaining biased toward the semantic preference induced by the instruction. This is essential since language changes the retrieval distribution over a bounded set of learned update atoms, which gives the method a principled middle ground between static LoRA and unstable text-to-weight generation. We state the theorem \ref{thm:app_variational_instruction_retrieval} below.

\begin{theorem}[Variational characterization of instruction-regularized retrieval]
\label{thm:app_variational_instruction_retrieval}
Fix a block $b$, instruction $c$, and active set $I$. The routing distribution $\boldsymbol{\alpha}_{b,I}(c)$ defined in~\eqref{eq:app_active_router} is the unique solution of
\begin{equation}
\boldsymbol{\alpha}_{b,I}(c) = \arg\max_{\boldsymbol{a}\in\Delta(I)}
\left\{ \langle \boldsymbol{a},\boldsymbol{\zeta}_{b,I}\rangle -
\mathrm{KL}\!\left(\boldsymbol{a}\,\middle\|\,\boldsymbol{\pi}^{(\tau)}_{I}(c)\right) \right\}
\label{eq:app_variational_router_exact}
\end{equation}
where $\Delta(I)$ is the probability simplex on $I$. Equivalently,
\begin{equation}
\boldsymbol{\alpha}_{b,I}(c) = \arg\max_{\boldsymbol{a}\in\Delta(I)}
\left\{ \langle \boldsymbol{a},\boldsymbol{\zeta}_{b,I}+\tau_{\mathrm{lang}}\log \mathbf{p}_{I}(c)\rangle +H(\boldsymbol{a}) \right\}
\label{eq:app_entropy_router}
\end{equation}
where $H(\boldsymbol{a})=-\sum_{m\in I}a_m\log a_m$.
\end{theorem}

\begin{proof}
Let
\[
\Phi(\boldsymbol{a}) := \langle \boldsymbol{a},\boldsymbol{\zeta}_{b,I}\rangle -
\mathrm{KL}\!\left(\boldsymbol{a}\,\middle\|\,\boldsymbol{\pi}^{(\tau)}_I(c)\right)
\]
for $\boldsymbol{a}\in\Delta(I)$. Expanding the KL divergence gives:
\begin{align*}
\Phi(\boldsymbol{a}) &= \sum_{m\in I}a_m\zeta_{b,m} -
\sum_{m\in I}a_m\log\frac{a_m}{\pi^{(\tau)}_{I,m}(c)}\\
&= \sum_{m\in I}a_m\zeta_{b,m} - \sum_{m\in I}a_m\log a_m +
\sum_{m\in I}a_m\log\pi^{(\tau)}_{I,m}(c)
\end{align*}
The map $\boldsymbol{a}\mapsto -\sum_{m\in I}a_m\log a_m$ is strictly concave on the relative interior of the simplex. The remaining terms are linear in $\boldsymbol{a}$. Hence, $\Phi$ is strictly concave on $\Delta(I)$ and has at most one maximizer. To compute the maximizer, introduce a Lagrange multiplier $\lambda\in\mathbb{R}$ for the constraint $\sum_{m\in I}a_m=1$. Hence, the Lagrangian is:
\[
\mathcal{J}(\boldsymbol{a},\lambda) = \sum_{m\in I}a_m\zeta_{b,m} - \sum_{m\in I}a_m\log\frac{a_m}{\pi^{(\tau)}_{I,m}(c)} + \lambda\left(\sum_{m\in I}a_m-1\right)
\]
$\forall$ $m\in I$,
\begin{align*}
\frac{\partial \mathcal{J}}{\partial a_m} &= \zeta_{b,m} -
\frac{\partial}{\partial a_m}\left(a_m\log\frac{a_m}{\pi^{(\tau)}_{I,m}(c)}\right) +\lambda\\
&= \zeta_{b,m}-\log a_m+\log\pi^{(\tau)}_{I,m}(c)-1+\lambda
\end{align*}
From the first-order condition, $\frac{\partial \mathcal{J}}{\partial a_m} = 0$. Hence,
\begin{align}
0 &=\zeta_{b,m}-\log a_m+\log\pi^{(\tau)}_{I,m}(c)-1+\lambda \\
\implies a_m &= \exp(\lambda-1)\,\pi^{(\tau)}_{I,m}(c)\exp(\zeta_{b,m})
\end{align}
$\exp(\lambda-1)$ is independent of $m$ and is determined by the simplex constraint. Hence,
\begin{equation}
 a_m = \frac{\pi^{(\tau)}_{I,m}(c)\exp(\zeta_{b,m})}{\sum_{j\in I}\pi^{(\tau)}_{I,j}(c)\exp(\zeta_{b,j})}
\label{eq:app_variational_softmax_intermediate}
\end{equation}
Since, $\pi^{(\tau)}_{I,m}(c) = \frac{p_m(c)^{\tau_{\mathrm{lang}}}}{\sum_{q\in I}p_q(c)^{\tau_{\mathrm{lang}}}}$,
\begin{align*}
 a_m = \frac{\exp(\zeta_{b,m})p_m(c)^{\tau_{\mathrm{lang}}}}{\sum_{j\in I}\exp(\zeta_{b,j})p_j(c)^{\tau_{\mathrm{lang}}}} =
\frac{\exp(\zeta_{b,m}+\tau_{\mathrm{lang}}\log p_m(c))}{\sum_{j\in I}\exp(\zeta_{b,j}+\tau_{\mathrm{lang}}\log p_j(c))}
\end{align*}
This is the required active-set router in ~\eqref{eq:app_active_router}. Since the objective is strictly concave, this maximizer is unique. To show the entropy form, by \ref{eq:app_active_router},
\[
\log\pi^{(\tau)}_{I,m}(c) = \tau_{\mathrm{lang}}\log p_m(c) - \log\left(\sum_{j\in I}p_j(c)^{\tau_{\mathrm{lang}}}\right)
\]
Using this identity in the expanded expression for $\Phi$ gives
\begin{equation*}
\Phi(\boldsymbol{a}) = \langle \boldsymbol{a},\boldsymbol{\zeta}_{b,I}\rangle
+ H(\boldsymbol{a}) +\tau_{\mathrm{lang}}\langle  \boldsymbol{a},\log\mathbf{p}_I(c)\rangle -\log\left(\sum_{j\in I}p_j(c)^{\tau_{\mathrm{lang}}}\right) \sum_{m\in I}a_m
\end{equation*}
Now, since, $\boldsymbol{a}\in\Delta(I)$, $\sum_{m\in I}a_m=1$. Hence, the final term is constant in $\boldsymbol{a}$ and does not change the optimizer. Removing this constant gives \ref{eq:app_entropy_router}.
\end{proof}

The state-prior tradeoff, in \ref{cor:app_state_prior_tradeoff}, clarifies what is lost when the instruction is allowed to influence routing. The hidden state may prefer one update atom, while the instruction prior may prefer another.  \ref{cor:app_state_prior_tradeoff}  demonstrates that the state-matching score lost by using instruction-guided retrieval, is small, whenever the tempered instruction prior assigns sufficient weight to the atom that the hidden state would choose on its own.

\begin{corollary}[State-prior tradeoff bound]
\label{cor:app_state_prior_tradeoff}
Let $m_b^{\star}\in\arg\max_{m\in I}\zeta_{b,m}$. Under the hypotheses of Theorem~\ref{thm:app_variational_instruction_retrieval},
\begin{equation}
0 \le \max_{m\in I}\zeta_{b,m} - \langle \boldsymbol{\alpha}_{b,I}(c),\boldsymbol{\zeta}_{b,I}\rangle \le \log\frac{1}{\pi^{(\tau)}_{I,m_b^{\star}}(c)}
\label{eq:app_state_prior_tradeoff}
\end{equation}
\end{corollary}

\begin{proof}
The left inequality follows because $\boldsymbol{\alpha}_{b,I}(c)$ is a probability vector. Indeed,
\begin{align*}
&\langle \boldsymbol{\alpha}_{b,I}(c),\boldsymbol{\zeta}_{b,I}\rangle = \sum_{m\in I}\alpha_{b,m}(c)\zeta_{b,m} \le \sum_{m\in I}\alpha_{b,m}(c)\max_{j\in I}\zeta_{b,j} = \max_{j\in I}\zeta_{b,j} \\ &\implies\max_{m\in I}\zeta_{b,m} - \langle \boldsymbol{\alpha}_{b,I}(c),\boldsymbol{\zeta}_{b,I}\rangle
\ge 0
\end{align*}
For the upper bound, define:
\[
\Phi(\boldsymbol{a}) = \langle \boldsymbol{a},\boldsymbol{\zeta}_{b,I}\rangle -\text{KL}\!\left(\boldsymbol{a}\,\middle\|\,\boldsymbol{\pi}^{(\tau)}_I(c)\right)
\]
By Theorem~\ref{thm:app_variational_instruction_retrieval}, $\boldsymbol{\alpha}_{b,I}(c)$ maximizes $\Phi$ over $\Delta(I)$. Let $\mathbf{e}_{m_b^{\star}}$ denote the simplex vertex that places all mass on $m_b^{\star}$. Then,
\[ \Phi(\boldsymbol{\alpha}_{b,I}(c)) \ge \Phi(\mathbf{e}_{m_b^{\star}})
\]
Expanding both sides gives
\[ \langle \boldsymbol{\alpha}_{b,I}(c),\boldsymbol{\zeta}_{b,I}\rangle - \text{KL}\!\left(\boldsymbol{\alpha}_{b,I}(c)\,\middle\|\,\boldsymbol{\pi}^{(\tau)}_I(c)\right) \ge \zeta_{b,m_b^{\star}} - \text{KL}\!\left(\mathbf{e}_{m_b^{\star}}\,\middle\|\,\boldsymbol{\pi}^{(\tau)}_I(c)\right)
\]
The KL divergence of a vertex against $\boldsymbol{\pi}^{(\tau)}_I(c)$ is
\begin{align} \label{eq:15}
\text{KL}\!\left(\mathbf{e}_{m_b^{\star}}\,\middle\|\,\boldsymbol{\pi}^{(\tau)}_I(c)\right) &= \log\frac{1}{\pi^{(\tau)}_{I,m_b^{\star}}(c)}\\
\implies \langle \boldsymbol{\alpha}_{b,I}(c),\boldsymbol{\zeta}_{b,I}\rangle -
\text{KL}\!\left(\boldsymbol{\alpha}_{b,I}(c)\,\middle\|\,\boldsymbol{\pi}^{(\tau)}_I(c)\right)
&\ge \zeta_{b,m_b^{\star}} - \log\frac{1}{\pi^{(\tau)}_{I,m_b^{\star}}(c)}
\end{align}
As $\text{KL}(.\|,.) \ge 0$, 
\begin{align} \label{eq:16}
\langle \boldsymbol{\alpha}_{b,I}(c),\boldsymbol{\zeta}_{b,I}\rangle \ge \langle \boldsymbol{\alpha}_{b,I}(c),\boldsymbol{\zeta}_{b,I}\rangle - \text{KL}\!\left(\boldsymbol{\alpha}_{b,I}(c)\,\middle\|\,\boldsymbol{\pi}^{(\tau)}_I(c)\right)
\end{align}
From \ref{eq:15} and \ref{eq:16},
\begin{equation} \langle \boldsymbol{\alpha}_{b,I}(c),\boldsymbol{\zeta}_{b,I}\rangle \ge \zeta_{b,m_b^{\star}} - \log\frac{1}{\pi^{(\tau)}_{I,m_b^{\star}}(c)}
\end{equation}
Rearranging gives us the upper bound in \ref{eq:app_state_prior_tradeoff}.
\end{proof}
Next, the Proposition \ref{prop:app_limiting_language_regimes} clarifies the importance of the instruction strength parameter $\tau_{\mathrm{lang}}$. When $\tau_{\mathrm{lang}} = 0$, the routing distribution is determined entirely by the state-dependent logits, and hence, the method reduces to the queryable update-memory mechanism without the instruction guidance. However, when $\tau_{\mathrm{lang}} \rightarrow \infty$, the language prior dominates the routing distribution; thus, the selected atoms are mainly those favored by the external task description. We observe that $\tau_{\mathrm{lang}}$ demonstrates an explicit interpolation between these directions, and that the instruction signal is a controlled regularizer of the existing router.

\begin{proposition}[Limiting language regimes]
\label{prop:app_limiting_language_regimes}
The instruction-regularized queryable update has the following limiting cases.
\begin{enumerate}[label=(\alph*),leftmargin=2.2em]
    \item If $\lambda_{\mathrm{ctx}}=0$ and $\tau_{\mathrm{lang}}=0$, then the router reduces to the purely state-dependent queryable router.
    \item If $g_{\ell}=0$ for every adapted layer, then the realized update reduces to the static LoRA update $\Delta\mathbf{W}^{\ell}=(\alpha_{\mathrm{L}}/r)\mathbf{B}_{\ell}\mathbf{A}_{\ell}$.
    \item Fix an active set $I$. Suppose that $\rho_m(c)$ has a unique maximizer $m_I^{\dagger}$ on $I$. If the state logits $\{\zeta_{b,m}\}_{m\in I}$ are finite and $\tau_{\mathrm{lang}}\to\infty$, then: 
    \begin{align}
    \alpha_{b,m_I^{\dagger}}(c)&\to 1 \\
    \alpha_{b,m}(c)&\to 0\quad \forall \ m\neq m_I^{\dagger} \\
    \mathbf{S}_b(c)&\to \mathbf{C}_{m_I^{\dagger}}
    \end{align}
    \item If the state logits are constant on $I$, then the router is the tempered instruction prior $\boldsymbol{\pi}^{(\tau)}_I(c)$. Specifically, if $\tau_{\mathrm{lang}}=1$, the router is exactly the restricted and renormalized language prior on $I$.
\end{enumerate}
\end{proposition}

\begin{proof}
For part (a), if $\lambda_{\mathrm{ctx}}=0$, then the pre-query in~\eqref{eq:app_prequery} becomes
\[ \mathbf{q}_b^{(0)} = \mathbf{w}_{\ell_b}+\mathbf{Q}_{\mathrm{cur}}\mathbf{s}^{\mathrm{entry}}_{\ell_b}
\]
Hence, the instruction embedding no longer enters the query. If $\tau_{\mathrm{lang}}=0$, then the fused logit in~\eqref{eq:app_tilde_logits} becomes
\[
\widetilde{\zeta}_{b,m}(c)=\zeta_{b,m}
\]
Thus, both the active-set selection and the softmax weights depend only on the state logits, which is the state-dependent queryable router. For part (b), substitute $g_{\ell}=0$ into~\eqref{eq:app_layer_update}:
\[ \Delta\mathbf{W}^{\ell}(\mathbf{h}_{\ell};b,c) = \frac{\alpha_{\mathrm{L}}}{r}\mathbf{B}_{\ell} (\mathbf{I}_r+0\cdot \mathbf{S}_b(c))\mathbf{A}_{\ell} = \frac{\alpha_{\mathrm{L}}}{r}\mathbf{B}_{\ell}\mathbf{A}_{\ell}
\]
This is the standard static LoRA update. For part (c), by Lemma \ref{lem:app_equivalent_raw_logit_fusion}, on $I$ the router can be written as
\[
\alpha_{b,m}(c) = \frac{\exp(\zeta_{b,m}+\tau_{\mathrm{lang}}\rho_m(c))}{\sum_{j\in I}\exp(\zeta_{b,j}+\tau_{\mathrm{lang}}\rho_j(c))}
\]
Let $m_I^{\dagger}$ be the unique maximizer of $\rho_m(c)$ on $I$. $\forall \ m\neq m_I^{\dagger}$,
\[
\frac{\alpha_{b,m}(c)}{\alpha_{b,m_I^{\dagger}}(c)} = \exp\left( \zeta_{b,m} - \zeta_{b,m_I^{\dagger}} + \tau_{\mathrm{lang}}(\rho_m(c)-\rho_{m_I^{\dagger}}(c)) \right)
\]
Since $m_I^{\dagger}$ is the unique maximizer, $\rho_m(c)-\rho_{m_I^{\dagger}}(c)<0$ $\ \forall \ m\neq m_I^{\dagger}$. The state-logit difference is finite by assumption. Hence, $\left( \zeta_{b,m}-\zeta_{b,m_I^{\dagger}} + \tau_{\mathrm{lang}}(\rho_m(c)-\rho_{m_I^{\dagger}}(c)) \right) \to -\infty$ as $\tau_{\mathrm{lang}}\to\infty$, hence, the ratio tends to zero. Since the routing weights sum to one, $\implies \alpha_{b,m_I^{\dagger}}(c)\to 1$ and $\alpha_{b,m}(c)\to 0$ for all $m\neq m_I^{\dagger}$. Substituting into~\eqref{eq:app_routed_operator} gives $\mathbf{S}_b(c)\to\mathbf{C}_{m_I^{\dagger}}$. For part (d), suppose $\zeta_{b,m}=\gamma$ $\ \forall \ m\in I$. Then,
\[
\alpha_{b,m}(c) = \frac{\exp(\gamma)p_m(c)^{\tau_{\mathrm{lang}}}}{\sum_{j\in I}\exp(\gamma)p_j(c)^{\tau_{\mathrm{lang}}}} = \frac{p_m(c)^{\tau_{\mathrm{lang}}}}{\sum_{j\in I}p_j(c)^{\tau_{\mathrm{lang}}}} =\pi^{(\tau)}_{I,m}(c)
\]
If $\tau_{\mathrm{lang}}=1$, this is the required restricted language prior renormalized on $I$.
\end{proof}

We note that, even when $\tau_{\mathrm{lang}} \rightarrow \infty$, the operator converges to a preexisting shared atom $C_{m_I^{\dagger}(\boldsymbol{c})}$ as opposed to an arbitrary matrix generated from the text. This is why we believe that our approach is different from \cite{TextToLoRA2025} as it remains a retrieval-and-composition mechanism over a bounded memory bank as compared to a direct text-to-weight generator in \cite{TextToLoRA2025}. 

The depth-summary bounds in \ref{prop:app_convex_hull_bound}, \ref{cor:app_update_norm_bound}, \ref{prop:app_depth_summary_bounded}, \ref{cor:app_query_norm_bound} prove that the attention-based summary of previous blocks stays bounded and stable. This stability ensures that information from earlier layers does not grow uncontrollably during retrieval. As a result, the router reliably uses the network's computational history to make decisions while maintaining a consistent representation space. \ref{prop:app_convex_hull_bound} focuses specifically on the bounds for the operator with regard to the shared atom bank.

\begin{proposition} \label{prop:app_convex_hull_bound}
For every block $b$ and instruction $\boldsymbol{c}$, we have,
\begin{align}
S_b(\boldsymbol{c}) &= \sum_{m=1}^{M}\alpha_{b,m}(\boldsymbol{c}) C_m \ \in \operatorname{conv}\{C_1, C_2, C_3, \dots,C_M\} \\
\|S_b(\boldsymbol{c})\| &\le \sum_{m=1}^{M} \alpha_{b,m}(\boldsymbol{c})\|C_m\| \le \max_{1 \le m \le M}\|C_m\| \le R_C
\end{align}
where, $\operatorname{conv}(.)$ is the convex hull of atoms. A similar bound also holds true for the top-$k$ router, as its weights also lie in the simplex.
\end{proposition}

\begin{proof}
From construction,  $\alpha_{b,m}(c) \ge 0$ and $\sum_{m=1}^{M}\alpha_{b,m}(\boldsymbol{c}) = 1$. Hence, $S_b(\boldsymbol{c})$ is a convex combination of atoms $\{C_m\}_{m=1}^{M}$, which demonstrates $S_b(c) \in \operatorname{conv}\{C_1, C_2, C_3, \dots,C_M\}$. Additionally,
\begin{multline}
\|S_b(c)\| = \Bigl\|\sum_{m=1}^{M}\alpha_{b,m}(c) C_m\Bigr\| \le \sum_{m=1}^{M}\alpha_{b,m}(c)\|C_m\| \\ \le \max_{1\le  m\le M}\| C_m\|\sum_{m=1}^{M}\alpha_{b,m}(c) = \max_{1\le  m \le M} \|C_m\|
\end{multline}
Using assumption \ref{ass:app_bounded_atoms_factors},  $\|S_b(\boldsymbol{c})\| \leq R_C$. 
\end{proof}

\begin{corollary}[Layerwise norm-to-norm bound]
\label{cor:app_update_norm_bound}
For any adapted layer $\ell\in\mathcal{B}_b$,
\begin{equation}
\|\Delta\mathbf{W}^{\ell}(\mathbf{h}_{\ell};b,c)\| \le \frac{\alpha_{\mathrm{L}}}{r}\,\|\mathbf{B}_{\ell}\|\bigl(1+|g_{\ell}|R_C\bigr)\|\mathbf{A}_{\ell}\|
\label{eq:app_update_matrix_norm_bound}
\end{equation}
Consequently, for any vector $\mathbf{x}$,
\begin{equation}
\|\Delta\mathbf{W}^{\ell}(\mathbf{h}_{\ell};b,c)\mathbf{x}\| \le \frac{\alpha_{\mathrm{L}}}{r}\,\|\mathbf{B}_{\ell}\|\bigl(1+|g_{\ell}|R_C\bigr)\|\mathbf{A}_{\ell}\|\,\|\mathbf{x}\|
\label{eq:app_update_vector_norm_bound}
\end{equation}
Under Assumption~\ref{ass:app_bounded_atoms_factors},
\begin{equation}
\|\Delta\mathbf{W}^{\ell}(\mathbf{h}_{\ell};b,c)\|
\le
\frac{\alpha_{\mathrm{L}}}{r}\,R_B(1+R_C)R_A
\label{eq:app_update_uniform_norm_bound}
\end{equation}
as $g_{\ell}\in(0,1)$.
\end{corollary}

\begin{proof}
From~\eqref{eq:app_layer_update},
\[
\Delta\mathbf{W}^{\ell} = \frac{\alpha_{\mathrm{L}}}{r}\mathbf{B}_{\ell} \bigl(\mathbf{I}_r+g_{\ell}\mathbf{S}_b(c)\bigr)\mathbf{A}_{\ell}
\]
Using submultiplicativity of the operator norm,
\begin{align} \label{eq:18}
\|\Delta\mathbf{W}^{\ell}\| &\le \frac{\alpha_{\mathrm{L}}}{r}\, \|\mathbf{B}_{\ell}\|\, \|\mathbf{I}_r+g_{\ell}\mathbf{S}_b(c)\|\, \|\mathbf{A}_{\ell}\|
\end{align}
From triangle inequality,
\[
\|\mathbf{I}_r+g_{\ell}\mathbf{S}_b(c)\| \le \|\mathbf{I}_r\|+|g_{\ell}|\|\mathbf{S}_b(c)\|
\]
$\|\mathbf{I}_r \|\ = 1$, and Proposition~\ref{prop:app_convex_hull_bound} gives $\|\mathbf{S}_b(c)\|\le R_C$. Hence,
\begin{equation} \label{eq:17}
\|\mathbf{I}_r+g_{\ell}\mathbf{S}_b(c)\|
\le
1+|g_{\ell}|R_C
\end{equation}
Substituting \ref{eq:17} in \ref{eq:18} proves~\eqref{eq:app_update_matrix_norm_bound}.  \ref{eq:app_update_vector_norm_bound} follows from the definition of the subordinate operator norm:
\[ \|\Delta\mathbf{W}^{\ell}\mathbf{x}\| \le \|\Delta\mathbf{W}^{\ell}\|\,\|\mathbf{x}\|
\]
Finally, Assumption~\ref{ass:app_bounded_atoms_factors} and $0<g_{\ell}<1$ yield~\eqref{eq:app_update_uniform_norm_bound}.
\end{proof}

\begin{proposition}[Bounded attention-style depth summary]
\label{prop:app_depth_summary_bounded}
For every block $b\ge 2$,
\begin{equation}
\|\mathbf{u}^{\mathrm{att}}_{b-1}\| \le \max_{1\le i\le b-1}\|\bar{\mathbf{s}}_i\|
\label{eq:app_depth_summary_bound}
\end{equation}
Under the Assumption~\ref{ass:app_bounded_states_instructions}, $\|\mathbf{u}^{\mathrm{att}}_{b-1}\|\le R_s$. For $b=1$, the convention $\mathbf{u}^{\mathrm{att}}_0=\mathbf{0}$ gives the same bound.
\end{proposition}

\begin{proof}
For $b\ge 2$, the attention weights satisfy
\begin{align*}
\beta_{i\to b} &\ge 0 \\
\sum_{i=1}^{b-1}\beta_{i\to b}&=1
\end{align*}
Using~\eqref{eq:app_depth_summary},
\begin{multline*}
\|\mathbf{u}^{\mathrm{att}}_{b-1}\| = \left\|\sum_{i=1}^{b-1}\beta_{i\to b}\bar{\mathbf{s}}_i\right\| \le \sum_{i=1}^{b-1}\beta_{i\to b}\|\bar{\mathbf{s}}_i\| \le \sum_{i=1}^{b-1}\beta_{i\to b}\max_{1\le j\le b-1}\|\bar{\mathbf{s}}_j\|\\ =\max_{1\le j\le b-1}\|\bar{\mathbf{s}}_j\|\sum_{i=1}^{b-1}\beta_{i\to b} =\max_{1\le j\le b-1}\|\bar{\mathbf{s}}_j\|
\end{multline*}
$\|\mathbf{u}^{\mathrm{att}}_{b-1}\|\le R_s$ follows from Assumption~\ref{ass:app_bounded_states_instructions}. If $b=1$, then $\mathbf{u}^{\mathrm{att}}_0=\mathbf{0}$, hence, $\|\mathbf{u}^{\mathrm{att}}_0\|=0\le R_s$.
\end{proof}

\begin{corollary}[Query norm bound]
\label{cor:app_query_norm_bound}
Assume $\|\mathbf{w}_{\ell_b}\|\le R_w$. Then
\begin{equation}
\|\mathbf{q}_b\| \le R_w + \|\mathbf{Q}_{\mathrm{cur}}\|R_s+ \lambda_{\mathrm{ctx}}\|\mathbf{Q}_{\mathrm{ctx}}\|R_e+ \|\mathbf{Q}_{\mathrm{dep}}\|R_s
\label{eq:app_query_norm_bound}
\end{equation}
\end{corollary}

\begin{proof}
By~\eqref{eq:app_prequery} and~\eqref{eq:app_final_query},
\[
\mathbf{q}_b = \mathbf{w}_{\ell_b} + \mathbf{Q}_{\mathrm{cur}}\mathbf{s}^{\mathrm{entry}}_{\ell_b} + \lambda_{\mathrm{ctx}}\mathbf{Q}_{\mathrm{ctx}}\mathbf{e}(c) + \mathbf{Q}_{\mathrm{dep}}\mathbf{u}^{\mathrm{att}}_{b-1}
\]
From triangle inequality,
\begin{align*}
\|\mathbf{q}_b\| &\le \|\mathbf{w}_{\ell_b}\| + \|\mathbf{Q}_{\mathrm{cur}}\mathbf{s}^{\mathrm{entry}}_{\ell_b}\| + \lambda_{\mathrm{ctx}}\|\mathbf{Q}_{\mathrm{ctx}}\mathbf{e}(c)\| + \|\mathbf{Q}_{\mathrm{dep}}\mathbf{u}^{\mathrm{att}}_{b-1}\|\\ &\le \|\mathbf{w}_{\ell_b}\| + \|\mathbf{Q}_{\mathrm{cur}}\|\,\|\mathbf{s}^{\mathrm{entry}}_{\ell_b}\| + \lambda_{\mathrm{ctx}}\|\mathbf{Q}_{\mathrm{ctx}}\|\,\|\mathbf{e}(c)\| + \|\mathbf{Q}_{\mathrm{dep}}\|\,\|\mathbf{u}^{\mathrm{att}}_{b-1}\|
\end{align*}
Using $\|\mathbf{w}_{\ell_b}\|\le R_w$, Assumption~\ref{ass:app_bounded_states_instructions}, and Proposition~\ref{prop:app_depth_summary_bounded}  gives~\eqref{eq:app_query_norm_bound}.
\end{proof}

The next two lemmas collect standard differential facts used in the stability analysis. RMSNorm follows the root-mean-square normalization map introduced by ~\cite{zhang2019root}; the Jacobian below is obtained by directly differentiating that map. The softmax Jacobian identity is classical and appears directly, for example, in ~\cite{martins2016softmax} and in analyses of the log-sum-exp Hessian by ~\cite{gao2017properties}.

\begin{lemma}[Exact Jacobian of RMS normalization]
\label{lem:app_rmsnorm_jacobian}
Let $r(\mathbf{x})=\text{RMS}_{\varepsilon}(\mathbf{x})=\sqrt{\frac{1}{d}\|\mathbf{x}\|_2^2+\varepsilon}$, and, $\mathbf{f}(\mathbf{x})=\frac{\mathbf{x}}{r(\mathbf{x})}$.
Then
\begin{equation}
\mathbf{J}_{\mathbf{f}}(\mathbf{x}) = \frac{1}{r(\mathbf{x})}\mathbf{I}_d - \frac{1}{d\,r(\mathbf{x})^3}\mathbf{x}\mathbf{x}^{\top}
\label{eq:app_rmsnorm_jacobian}
\end{equation}
If $r(\mathbf{x})\ge \rho_{\min}>0$, then
\begin{equation}
\|\mathbf{J}_{\mathbf{f}}(\mathbf{x})\|_{2\to 2} \le \frac{1}{\rho_{\min}}
\label{eq:app_rmsnorm_jacobian_bound}
\end{equation}
\end{lemma}

\begin{proof}
We refer to \cite{stollenwerk2026mathematical} for this proof. Write $\mathbf{f}(\mathbf{x})=r(\mathbf{x})^{-1}\mathbf{x}$. By the product rule for Jacobians,
\begin{equation} \label{eq:20}
\mathbf{J}_{\mathbf{f}}(\mathbf{x}) = r(\mathbf{x})^{-1}\mathbf{I}_d + \mathbf{x}\,\nabla(r(\mathbf{x})^{-1})^{\top}
\end{equation}
We compute $\nabla r(\mathbf{x})$. Since, $r(\mathbf{x})=\biggl(\frac{1}{d}\mathbf{x}^{\top}\mathbf{x}+\varepsilon\biggr)^{1/2}$, we have,
\[
\nabla r(\mathbf{x}) = \frac{1}{2}\biggl(\frac{1}{d}\mathbf{x}^{\top}\mathbf{x}+\varepsilon\biggr)^{-1/2}\frac{2}{d}\mathbf{x} = \frac{\mathbf{x}}{d\,r(\mathbf{x})}
\]
From chain rule,
\begin{equation} \label{eq:19}
\nabla(r(\mathbf{x})^{-1})
=-r(\mathbf{x})^{-2}\nabla r(\mathbf{x})
=-\frac{\mathbf{x}}{d\,r(\mathbf{x})^3}
\end{equation}
From \ref{eq:19} and \ref{eq:20}, we get ~\eqref{eq:app_rmsnorm_jacobian}. For the operator norm bound, note that $\mathbf{x}\mathbf{x}^{\top}$ is symmetric rank one. It has eigenvalue $\|\mathbf{x}\|_2^2$ in the direction of $\mathbf{x}$ and eigenvalue zero on the orthogonal complement. Hence $\mathbf{J}_{\mathbf{f}}(\mathbf{x})$ has eigenvalue $\frac{1}{r(\mathbf{x})}$ on every vector orthogonal to $\mathbf{x}$, and eigenvalue
\begin{align*}
\frac{1}{r(\mathbf{x})}-\frac{\|\mathbf{x}\|_2^2}{d\,r(\mathbf{x})^3}
= \frac{d\,r(\mathbf{x})^2-\|\mathbf{x}\|_2^2}{d\,r(\mathbf{x})^3} =
\frac{d\left(\|\mathbf{x}\|_2^2/d+\varepsilon\right)-\|\mathbf{x}\|_2^2}{d\,r(\mathbf{x})^3} = 
\frac{d\varepsilon}{d\,r(\mathbf{x})^3} =
\frac{\varepsilon}{r(\mathbf{x})^3}
\end{align*}
in the direction of $\mathbf{x}$. Since $r(\mathbf{x})^2=\|\mathbf{x}\|_2^2/d+\varepsilon\ge\varepsilon$, we have $\frac{\varepsilon}{r(\mathbf{x})^3}\le \frac{1}{r(\mathbf{x})}$. Hence, the largest absolute eigenvalue is at most $\frac{1}{r(\mathbf{x})}$. Since the Jacobian is symmetric, its spectral norm is the largest absolute eigenvalue.
\[
\implies \|\mathbf{J}_{\mathbf{f}}(\mathbf{x})\|_{2\to 2} \le \frac{1}{r(\mathbf{x})} \le \frac{1}{\rho_{\min}}
\]
\end{proof}

\begin{lemma}[Softmax Jacobian bound]
\label{lem:app_softmax_jacobian}
Let $\text{softmax}:\mathbb{R}^k\to\Delta_k$ be the softmax map and set $\boldsymbol{\alpha}=\text{softmax}(\mathbf{z})$. Then,
\begin{equation}
\mathbf{J}_{\text{softmax}}(\mathbf{z}) =
\text{diag}(\boldsymbol{\alpha})-\boldsymbol{\alpha}\boldsymbol{\alpha}^{\top}
\label{eq:app_softmax_jacobian}
\end{equation}
Moreover,
\begin{equation}
\|\mathbf{J}_{\text{softmax}}(\mathbf{z})\|_{2\to 2}\le \frac{1}{2}
\label{eq:app_softmax_jacobian_bound}
\end{equation}
\end{lemma}

\begin{proof}
Result directly follows from \cite{martins2016softmax}.
\end{proof}

\ref{prop:app_summary_lipschitz} showcases that the attention-style depth summary changes smoothly when the block query or previous block states are perturbed, as long as the active routing structure does not switch abruptly. Small changes in the computation trajectory lead to controlled changes in the retrieved historical context.

\begin{proposition}[Local Lipschitz stability of the attention-style summary]
\label{prop:app_summary_lipschitz}
Fix a block $b\ge 2$ and suppose the completed block means $\bar{\mathbf{s}}_1,\bar{\mathbf{s}}_2, \bar{\mathbf{s}}_3,\ldots,\bar{\mathbf{s}}_{b-1}$ are fixed. On any region satisfying Assumption~\ref{ass:app_rms_region}, the attention-style summary is locally Lipschitz in the instruction embedding. One admissible bound is
\begin{equation}
\left\|\mathbf{u}^{\mathrm{att}}_{b-1}(\mathbf{e})-\mathbf{u}^{\mathrm{att}}_{b-1}(\mathbf{e}')\right\| \le \frac{(b-1)R_s\,\|\mathbf{Q}_{\mathrm{dep},q}\|\lambda_{\mathrm{ctx}}\|\mathbf{Q}_{\mathrm{ctx}}\|}{2T_{\mathrm{dep}}\rho_{\min}}
\,\|\mathbf{e}-\mathbf{e}'\|
\label{eq:app_summary_lipschitz}
\end{equation}
Consequently,
\begin{equation}
\|\mathbf{q}_b(\mathbf{e})-\mathbf{q}_b(\mathbf{e}')\|
\le
L_{q,b}\|\mathbf{e}-\mathbf{e}'\|
\label{eq:app_query_lipschitz}
\end{equation}
where
\begin{equation}
L_{q,b} := \lambda_{\mathrm{ctx}}\|\mathbf{Q}_{\mathrm{ctx}}\| \left( 1 + \frac{(b-1)R_s\,\|\mathbf{Q}_{\mathrm{dep}}\|\|\mathbf{Q}_{\mathrm{dep},q}\|}{2T_{\mathrm{dep}}\rho_{\min}} \right)
\label{eq:app_query_lipschitz_constant}
\end{equation}
\end{proposition}

\begin{proof}
The only part of $\mathbf{q}_b^{(0)}$ that depends on the instruction embedding is $\lambda_{\mathrm{ctx}}\mathbf{Q}_{\mathrm{ctx}}\mathbf{e}$. Hence,
\begin{equation}
\|\mathbf{q}^{(0)}_b(\mathbf{e})-\mathbf{q}^{(0)}_b(\mathbf{e}')\| \le \lambda_{\mathrm{ctx}}\|\mathbf{Q}_{\mathrm{ctx}}\|\,\|\mathbf{e}-\mathbf{e}'\|
\label{eq:app_prequery_lipschitz}
\end{equation}
For each completed block $i<b$, define $\xi_{i\to b}(\mathbf{e})
=
\frac{\langle \widehat{\mathbf{q}}^{(0)}_b(\mathbf{e}),\widehat{\boldsymbol{\kappa}}_i\rangle}{\sqrt{d_k}T_{\mathrm{dep}}}$. $\widehat{\boldsymbol{\kappa}}_i$ is fixed because $\bar{\mathbf{s}}_i$ is fixed. RMS normalization gives $\|\widehat{\boldsymbol{\kappa}}_i\|_2\le\sqrt{d_k}$. Hence, by Cauchy-Schwarz inequality,
\begin{multline*}
|\xi_{i\to b}(\mathbf{e})-\xi_{i\to b}(\mathbf{e}')| \le \frac{1}{\sqrt{d_k}T_{\mathrm{dep}}} \left|\left\langle \widehat{\mathbf{q}}^{(0)}_b(\mathbf{e})-\widehat{\mathbf{q}}^{(0)}_b(\mathbf{e}'), \widehat{\boldsymbol{\kappa}}_i \right\rangle\right| \\ \le \frac{1}{\sqrt{d_k}T_{\mathrm{dep}}} \left\|\widehat{\mathbf{q}}^{(0)}_b(\mathbf{e})-\widehat{\mathbf{q}}^{(0)}_b(\mathbf{e}')\right\|_2 \|\widehat{\boldsymbol{\kappa}}_i\|_2 \le \frac{1}{T_{\mathrm{dep}}} \left\|\widehat{\mathbf{q}}^{(0)}_b(\mathbf{e})-\widehat{\mathbf{q}}^{(0)}_b(\mathbf{e}')\right\|_2
\end{multline*}
By Lemma~\ref{lem:app_rmsnorm_jacobian}, RMS normalization is $\frac{1}{\rho_{\min}}$-Lipschitz on the non-degenerate region. Hence,
\begin{equation*}
\left\|\widehat{\mathbf{q}}^{(0)}_b(\mathbf{e})-\widehat{\mathbf{q}}^{(0)}_b(\mathbf{e}')\right\|_2 \le \frac{1}{\rho_{\min}} \left\|\mathbf{Q}_{\mathrm{dep},q}\bigl(\mathbf{q}^{(0)}_b(\mathbf{e})-\mathbf{q}^{(0)}_b(\mathbf{e}')\bigr)\right\|_2 \le \frac{\|\mathbf{Q}_{\mathrm{dep},q}\|\lambda_{\mathrm{ctx}}\|\mathbf{Q}_{\mathrm{ctx}}\|}{\rho_{\min}} \|\mathbf{e}-\mathbf{e}'\|
\end{equation*}
where the final inequality uses~\eqref{eq:app_prequery_lipschitz}. Thus each depth logit satisfies the inequality,
\[
|\xi_{i\to b}(\mathbf{e})-\xi_{i\to b}(\mathbf{e}')| \le \frac{\|\mathbf{Q}_{\mathrm{dep},q}\|\lambda_{\mathrm{ctx}}\|\mathbf{Q}_{\mathrm{ctx}}\|}{T_{\mathrm{dep}}\rho_{\min}} \|\mathbf{e}-\mathbf{e}'\|
\]
Since there are $b-1$ depth logits,
\begin{equation}
\|\boldsymbol{\xi}_b(\mathbf{e})-\boldsymbol{\xi}_b(\mathbf{e}')\|_2 \le \sqrt{b-1}\, \frac{\|\mathbf{Q}_{\mathrm{dep},q}\|\lambda_{\mathrm{ctx}}\|\mathbf{Q}_{\mathrm{ctx}}\|}{T_{\mathrm{dep}}\rho_{\min}} \|\mathbf{e}-\mathbf{e}'\|
\label{eq:app_depth_logit_vector_bound}
\end{equation}
The depth weights satisfy $\boldsymbol{\beta}_b(\mathbf{e})=\text{softmax}(\boldsymbol{\xi}_b(\mathbf{e}))$. Lemma~\ref{lem:app_softmax_jacobian} gives
\begin{equation}
\|\boldsymbol{\beta}_b(\mathbf{e})-\boldsymbol{\beta}_b(\mathbf{e}')\|_2 \le \frac{1}{2}\|\boldsymbol{\xi}_b(\mathbf{e})-\boldsymbol{\xi}_b(\mathbf{e}')\|_2
\label{eq:app_depth_weight_bound}
\end{equation}
As,
\[
\mathbf{u}^{\mathrm{att}}_{b-1}(\mathbf{e})-\mathbf{u}^{\mathrm{att}}_{b-1}(\mathbf{e}') = \sum_{i=1}^{b-1}(\beta_{i\to b}(\mathbf{e})-\beta_{i\to b}(\mathbf{e}'))\bar{\mathbf{s}}_i
\]
Hence,
\begin{multline} \label{eq:21}
\left\|\mathbf{u}^{\mathrm{att}}_{b-1}(\mathbf{e})-\mathbf{u}^{\mathrm{att}}_{b-1}(\mathbf{e}')\right\| \le \sum_{i=1}^{b-1}|\beta_{i\to b}(\mathbf{e})-\beta_{i\to b}(\mathbf{e}')|\,\|\bar{\mathbf{s}}_i\| \\ \le R_s\|\boldsymbol{\beta}_b(\mathbf{e})-\boldsymbol{\beta}_b(\mathbf{e}')\|_1 \le R_s\sqrt{b-1}\|\boldsymbol{\beta}_b(\mathbf{e})-\boldsymbol{\beta}_b(\mathbf{e}')\|_2
\end{multline}
Combining ~\eqref{eq:21}, ~\eqref{eq:app_depth_logit_vector_bound} and~\eqref{eq:app_depth_weight_bound} proves~\eqref{eq:app_summary_lipschitz}. Finally,
\[
\mathbf{q}_b(\mathbf{e})-\mathbf{q}_b(\mathbf{e}') = \mathbf{q}^{(0)}_b(\mathbf{e})-\mathbf{q}^{(0)}_b(\mathbf{e}') + \mathbf{Q}_{\mathrm{dep}}\left(\mathbf{u}^{\mathrm{att}}_{b-1}(\mathbf{e})-\mathbf{u}^{\mathrm{att}}_{b-1}(\mathbf{e}')\right)
\]
Using~\eqref{eq:app_prequery_lipschitz}, and ~\eqref{eq:app_summary_lipschitz} yields,
\begin{align*}
\|\mathbf{q}_b(\mathbf{e})-\mathbf{q}_b(\mathbf{e}')\| &\le \lambda_{\mathrm{ctx}}\|\mathbf{Q}_{\mathrm{ctx}}\|\|\mathbf{e}-\mathbf{e}'\| + \|\mathbf{Q}_{\mathrm{dep}}\| \left\|\mathbf{u}^{\mathrm{att}}_{b-1}(\mathbf{e})-\mathbf{u}^{\mathrm{att}}_{b-1}(\mathbf{e}')\right\|\\ &\le \lambda_{\mathrm{ctx}}\|\mathbf{Q}_{\mathrm{ctx}}\| \left( 1 + \frac{(b-1)R_s\,\|\mathbf{Q}_{\mathrm{dep}}\|\|\mathbf{Q}_{\mathrm{dep},q}\|}{2T_{\mathrm{dep}}\rho_{\min}}
\right)
\|\mathbf{e}-\mathbf{e}'\|
\end{align*}
\end{proof}

\ref{prop:app_router_operator_lipschitz} extends the same stability idea from the depth summary to the router itself. It demonstrates that, away from top-$k$ switching boundaries, small changes in the state query or instruction prior lead to controlled changes in both the routing weights and the final retrieved operator, hence, the dynamic adapter behaves locally smoothly.

\begin{proposition}[Local Lipschitz stability of the router and routed operator]
\label{prop:app_router_operator_lipschitz}
Let $I$ be a fixed active set of size $k$. Under Assumptions~\ref{ass:app_bounded_atoms_factors}-\ref{ass:app_fixed_active_set},
\begin{equation}
\|\boldsymbol{\alpha}_{b,I}(\mathbf{e})-\boldsymbol{\alpha}_{b,I}(\mathbf{e}')\|_2 \le L_{\alpha,b}\|\mathbf{e}-\mathbf{e}'\|,
\label{eq:app_router_lipschitz}
\end{equation}
where one admissible constant is
\begin{equation}
L_{\alpha,b} := \frac{\sqrt{k}}{2\rho_{\min}} \left( \frac{L_{q,b}}{T_{\mathrm{attn}}} + \frac{\tau_{\mathrm{lang}}\|\mathbf{R}_{\mathrm{ctx}}\|}{T_{\mathrm{lang}}} \right)
\label{eq:app_router_lipschitz_constant}
\end{equation}
Moreover,
\begin{equation}
\|\mathbf{S}_b(\mathbf{e})-\mathbf{S}_b(\mathbf{e}')\| \le R_C\sqrt{k}\,L_{\alpha,b}\|\mathbf{e}-\mathbf{e}'\|
\label{eq:app_routed_operator_lipschitz}
\end{equation}
and $\forall \ \ell\in\mathcal{B}_b$ and $\forall$ input vector $ \mathbf{x}$,
\begin{equation}
\|\Delta\mathbf{W}^{\ell}(\mathbf{e})\mathbf{x}-\Delta\mathbf{W}^{\ell}(\mathbf{e}')\mathbf{x}\| \le \frac{\alpha_{\mathrm{L}}}{r}\|\mathbf{B}_{\ell}\|\,|g_{\ell}|\,\|\mathbf{A}_{\ell}\|R_C\sqrt{k}\,L_{\alpha,b}\|\mathbf{e}-\mathbf{e}'\|\,\|\mathbf{x}\|
\label{eq:app_layer_update_lipschitz}
\end{equation}
\end{proposition}

\begin{proof}
From Lemma~\ref{lem:app_equivalent_raw_logit_fusion}, the router on a fixed active set can be written as $\boldsymbol{\alpha}_{b,I}(\mathbf{e})=\text{softmax}(\mathbf{z}_{b,I}(\mathbf{e}))$ and $z_{b,m}(\mathbf{e})=\zeta_{b,m}(\mathbf{e})+\tau_{\mathrm{lang}}\rho_m(\mathbf{e})$. Since, $\zeta_{b,m}(\mathbf{e}) = \frac{\langle \text{RMSNorm}(\mathbf{q}_b(\mathbf{e})),\widehat{\mathbf{k}}_m\rangle}{\sqrt{d_k}T_{\mathrm{attn}}}$, Cauchy-Schwarz inequality, and $\|\widehat{\mathbf{k}}_m\|_2\le\sqrt{d_k}$ imply,
\begin{multline*}
|\zeta_{b,m}(\mathbf{e})-\zeta_{b,m}(\mathbf{e}')| \le \frac{1}{\sqrt{d_k}T_{\mathrm{attn}}} \left|\left\langle \text{RMSNorm}(\mathbf{q}_b(\mathbf{e}))-\text{RMSNorm}(\mathbf{q}_b(\mathbf{e}')), \widehat{\mathbf{k}}_m \right\rangle\right|\\  \le \frac{1}{T_{\mathrm{attn}}} \|\text{RMSNorm}(\mathbf{q}_b(\mathbf{e}))-\text{RMSNorm}(\mathbf{q}_b(\mathbf{e}'))\|_2
\end{multline*}
From Lemma~\ref{lem:app_rmsnorm_jacobian} and Proposition~\ref{prop:app_summary_lipschitz},
\[
\|\text{RMSNorm}(\mathbf{q}_b(\mathbf{e}))-\text{RMSNorm}(\mathbf{q}_b(\mathbf{e}'))\|_2 \le \frac{1}{\rho_{\min}}\|\mathbf{q}_b(\mathbf{e})-\mathbf{q}_b(\mathbf{e}')\|_2 \le \frac{L_{q,b}}{\rho_{\min}}\|\mathbf{e}-\mathbf{e}'\|.
\]
Hence,
\begin{equation}
|\zeta_{b,m}(\mathbf{e})-\zeta_{b,m}(\mathbf{e}')| \le \frac{L_{q,b}}{T_{\mathrm{attn}}\rho_{\min}}\|\mathbf{e}-\mathbf{e}'\|
\label{eq:app_state_logit_lipschitz}
\end{equation}
For the language logit,
\[
\rho_m(\mathbf{e}) = \frac{\langle \text{RMSNorm}(\mathbf{R}_{\mathrm{ctx}}\mathbf{e}),\widehat{\mathbf{k}}_m\rangle}{\sqrt{d_k}T_{\mathrm{lang}}}
\]
\begin{equation}
\implies |\rho_m(\mathbf{e})-\rho_m(\mathbf{e}')| \le \frac{\|\mathbf{R}_{\mathrm{ctx}}\|}{T_{\mathrm{lang}}\rho_{\min}}\|\mathbf{e}-\mathbf{e}'\|
\label{eq:app_language_logit_lipschitz}
\end{equation}
Combining~\eqref{eq:app_state_logit_lipschitz} and~\eqref{eq:app_language_logit_lipschitz},  $\forall \ m\in I$,
\[
|z_{b,m}(\mathbf{e})-z_{b,m}(\mathbf{e}')| \le \frac{1}{\rho_{\min}} \left( \frac{L_{q,b}}{T_{\mathrm{attn}}} + \frac{\tau_{\mathrm{lang}}\|\mathbf{R}_{\mathrm{ctx}}\|}{T_{\mathrm{lang}}} \right) \|\mathbf{e}-\mathbf{e}'\|
\]
Since there are $k$ active logits,
\[
\|\mathbf{z}_{b,I}(\mathbf{e})-\mathbf{z}_{b,I}(\mathbf{e}')\|_2 \le \frac{\sqrt{k}}{\rho_{\min}} \left( \frac{L_{q,b}}{T_{\mathrm{attn}}} + \frac{\tau_{\mathrm{lang}}\|\mathbf{R}_{\mathrm{ctx}}\|}{T_{\mathrm{lang}}} \right) \|\mathbf{e}-\mathbf{e}'\|
\]
By Lemma~\ref{lem:app_softmax_jacobian}, softmax is $\frac{1}{2}$-Lipschitz in Euclidean norm. Hence,
\[
\|\boldsymbol{\alpha}_{b,I}(\mathbf{e})-\boldsymbol{\alpha}_{b,I}(\mathbf{e}')\|_2 \le \frac{1}{2} \|\mathbf{z}_{b,I}(\mathbf{e})-\mathbf{z}_{b,I}(\mathbf{e}')\|_2
\]
which proves~\eqref{eq:app_router_lipschitz} with the constant in~\eqref{eq:app_router_lipschitz_constant}. For the routed operator, since the active set is fixed,
\[
\mathbf{S}_b(\mathbf{e})-\mathbf{S}_b(\mathbf{e}') = \sum_{m\in I}(\alpha_{b,m}(\mathbf{e})-\alpha_{b,m}(\mathbf{e}'))\mathbf{C}_m
\]
\begin{align} \label{eq:22}
\implies \|\mathbf{S}_b(\mathbf{e})-\mathbf{S}_b(\mathbf{e}')\| &\le \sum_{m\in I}|\alpha_{b,m}(\mathbf{e})-\alpha_{b,m}(\mathbf{e}')|\,\|\mathbf{C}_m\| \le R_C\|\boldsymbol{\alpha}_{b,I}(\mathbf{e})-\boldsymbol{\alpha}_{b,I}(\mathbf{e}')\|_1\\
&\le R_C\sqrt{k}\|\boldsymbol{\alpha}_{b,I}(\mathbf{e})-\boldsymbol{\alpha}_{b,I}(\mathbf{e}')\|_2
\end{align}
Combining ~\eqref{eq:22} and ~\eqref{eq:app_router_lipschitz} proves~\eqref{eq:app_routed_operator_lipschitz}. Since,
\[
\Delta\mathbf{W}^{\ell}(\mathbf{e})-\Delta\mathbf{W}^{\ell}(\mathbf{e}') = \frac{\alpha_{\mathrm{L}}}{r}\mathbf{B}_{\ell}g_{\ell} \bigl(\mathbf{S}_b(\mathbf{e})-\mathbf{S}_b(\mathbf{e}')\bigr) \mathbf{A}_{\ell}
\]
By submultiplicativity,
\[
\|\Delta\mathbf{W}^{\ell}(\mathbf{e})-\Delta\mathbf{W}^{\ell}(\mathbf{e}')\| \le \frac{\alpha_{\mathrm{L}}}{r}\|\mathbf{B}_{\ell}\|\,|g_{\ell}|\, \|\mathbf{S}_b(\mathbf{e})-\mathbf{S}_b(\mathbf{e}')\|\,\|\mathbf{A}_{\ell}\|
\]
Using~\eqref{eq:app_routed_operator_lipschitz} and then applying the resulting operator to $\mathbf{x}$ proves~\eqref{eq:app_layer_update_lipschitz}.
\end{proof}
\ref{prop:app_compressed_gradient_identity} connects the full adapter update to the lower-dimensional rank-space object that the router controls and shows that the first-order effect of a routed update is captured by a compressed gradient. This leads to learning in the shared atom space being aligned with the directions of the original layer on the loss landscape.

\begin{proposition}[Exact compressed-gradient identity]
\label{prop:app_compressed_gradient_identity}
Fix an adapted layer $\ell\in\mathcal{B}_b$. Let $\mathbf{G}_{\ell}$ be the local dense gradient with respect to the effective weight update at layer $\ell$, to ensure that the first-order loss variation induced by a dense update $\mathbf{U}$ is $\langle \mathbf{G}_{\ell},\mathbf{U}\rangle_F$. Then the dynamic routed part satisfies:
\begin{equation}
\left\langle \mathbf{G}_{\ell}, \frac{\alpha_{\mathrm{L}}}{r}g_{\ell}\mathbf{B}_{\ell}\mathbf{S}_b(c)\mathbf{A}_{\ell} \right\rangle_F = \frac{\alpha_{\mathrm{L}}}{r}g_{\ell} \left\langle \mathbf{B}_{\ell}^{\top}\mathbf{G}_{\ell}\mathbf{A}_{\ell}^{\top}, \mathbf{S}_b(c) \right\rangle_F
\label{eq:app_compressed_gradient_identity}
\end{equation}
Hence, the first-order optimization signal for the routed operator is the compressed rank-space gradient $\mathbf{B}_{\ell}^{\top}\mathbf{G}_{\ell}\mathbf{A}_{\ell}^{\top}$.
\end{proposition}

\begin{proof}
Via definition of the Frobenius inner product, $\left\langle \mathbf{G}_{\ell}, \mathbf{B}_{\ell}\mathbf{S}_b(c)\mathbf{A}_{\ell} \right\rangle_F = \text{tr}\left(\mathbf{G}_{\ell}^{\top}\mathbf{B}_{\ell}\mathbf{S}_b(c)\mathbf{A}_{\ell}\right)$. From cyclic invariance of the trace, $\text{tr}\left(\mathbf{G}_{\ell}^{\top}\mathbf{B}_{\ell}\mathbf{S}_b(c)\mathbf{A}_{\ell}\right) = \text{tr}\left(\mathbf{A}_{\ell}\mathbf{G}_{\ell}^{\top}\mathbf{B}_{\ell}\mathbf{S}_b(c)\right)$. Observe that,
\begin{align*}
\text{tr}\left(\mathbf{A}_{\ell}\mathbf{G}_{\ell}^{\top}\mathbf{B}_{\ell}\mathbf{S}_b(c)\right) = \text{tr}\left(\left(\mathbf{B}_{\ell}^{\top}\mathbf{G}_{\ell}\mathbf{A}_{\ell}^{\top}\right)^{\top}\mathbf{S}_b(c)\right) &= \left\langle \mathbf{B}_{\ell}^{\top}\mathbf{G}_{\ell}\mathbf{A}_{\ell}^{\top},\mathbf{S}_b(c)\right\rangle_F \\ \implies \left\langle \mathbf{G}_{\ell}, \frac{\alpha_{\mathrm{L}}}{r}g_{\ell}\mathbf{B}_{\ell}\mathbf{S}_b(c)\mathbf{A}_{\ell} \right\rangle_F &= \frac{\alpha_{\mathrm{L}}}{r}g_{\ell} \left\langle \mathbf{B}_{\ell}^{\top}\mathbf{G}_{\ell}\mathbf{A}_{\ell}^{\top}, \mathbf{S}_b(c) \right\rangle_F
\end{align*}
\end{proof}

\ref{thm:app_attention_projection_norm_stability} lifts the norm-control argument to the Transformer attention projections. When the queryable operator is inserted into the $Q, K, V$ and $O$ projection adapters, the resulting projection updates remain uniformly bounded since they are built from a bounded atom memory. The instruction-conditioned router can make the attention adaptation bounded and dynamic, and specific to the projections.

\begin{theorem}[Norm bound and instruction stability for attention-projection updates]
\label{thm:app_attention_projection_norm_stability}
Consider a Transformer attention block $\ell\in\mathcal{B}_b$ with projection-specific updates
\begin{equation}
\Delta\mathbf{W}^{\ell,p}(\mathbf{h}_{\ell};b,c) = \frac{\alpha_{\mathrm{L}}}{r}\, \mathbf{B}^{p}_{\ell} \bigl(\mathbf{I}_r+g^{p}_{\ell}\mathbf{S}_b(c)\bigr) \mathbf{A}^{p}_{\ell}
\label{eq:app_attention_projection_update}
\end{equation}
where, $p\in\{Q,K,V,O\}$. Assume $\|\mathbf{A}^{p}_{\ell}\|\le R_A^p$ and $\|\mathbf{B}^{p}_{\ell}\|\le R_B^p$. Then, $\forall \ p\in\{Q,K,V,O\}$,
\begin{equation}
\|\Delta\mathbf{W}^{\ell,p}(\mathbf{h}_{\ell};b,c)\| \le \frac{\alpha_{\mathrm{L}}}{r}R_B^p\bigl(1+|g^{p}_{\ell}|R_C\bigr)R_A^p
\label{eq:app_attention_projection_norm_bound}
\end{equation}
Moreover, on a fixed active-set region,
\begin{equation}
\|\Delta\mathbf{W}^{\ell,p}(\mathbf{e})-\Delta\mathbf{W}^{\ell,p}(\mathbf{e}')\| \le \frac{\alpha_{\mathrm{L}}}{r}R_B^p|g^{p}_{\ell}|R_A^p R_C\sqrt{k}\,L_{\alpha,b}\|\mathbf{e}-\mathbf{e}'\|
\label{eq:app_attention_projection_stability_bound}
\end{equation}
\end{theorem}

\begin{proof}
Equation~\eqref{eq:app_attention_projection_norm_bound} is Corollary~\ref{cor:app_update_norm_bound} applied to $(\mathbf{A}^{p}_{\ell},\mathbf{B}^{p}_{\ell})$. For the stability bound,
\[
\Delta\mathbf{W}^{\ell,p}(\mathbf{e})-\Delta\mathbf{W}^{\ell,p}(\mathbf{e}') = \frac{\alpha_{\mathrm{L}}}{r}\mathbf{B}^{p}_{\ell}g^{p}_{\ell} \bigl(\mathbf{S}_b(\mathbf{e})-\mathbf{S}_b(\mathbf{e}')\bigr) \mathbf{A}^{p}_{\ell}
\]
Using submultiplicativity gives,
\begin{equation} \label{eq:24}
\|\Delta\mathbf{W}^{\ell,p}(\mathbf{e})-\Delta\mathbf{W}^{\ell,p}(\mathbf{e}')\| \le \frac{\alpha_{\mathrm{L}}}{r}R_B^p|g^{p}_{\ell}|R_A^p \|\mathbf{S}_b(\mathbf{e})-\mathbf{S}_b(\mathbf{e}')\|
\end{equation}
From proposition~\ref{prop:app_router_operator_lipschitz},
\begin{equation} \label{eq:23}
\|\mathbf{S}_b(\mathbf{e})-\mathbf{S}_b(\mathbf{e}')\| \le R_C\sqrt{k}\,L_{\alpha,b}\|\mathbf{e}-\mathbf{e}'\|
\end{equation}
From \ref{eq:23} and \ref{eq:24}, we get ~\eqref{eq:app_attention_projection_stability_bound}.
\end{proof}

This explains the mechanism of the shared atom memory receiving the learning signal from all layers in a block. As the same routed operator is reused across the block, the gradient with respect to that operator accumulates the low-rank contributions from each adapted layer. Hence, each atom is updated in proportion to its routing weight and its alignment with this blockwise gradient signal. This makes the global memory trainable through structured supervision across depth. For each layer $\ell\in\mathcal{B}_b$, define the rank-space forward quantities:
\begin{align}
\mathbf{s}_{\ell}&:=\mathbf{A}_{\ell}\mathbf{x}_{\ell} \\
\mathbf{d}_{\ell}&:=\mathbf{S}_b(c)\mathbf{s}_{\ell} \\
\mathbf{t}_{\ell}&:=\mathbf{s}_{\ell}+g_{\ell}\mathbf{d}_{\ell} \\
\delta\mathbf{h}_{\ell}&:=\frac{\alpha_{\mathrm{L}}}{r}\mathbf{B}_{\ell}\mathbf{t}_{\ell}
\label{eq:app_rank_forward_quantities}
\end{align}
where, $\mathbf{x}_{\ell}$ denotes the input vector to the adapted projection. Let $\mathcal{L}$ be the scalar training loss. Define:
\begin{equation}
\mathbf{r}_{\ell} := \frac{\partial\mathcal{L}}{\partial \mathbf{t}_{\ell}} = \frac{\alpha_{\mathrm{L}}}{r}\mathbf{B}_{\ell}^{\top} \frac{\partial\mathcal{L}}{\partial\delta\mathbf{h}_{\ell}} \in\mathbb{R}^r
\label{eq:app_rank_backprop_vector}
\end{equation}

\begin{theorem}[Exact blockwise gradient factorization]
\label{thm:app_blockwise_gradient_factorization}
For the shared routed operator reused across block $\mathcal{B}_b$,
\begin{equation}
\nabla_{\mathbf{S}_b(c)}\mathcal{L} = \sum_{\ell\in\mathcal{B}_b}g_{\ell}\mathbf{r}_{\ell}\mathbf{s}_{\ell}^{\top}
\label{eq:app_grad_S_exact}
\end{equation}
Consequently,
\begin{equation}
\nabla_{\mathbf{C}_m}\mathcal{L} = \alpha_{b,m}(c)\nabla_{\mathbf{S}_b(c)}\mathcal{L}
\label{eq:app_grad_atom_exact}
\end{equation}
and
\begin{equation}
\frac{\partial\mathcal{L}}{\partial\alpha_{b,m}(c)} = \left\langle \nabla_{\mathbf{S}_b(c)}\mathcal{L},\mathbf{C}_m\right\rangle_F
\label{eq:app_grad_alpha_exact}
\end{equation}
$m = 1,2,3,4\ldots, M$. Moreover, the gate gradient is,
\begin{equation}
\frac{\partial\mathcal{L}}{\partial\eta_{\ell}} = \sigma(\eta_{\ell})(1-\sigma(\eta_{\ell}))\langle \mathbf{r}_{\ell},\mathbf{d}_{\ell}\rangle
\label{eq:app_grad_gate_exact}
\end{equation}
for $\ell\in\mathcal{B}_b$.
\end{theorem}

\begin{proof}
For a fixed layer $\ell$, the only quantity in~\eqref{eq:app_rank_forward_quantities} that depends directly on $\mathbf{S}_b(c)$ is $\mathbf{d}_{\ell}=\mathbf{S}_b(c)\mathbf{s}_{\ell}$. Hence, if $\mathrm{d}\mathbf{S}_b$ is a perturbation of $\mathbf{S}_b(c)$, then $\mathrm{d}\mathbf{d}_{\ell}=(\mathrm{d}\mathbf{S}_b)\mathbf{s}_{\ell}$. Since, $\mathbf{t}_{\ell}=\mathbf{s}_{\ell}+g_{\ell}\mathbf{d}_{\ell}$, we have: $\mathrm{d}\mathbf{t}_{\ell} =g_{\ell}\mathrm{d}\mathbf{d}_{\ell} =g_{\ell}(\mathrm{d}\mathbf{S}_b)\mathbf{s}_{\ell}$. By definition of $\mathbf{r}_{\ell}$, $\mathrm{d}\mathcal{L}_{\ell} = \langle \mathbf{r}_{\ell},\mathrm{d}\mathbf{t}_{\ell}\rangle = \left\langle \mathbf{r}_{\ell},g_{\ell}(\mathrm{d}\mathbf{S}_b)\mathbf{s}_{\ell}\right\rangle$. As $g_{\ell}$ is scalar, $\mathrm{d}\mathcal{L}_{\ell} =g_{\ell}\mathbf{r}_{\ell}^{\top}(\mathrm{d}\mathbf{S}_b)\mathbf{s}_{\ell}$. Since,
\[
\left\langle \mathbf{r}_{\ell}\mathbf{s}_{\ell}^{\top},\mathrm{d}\mathbf{S}_b\right\rangle_F =\text{tr}\left((\mathbf{r}_{\ell}\mathbf{s}_{\ell}^{\top})^{\top}\mathrm{d}\mathbf{S}_b\right) =\text{tr}\left(\mathbf{s}_{\ell}\mathbf{r}_{\ell}^{\top}\mathrm{d}\mathbf{S}_b\right) =\mathbf{r}_{\ell}^{\top}(\mathrm{d}\mathbf{S}_b)\mathbf{s}_{\ell}
\]
\[
\implies \mathrm{d}\mathcal{L}_{\ell} =\left\langle g_{\ell}\mathbf{r}_{\ell}\mathbf{s}_{\ell}^{\top},\mathrm{d}\mathbf{S}_b\right\rangle_F
\]
\begin{align} \label{eq:26}
\implies \mathrm{d}\mathcal{L} &=\sum_{\ell\in\mathcal{B}_b}\mathrm{d}\mathcal{L}_{\ell} =\sum_{\ell\in\mathcal{B}_b} \left\langle g_{\ell}\mathbf{r}_{\ell}\mathbf{s}_{\ell}^{\top},\mathrm{d}\mathbf{S}_b\right\rangle_F =\left\langle \sum_{\ell\in\mathcal{B}_b}g_{\ell}\mathbf{r}_{\ell}\mathbf{s}_{\ell}^{\top}, \mathrm{d}\mathbf{S}_b \right\rangle_F
\end{align}
From the definition of the Frobenius gradient, this proves~\eqref{eq:app_grad_S_exact}. Next, use the linear decomposition $\mathbf{S}_b(c)=\sum_{m=1}^{M}\alpha_{b,m}(c)\mathbf{C}_m$. If the routing weights are held fixed and the atom $\mathbf{C}_m$ is perturbed, then,
\begin{equation} \label{eq:25}
\mathrm{d}\mathbf{S}_b=
\sum_{m=1}^{M}\alpha_{b,m}(c)\mathrm{d}\mathbf{C}_m
\end{equation}
Substituting \ref{eq:25} in \ref{eq:26},
\begin{align*}
\mathrm{d}\mathcal{L} =\left\langle \nabla_{\mathbf{S}_b(c)}\mathcal{L},\mathrm{d}\mathbf{S}_b\right\rangle_F =\left\langle \nabla_{\mathbf{S}_b(c)}\mathcal{L}, \sum_{m=1}^{M}\alpha_{b,m}(c)\mathrm{d}\mathbf{C}_m \right\rangle_F =\sum_{m=1}^{M}\alpha_{b,m}(c) \left\langle \nabla_{\mathbf{S}_b(c)}\mathcal{L},\mathrm{d}\mathbf{C}_m\right\rangle_F
\end{align*}

\[
\implies \nabla_{\mathbf{C}_m}\mathcal{L}=\alpha_{b,m}(c)\nabla_{\mathbf{S}_b(c)}\mathcal{L}
\]
which proves~\eqref{eq:app_grad_atom_exact}. If, instead, the atoms are held fixed and the routing weights are perturbed, then
\[
\mathrm{d}\mathbf{S}_b=
\sum_{m=1}^{M}(\mathrm{d}\alpha_{b,m}(c))\mathbf{C}_m
\]
\begin{align*}
\implies \mathrm{d}\mathcal{L} &=\left\langle \nabla_{\mathbf{S}_b(c)}\mathcal{L},\mathrm{d}\mathbf{S}_b\right\rangle_F =\sum_{m=1}^{M}(\mathrm{d}\alpha_{b,m}(c)) \left\langle \nabla_{\mathbf{S}_b(c)}\mathcal{L},\mathbf{C}_m\right\rangle_F
\end{align*}
The coefficient of $\mathrm{d}\alpha_{b,m}(c)$ is the partial derivative with respect to $\alpha_{b,m}(c)$, proving~\eqref{eq:app_grad_alpha_exact}. Finally, as, $g_{\ell}=\sigma(\eta_{\ell})$ and $\mathbf{t}_{\ell}=\mathbf{s}_{\ell}+g_{\ell}\mathbf{d}_{\ell}$,
\[
\frac{\partial \mathbf{t}_{\ell}}{\partial \eta_{\ell}} =\sigma(\eta_{\ell})(1-\sigma(\eta_{\ell}))\mathbf{d}_{\ell}
\]
Using the chain rule, we get:
\[
\frac{\partial\mathcal{L}}{\partial\eta_{\ell}} = \left\langle \frac{\partial\mathcal{L}}{\partial\mathbf{t}_{\ell}}, \frac{\partial \mathbf{t}_{\ell}}{\partial \eta_{\ell}} \right\rangle = \sigma(\eta_{\ell})(1-\sigma(\eta_{\ell}))\langle \mathbf{r}_{\ell},\mathbf{d}_{\ell}\rangle
\]
This proves~\eqref{eq:app_grad_gate_exact}.
\end{proof}

A direct consequence of the blockwise gradient factorization is that the router logits receive an interpretable training signal. \ref{cor:app_exact_logit_gradients} presents the promotion of each atom relative to the currently routed average.

\begin{corollary}[Exact logit gradients on a fixed active set]
\label{cor:app_exact_logit_gradients}
Let $I$ be a fixed active set. Define $\psi_{b,m}:=\left\langle \nabla_{\mathbf{S}_b(c)}\mathcal{L},\mathbf{C}_m\right\rangle_F$ and $\bar\psi_b:=\sum_{j\in I}\alpha_{b,j}(c)\psi_{b,j}$. If $\boldsymbol{\alpha}_{b,I}(c)=\text{softmax}(\mathbf{z}_{b,I}(c))$, then $\forall m\in I$,
\begin{align}
\frac{\partial\mathcal{L}}{\partial z_{b,m}(c)} &= \alpha_{b,m}(c)(\psi_{b,m}-\bar\psi_b) \label{eq:app_joint_logit_gradient}\\
\frac{\partial\mathcal{L}}{\partial \zeta_{b,m}} &= \alpha_{b,m}(c)(\psi_{b,m}-\bar\psi_b)
\label{eq:app_state_logit_gradient}\\
\frac{\partial\mathcal{L}}{\partial \rho_m(c)} &= \tau_{\mathrm{lang}}\alpha_{b,m}(c)(\psi_{b,m}-\bar\psi_b)
\label{eq:app_language_logit_gradient}
\end{align}
\end{corollary}

\begin{proof}
From ~\ref{thm:app_blockwise_gradient_factorization}, $\frac{\partial\mathcal{L}}{\partial \alpha_{b,j}(c)} = \psi_{b,j}$. On the fixed active set, $\alpha_{b,j}(c)=\frac{\exp(z_{b,j}(c))}{\sum_{q\in I}\exp(z_{b,q}(c))}$. From ~\ref{lem:app_softmax_jacobian},
\[
\frac{\partial\alpha_{b,j}(c)}{\partial z_{b,m}(c)} = \alpha_{b,j}(c)(\delta_{jm}-\alpha_{b,m}(c))
\]
Using the chain rule,
\begin{multline*}
\frac{\partial\mathcal{L}}{\partial z_{b,m}(c)} = \sum_{j\in I} \frac{\partial\mathcal{L}}{\partial\alpha_{b,j}(c)} \frac{\partial\alpha_{b,j}(c)}{\partial z_{b,m}(c)} = \sum_{j\in I}\psi_{b,j}\alpha_{b,j}(c)(\delta_{jm}-\alpha_{b,m}(c))\\ = \psi_{b,m}\alpha_{b,m}(c) - \alpha_{b,m}(c)\sum_{j\in I}\psi_{b,j}\alpha_{b,j}(c) = \alpha_{b,m}(c)(\psi_{b,m}-\bar\psi_b)
\end{multline*}
Hence, this proves~\eqref{eq:app_joint_logit_gradient}. From Lemma~\ref{lem:app_equivalent_raw_logit_fusion},
\[
z_{b,m}(c)=\zeta_{b,m}+\tau_{\mathrm{lang}}\rho_m(c)
\]
\[
\implies \frac{\partial z_{b,m}(c)}{\partial \zeta_{b,m}}=1 \ \text{and} \ \frac{\partial z_{b,m}(c)}{\partial \rho_m(c)}=\tau_{\mathrm{lang}}
\]
Using the chain rule gives~\eqref{eq:app_state_logit_gradient} and~\eqref{eq:app_language_logit_gradient}.
\end{proof}

%% file: Chapters_Appendix/7x_other.tex
\subsection{Compute Resources}\label{s:Compute}

Experiments were conducted on a heterogeneous computing landscape comprising resources from the Texas Advanced Computing Center (TACC), Lambda, Inc., RunPod, and local NVIDIA RTX 4090 GPUs. The total runtime across this setup was approximately 720 A100-equivalent hours.

\subsection{Dataset Details}\label{s:Licenses}

\paragraph{Synthetic regression functions.}
The noisy non-convex regression tasks use analytic benchmark functions to Surjanovic and Bingham's \emph{Virtual Library of Simulation Experiments: Test Functions and Datasets}. In the paper, we generate synthetic samples ourselves from these analytic functions; we do not redistribute a third-party raw dataset from this source. See also: 
\begin{itemize}
    \item[] \url{https://www.sfu.ca/~ssurjano/optimization.html}
\end{itemize}

\paragraph{Pretrained model checkpoints.}
The LLM experiments use pretrained checkpoints as frozen backbones with parameter-efficient adapters. We do not redistribute the original model weights. Any released adapter should preserve the model-card citations and should be used
only with the corresponding upstream model license.
\begin{itemize}[leftmargin=*]
    \item \textbf{Qwen0.5B.} Version used: \texttt{Qwen/Qwen2.5-0.5B} base causal language model. Original owner/creator: Qwen Team. License: Apache License 2.0. See also: model card, \url{https://huggingface.co/Qwen/Qwen2.5-0.5B}; license text, \url{https://www.apache.org/licenses/LICENSE-2.0}; technical report, \url{https://arxiv.org/abs/2407.10671}.

    \item \textbf{Mistral7B.} Version used: \texttt{mistralai/Mistral-7B-v0.1}, original owner/creator: Mistral AI. License: Apache License 2.0. See also: model card, \url{https://huggingface.co/mistralai/Mistral-7B-v0.1}; license text, \url{https://www.apache.org/licenses/LICENSE-2.0}; paper, \url{https://arxiv.org/abs/2310.06825}.
        \item \textbf{SmolLM2-360M-Instruct.} Version used: \texttt{HuggingFaceTB/SmolLM2-360M-Instruct}. Original owner/creator: Hugging Face Smol Models Research. License: Apache License 2.0. See also: model card, \url{https://huggingface.co/HuggingFaceTB/SmolLM2-360M-Instruct}; license text, \url{https://www.apache.org/licenses/LICENSE-2.0}; paper, \url{https://arxiv.org/abs/2502.02737}.

    \item \textbf{LFM2-700M.} Version used: \texttt{LiquidAI/LFM2-700M}. Original owner/creator: Liquid AI, Inc. License: LFM Open License v1.0, identified on Hugging Face as \texttt{lfm1.0}; this is a custom license with commercial-use limitations. See also: model card, \url{https://huggingface.co/LiquidAI/LFM2-700M}; license text, \url{https://huggingface.co/LiquidAI/LFM2-700M/blob/main/LICENSE}; technical report, \url{https://arxiv.org/abs/2511.23404}.

    \item \textbf{LFM2.5-350M.} Version used: \texttt{LiquidAI/LFM2.5-350M}. Original owner/creator: Liquid AI, Inc. License: LFM Open License v1.0, identified on Hugging Face as \texttt{lfm1.0}; this is a custom license with commercial-use limitations. See also: model card, \url{https://huggingface.co/LiquidAI/LFM2.5-350M}; license text, \url{https://huggingface.co/LiquidAI/LFM2.5-350M/blob/main/LICENSE}; technical report, \url{https://arxiv.org/abs/2511.23404}.

    \item \textbf{Qwen2.5-0.5B-Instruct.} Version used: \texttt{Qwen/Qwen2.5-0.5B-Instruct}. Original owner/creator: Qwen Team. License: Apache License 2.0. See also: model card, \url{https://huggingface.co/Qwen/Qwen2.5-0.5B-Instruct}; license text, \url{https://www.apache.org/licenses/LICENSE-2.0}; technical report, \url{https://arxiv.org/abs/2407.10671}.

    \item \textbf{Qwen2.5-Coder-0.5B-Instruct.} Version used: \texttt{Qwen/Qwen2.5-Coder-0.5B-Instruct}. Original owner/creator: Qwen Team. License: Apache License 2.0. See also: model card, \url{https://huggingface.co/Qwen/Qwen2.5-Coder-0.5B-Instruct}; license text, \url{https://www.apache.org/licenses/LICENSE-2.0}; Qwen2.5-Coder technical report, \url{https://arxiv.org/abs/2409.12186}; Qwen2.5 technical report, \url{https://arxiv.org/abs/2407.10671}.

    \item \textbf{Qwen3-0.6B.} Version used: \texttt{Qwen/Qwen3-0.6B}. Original owner/creator: Qwen Team. License: Apache License 2.0. See also: model card, \url{https://huggingface.co/Qwen/Qwen3-0.6B}; license text, \url{https://www.apache.org/licenses/LICENSE-2.0}; technical report, \url{https://arxiv.org/abs/2505.09388}.

    \item \textbf{ReasonLite-0.6B.} Version used: \texttt{amd/ReasonLite-0.6B}. Original owner/creator: Advanced Micro Devices, Inc. (AMD) and the ReasonLite project contributors. License: ReasonLite Open RAIL-D / OpenRAIL, with responsible-use restrictions; the upstream model card identifies the model as distilled from \texttt{Qwen/Qwen3-0.6B}, so applicable upstream Qwen3 notices should also be preserved. See also: model card, \url{https://huggingface.co/amd/ReasonLite-0.6B}; license text, \url{https://huggingface.co/amd/ReasonLite-0.6B/blob/main/LICENSE-AMD-OpenRAIL-D}; project repository, \url{https://github.com/AMD-AGI/ReasonLite}; technical article, \url{https://www.amd.com/en/developer/resources/technical-articles/2026/introducing-reasonlite-0-6b.html}.

    \item \textbf{ReasonLite-0.6B-Turbo.} Version used: \texttt{amd/ReasonLite-0.6B-Turbo}. Original owner/creator: Advanced Micro Devices, Inc. (AMD) and the ReasonLite project contributors. License: ReasonLite Open RAIL-D / OpenRAIL, with responsible-use restrictions; the upstream model card identifies the model as distilled from \texttt{Qwen/Qwen3-0.6B}, so applicable upstream Qwen3 notices should also be preserved. See also: model card, \url{https://huggingface.co/amd/ReasonLite-0.6B-Turbo}; license text, \url{https://huggingface.co/amd/ReasonLite-0.6B-Turbo/blob/main/LICENSE-AMD-OpenRAIL-D}; project repository, \url{https://github.com/AMD-AGI/ReasonLite}; technical article, \url{https://www.amd.com/en/developer/resources/technical-articles/2026/introducing-reasonlite-0-6b.html}.

    \item \textbf{Granite-4.0-350M.} Version used: \texttt{ibm-granite/granite-4.0-350m}. Original owner/creator: Granite Team, IBM. License: Apache License 2.0. See also: model card, \url{https://huggingface.co/ibm-granite/granite-4.0-350m}; license text, \url{https://www.apache.org/licenses/LICENSE-2.0}; project repository, \url{https://github.com/ibm-granite/granite-4.0-nano-language-models}; documentation, \url{https://www.ibm.com/granite/docs/models/granite}.
\end{itemize}

\paragraph{Language-model benchmarks and datasets.}
For all benchmarks below, we use the standard public versions. We do not repackage these datasets as new datasets, and we do not quote or reproduce individual held-out examples in the paper or in release materials.
\begin{itemize}[leftmargin=*]
    \item \textbf{GSM8K.} Version used: \texttt{openai/gsm8k}, standard train/test split of grade-school math word problems. Original owners/creators: OpenAI; Cobbe et al. License: MIT License. See also: dataset card, \url{https://huggingface.co/datasets/openai/gsm8k}; paper, \url{https://arxiv.org/abs/2110.14168}; license text, \url{https://opensource.org/licenses/MIT}.

    \item \textbf{MATH.} Version used: the MATH benchmark from
    \texttt{hendrycks/math}, with the common Hugging Face mirror
    \texttt{EleutherAI/hendrycks\_math}. Original owners/creators: Hendrycks, Burns, Kadavath, Arora, Basart, Tang, Song, and Steinhardt. License: MIT License. See also: official repository, \url{https://github.com/hendrycks/math}; Hugging Face mirror, \url{https://huggingface.co/datasets/EleutherAI/hendrycks_math}; paper, \url{https://arxiv.org/abs/2103.03874}; license text, \url{https://opensource.org/licenses/MIT}.

    \item \textbf{Orca-Math.} Version used:
    \texttt{microsoft/orca-math-word-problems-200k}, default training split of approximately 200K grade-school math word problems. Original owner/creator: Microsoft; Mitra, Khanpour, Rosset, and Awadallah. License: MIT License. See also: dataset card, \url{https://huggingface.co/datasets/microsoft/orca-math-word-problems-200k}; paper, \url{https://arxiv.org/abs/2402.14830}; license text, \url{https://opensource.org/licenses/MIT}.

    \item \textbf{Omni-MATH.} Version used: \texttt{KbsdJames/Omni-MATH}, default test split. Original owners/creators: the Omni-MATH authors. License: Apache License 2.0. See also: dataset card, \url{https://huggingface.co/datasets/KbsdJames/Omni-MATH}; paper, \url{https://arxiv.org/abs/2410.07985}; license text, \url{https://www.apache.org/licenses/LICENSE-2.0}.

    \item \textbf{NuminaMath-CoT.} Version used:
    \texttt{AI-MO/NuminaMath-CoT}, default public split of approximately 860K mathematical problem-solution pairs with Chain-of-Thought (CoT) solutions. Original owners/creators: Numina / Project Numina; Jia Li and others. License: Apache License 2.0. Data as drawn from sources including Chinese high-school mathematics exercises, US and international mathematics olympiad problems, online exam-paper PDFs, and mathematics discussion forums, with processing steps including OCR, segmentation into problem-solution pairs, translation into English, realignment into a CoT reasoning format, and final-answer formatting. We use the public dataset for research/training and evaluation purposes, report only aggregate results, and do not reproduce individual examples. See also: dataset card,
    \url{https://huggingface.co/datasets/AI-MO/NuminaMath-CoT}; project repository,
    \url{https://github.com/project-numina/aimo-progress-prize}; license text,
    \url{https://www.apache.org/licenses/LICENSE-2.0}; dataset paper/report,
    \url{https://github.com/project-numina/aimo-progress-prize/blob/main/report/numina_dataset.pdf}.

    \item \textbf{GPQA-Diamond.} Version used: \texttt{Idavidrein/gpqa}, diamond split/file \texttt{gpqa\_diamond.csv}. Original owners/creators: Rein, Hou, Stickland, Petty, Pang, Dirani, Michael, and Bowman. License: Creative Commons Attribution 4.0 International (CC BY 4.0). Access terms: the upstream dataset is gated and requires agreeing not to reveal examples from the dataset in plain text or images online; we respect this by reporting only aggregate results and not reproducing examples. See also: dataset card, \url{https://huggingface.co/datasets/Idavidrein/gpqa}; source repository, \url{https://github.com/idavidrein/gpqa}; paper, \url{https://arxiv.org/abs/2311.12022}; license text, \url{https://creativecommons.org/licenses/by/4.0/}.

    \item \textbf{MBPP.} Version used:
    \texttt{google-research-datasets/mbpp}, full and/or sanitized splits as used by the evaluation loader. Original owners/creators: Google Research; Austin et al. License: Creative Commons Attribution 4.0 International (CC BY 4.0). See also: dataset card, \url{https://huggingface.co/datasets/google-research-datasets/mbpp}; original repository, \url{https://github.com/google-research/google-research/tree/master/mbpp}; paper \url{https://arxiv.org/abs/2108.07732}; license text, \url{https://creativecommons.org/licenses/by/4.0/}.

    \item \textbf{AI2 ARC.} Version used: \texttt{allenai/ai2\_arc}, ARC-Challenge and/or ARC-Easy splits used in evaluation. Original owners/creators: Allen Institute for AI; Clark et al. License: Creative Commons Attribution-ShareAlike 4.0 International (CC BY-SA 4.0). We report aggregate scores only and do not redistribute modified dataset contents; any redistributed derivative of the dataset would need to preserve the CC BY-SA 4.0 attribution and share-alike requirements. See also: dataset card, \url{https://huggingface.co/datasets/allenai/ai2_arc}; paper, \url{https://arxiv.org/abs/1803.05457}; license text,
    \url{https://creativecommons.org/licenses/by-sa/4.0/}.

    \item \textbf{SuperGLUE.} Version used: \texttt{aps/super\_glue}, the standard SuperGLUE benchmark tasks used by the evaluation loader. Original owners/creators: Wang et al. and the original creators of each component task. License/terms: composite/other; the upstream dataset card states that the primary SuperGLUE tasks are built on or derived from existing datasets and refers users to the original licenses for each component dataset, while noting that the licenses permit use and redistribution in a research context. We use SuperGLUE for research evaluation only and do not redistribute component task data. See also: dataset card, \url{https://huggingface.co/datasets/aps/super_glue}; benchmark website, \url{https://super.gluebenchmark.com/}; paper,
    \url{https://arxiv.org/abs/1905.00537}.

    \item \textbf{OpenBookQA.} Version used: OpenBookQA v1.0, via the official AllenAI repository and/or the common Hugging Face mirror \texttt{allenai/openbookqa}. Original owners/creators: Mihaylov, Clark, Khot, and Sabharwal; Allen Institute for AI. License/terms: the official AllenAI repository that distributes the OpenBookQA code, download scripts, and release information is licensed under Apache License 2.0; the Hugging Face dataset metadata may not provide a separate SPDX dataset license, so we preserve the official repository notice, cite the paper, and do not redistribute question data. See also: official repository,
    \url{https://github.com/allenai/OpenBookQA}; Hugging Face mirror,
    \url{https://huggingface.co/datasets/allenai/openbookqa}; paper,
    \url{https://arxiv.org/abs/1809.02789}; license text,
\url{https://www.apache.org/licenses/LICENSE-2.0}.

    \item \textbf{RACE.} Version used: RACE reading-comprehension dataset, via the official CMU page and/or Hugging Face mirror \texttt{ehovy/race}. Original owners/creators: Lai, Xie, Liu, Yang, and Hovy. License and terms: custom/other; the upstream terms state that RACE is available for non-commercial research only, that passages were obtained from the Internet and are not property of Carnegie Mellon University, and that users may not reproduce, duplicate, copy, sell, trade, resell, or exploit the contexts or derived data for any commercial purpose. We use RACE only for non-commercial research evaluation, report aggregate scores, and do not reproduce passages. See also: official dataset page,
    \url{http://www.cs.cmu.edu/~glai1/data/race/}; Hugging Face card,
    \url{https://huggingface.co/datasets/ehovy/race}; paper,
    \url{https://arxiv.org/abs/1704.04683}.
\end{itemize}

\section{Acknowledgements}
This research was supported in part by Lambda, Inc.